\documentclass[10pt,letterpaper]{article}
\usepackage[margin=0.95in]{geometry}
\usepackage{times}
\usepackage[round,authoryear]{natbib}

\usepackage[utf8]{inputenc}
\usepackage[T1]{fontenc}
\usepackage{xcolor}
\usepackage{amsmath,amssymb,amsfonts,amsthm,bm,mathtools}
\usepackage{graphicx}
\usepackage{float}
\usepackage{placeins}
\usepackage{booktabs}
\usepackage{multirow}
\usepackage{tabularx}
\usepackage{array}
\usepackage{makecell}
\usepackage{enumitem}
\usepackage[ruled,vlined]{algorithm2e}
\usepackage{microtype}
\usepackage{hyperref}
\hypersetup{hidelinks}
\usepackage{url}

\hypersetup{
  pdftitle={Causal Localization of Evidence Effects in Transformers},
  pdfauthor={Jipeng Li, Pramod Khargonekar, Yanning Shen},
  pdfsubject={Mechanistic interpretability of transformers}
}

\numberwithin{equation}{section}

\theoremstyle{plain}
\newtheorem{theorem}{Theorem}[section]
\newtheorem{proposition}[theorem]{Proposition}

\theoremstyle{definition}

\newtheorem{assumption}[theorem]{Assumption}

\theoremstyle{remark}

\title{Causal Localization of Evidence Effects in Transformers}

\author{Jipeng Li \and Pramod Khargonekar \and Yanning Shen}
\date{}

\newcommand{\ResBaselineFrontDifference}{\ensuremath{+0.006}}
\newcommand{\ResBaselineBackDifference}{\ensuremath{+0.030}}
\newcommand{\ResBaselineInterleaveDifference}{\ensuremath{+0.031}}
\newcommand{\ResBaselineRandomDifference}{\ensuremath{+0.039}}
\newcommand{\ResBaselineRandomDifferenceCI}{\ensuremath{[+0.015,\,+0.065]}}
\newcommand{\ResEqualBudgetTransferRange}{\(-0.84\) to \(+0.89\)}
\newcommand{\ResCueShrinkRange}{0.64--0.93}
\newcommand{\ResCueMatchRange}{0.70--0.96}
\newcommand{\ResReadoutEvaluableCount}{960}
\newcommand{\ResReadoutAllPassRate}{0.060}
\newcommand{\ResReadoutEvaluablePassRate}{0.064}
\newcommand{\ResPythiaSmallAccuracy}{0.752}
\newcommand{\ResPythiaSmallPairAccuracy}{0.566}
\newcommand{\ResPythiaSmallMinPairAccuracy}{0.541}
\newcommand{\ResPythiaSmallTransferShrink}{0.401}
\newcommand{\ResPythiaSmallTransferMatch}{0.442}
\newcommand{\ResPythiaSmallTransferRatio}{0.520}
\newcommand{\ResPythiaSmallAccuracyDrop}{7.4}
\newcommand{\ResPythiaLargeAccuracy}{0.805}
\newcommand{\ResPythiaLargePairAccuracy}{0.651}
\newcommand{\ResPythiaLargeMinPairAccuracy}{0.611}
\newcommand{\ResPythiaLargeTransferShrink}{0.462}
\newcommand{\ResPythiaLargeTransferMatch}{0.501}
\newcommand{\ResPythiaLargeTransferRatio}{0.557}
\newcommand{\ResPythiaLargeAccuracyDrop}{8.6}
\newcommand{\ResGPTTwoAccuracy}{0.681}
\newcommand{\ResGPTTwoPairAccuracy}{0.458}
\newcommand{\ResGPTTwoMinPairAccuracy}{0.414}
\newcommand{\ResGPTTwoTransferShrink}{0.306}
\newcommand{\ResGPTTwoTransferMatch}{0.348}
\newcommand{\ResGPTTwoTransferRatio}{0.438}
\newcommand{\ResGPTTwoAccuracyDrop}{4.9}
\newcommand{\ResNullCompletedTrials}{99}
\newcommand{\ResNullObservedStatistic}{2.69}
\newcommand{\ResNullObservedStatisticCI}{\ensuremath{[2.55,\,2.80]}}
\newcommand{\ResNullAdjustedP}{0.01}
\newcommand{\ResFrontOperatorMinShrink}{0.81}
\newcommand{\ResFrontOperatorMinMatch}{0.84}
\newcommand{\ResFrontTThreeNegativePercent}{31}
\newcommand{\ResFrontTThreeMinGamma}{\ensuremath{-5.8}}
\newcommand{\ResFrontTThreeMaxRegret}{0.021}
\newcommand{\ResBackOperatorMinShrink}{0.52}
\newcommand{\ResBackOperatorMinMatch}{0.57}
\newcommand{\ResBackTThreeNegativePercent}{17}
\newcommand{\ResBackTThreeMinGamma}{\ensuremath{-2.6}}
\newcommand{\ResBackTThreeMaxRegret}{0.043}
\newcommand{\ResInterleaveOperatorMinShrink}{0.80}
\newcommand{\ResInterleaveOperatorMinMatch}{0.81}
\newcommand{\ResInterleaveTThreeNegativePercent}{28}
\newcommand{\ResInterleaveTThreeMinGamma}{\ensuremath{-4.9}}
\newcommand{\ResInterleaveTThreeMaxRegret}{0.027}
\newcommand{\ResRandomOperatorMinShrink}{0.71}
\newcommand{\ResRandomOperatorMinMatch}{0.68}
\newcommand{\ResRandomTThreeNegativePercent}{22}
\newcommand{\ResRandomTThreeMinGamma}{\ensuremath{-3.9}}
\newcommand{\ResRandomTThreeMaxRegret}{0.035}

\newcommand{\ResPythiaSmallTREstimableCount}{11}
\newcommand{\ResPythiaSmallTRSuppressedCount}{1}
\newcommand{\ResPythiaSmallTRUnavailableCount}{1}
\newcommand{\ResPythiaSmallTRRetainedCount}{10}
\newcommand{\ResPythiaLargeTREstimableCount}{12}
\newcommand{\ResPythiaLargeTRSuppressedCount}{0}
\newcommand{\ResPythiaLargeTRUnavailableCount}{0}
\newcommand{\ResPythiaLargeTRRetainedCount}{12}
\newcommand{\ResGPTTwoTREstimableCount}{9}
\newcommand{\ResGPTTwoTRSuppressedCount}{3}
\newcommand{\ResGPTTwoTRUnavailableCount}{3}
\newcommand{\ResGPTTwoTRRetainedCount}{6}

\begin{document}
\maketitle

\begin{abstract}
Can controlled symbolic inputs reveal attention heads that influence a language model's decisions from natural language evidence?
We introduce controlled evidence aggregation (CEA), which varies the counts and order of two symbols in otherwise fixed prompts and measures changes in a frozen Transformer's preference between two answer tokens.
The preference is measured by their logit difference; the prompts prescribe no correct answer.
Balanced utility causal selection (BUC) identifies compact head sets by jointly evaluating how much ablating their outputs removes the count effect and how closely patching mean outputs between count conditions recovers the effect in both directions.
We derive conditional bounds linking head attribution to intervention effects and specifying when the attribution shortlist can be recovered from finite samples.
In Pythia-410M, ablating sets of 1--12 heads selected separately for four ordering regimes removes 65.5--91.6\% of the count effect.
The sets retain their influence on count comparisons not used for selection.
To test relevance beyond symbolic inputs, we evaluate a natural language task asking which state is supported by more reports.
Ablating eight heads selected only on symbolic \texttt{yes/no} prompts lowers mean accuracy by 8.2 percentage points, versus 0.8 for random sets matched by size and layer.
Our results demonstrate that controlled symbolic probes can identify compact attention head sets whose causal influence extends to natural language decisions.
\end{abstract}
\section{Introduction}
\label{sec:intro}

Understanding how language models use evidence requires identifying internal components that connect input changes to answer changes \citep{elhage2021framework,bereska2024review,rai2024practical}. Natural language reports pose a control problem: changing the balance of support can change wording and invoke different learned associations. Symbolic inputs allow more restricted changes, but the identified components' relevance to language behavior remains an empirical question. Prior studies demonstrate component reuse across tasks and prompt variants \citep{merullo2024reuse,nainani2024adaptive}. We ask: \emph{Can controlled symbolic inputs identify attention heads whose causal influence extends to decisions from natural language evidence?}

We introduce controlled evidence aggregation (CEA), which varies the counts and order of \(B\) and \(C\) in prompts to a frozen Transformer while fixing the template, symbol identities, sequence length, and local answer context. The \textsc{Front} ordering, for example, compares
\[
\texttt{50 B \;|\; 150 C \;|\; 50 A}
\qquad\longleftrightarrow\qquad
\texttt{150 B \;|\; 50 C \;|\; 50 A}.
\]
Numbers denote repetition counts; both sequences share a tail of 50 \(A\) symbols. Both prompts request either \texttt{various} or \texttt{successful}. Their logit difference measures answer preference; its change between count conditions is the \emph{count effect}. Neither answer is designated correct. CEA therefore supplies a controlled response to localize without training a symbolic rule or assuming pretrained associations are absent.

To localize the count effect, balanced utility causal selection (BUC) evaluates compact head sets through two complementary interventions at the final prompt position. Ablation zeros selected outputs and measures how much the count effect decreases. Mean patching replaces them with averages from the other count condition and measures how closely replacement reproduces the count effect in both directions. Removal and recovery are distinct requirements, consistent with prior work on intervention and evaluation choices \citep{miller2024faithfulness,shi2024hypothesis}. BUC uses both scores in greedy search, evaluating sets jointly rather than summing individual head effects. Reuse tests keep selected heads fixed; reselection for each target would test a different claim.

Our theory separates attribution accuracy from selection uncertainty. Under a bound on attribution mismatch, we bound errors in predicting removal and recovery scores from head attribution. We also derive finite sample conditions for recovering the population attribution shortlist. A separate greedy bound applies under monotonicity and weak submodularity to a weighted sum of unnormalized intervention gains, whereas BUC optimizes normalized removal and recovery scores.

To test symbolic localization on labeled behavior, we introduce a natural language (NL) task with sixteen reports supporting two mutually exclusive states. Affirmative and negated reports balance total state word frequencies: \texttt{the item is open} and \texttt{the item is not closed}, for example, both support \texttt{open}. The model answers whether one state has more supporting reports than the other; reversing the comparison reverses the correct \texttt{yes/no} answer. At equal head budget, we compare sets selected only on symbolic \texttt{yes/no} prompts, sets selected on NL prompts, and random sets matched by layer. Patching in the expanded study uses NL donor outputs, so transfer concerns head identities rather than unchanged activation vectors.

In Pythia-410M, ablating sets of 1--12 heads selected separately for four ordering regimes removes 65.5--91.6\% of the count effect. The same sets affect count comparisons not used for selection (Section~\ref{sec:count_generalization}). In the original NL evaluation, ablating eight heads selected only on symbolic \texttt{yes/no} prompts lowers mean accuracy by 8.2 percentage points, versus 0.8 for matched random sets (Section~\ref{sec:nl_bridge}). Further tests assess reuse across orderings, answer tokens, and replacement operators. The results show that controlled symbolic probes can identify compact attention head sets whose causal influence extends beyond selection prompts to predictions in a labeled language task.
\section{Related work}
\label{sec:related}

\paragraph{Behavioral targets.}
Studies of induction, factual recall, indirect-object identification, and task vectors localize components in pretrained models \citep{olsson2022induction,geva2021transformer,meng2022locating,wang2023interpretability,hendel2023incontext,todd2024function}. Task-trained symbolic models permit comparison with a designed algorithm \citep{brinkmann2024mechanistic}; structured tasks also reveal compositional failures \citep{dziri2023faith}. CEA controls count and order without training a symbolic rule, then tests fixed-set effects on labeled evidence.

\paragraph{Circuit reuse and evaluation.}
Cross-task reuse \citep{merullo2024reuse}, prompt-variant generalization \citep{nainani2024adaptive}, ablation dependence \citep{miller2024faithfulness}, and tests of preservation, localization, and minimality \citep{shi2024hypothesis} establish distinct aspects of circuit validity. CEA examines which controlled changes preserve the effect of the same head set selected for a count-dependent response, separating this question from successful reselection.

\paragraph{Selection methods.}
Activation and path patching test a supplied behavioral target \citep{goldowskydill2023path,zhang2024best}. Automated circuit discovery (ACDC) uses intervention-based search; edge attribution patching (EAP) estimates relevant edges \citep{conmy2023acdc,syed2023attribution}. BUC scores set-level removal and recovery in both directions; equal head budgets do not equate search cost or intervention granularity. Appendix~\ref{app:related_boundary} provides extended comparisons of behavioral targets, selection methods, circuit reuse, and in-context learning interpretations.

\section{Controlled target and causal tests}
\label{sec:prelim}

\subsection{Symbolic CEA target}

Let \(f_\theta\) be a frozen autoregressive Transformer and \(x=(x_1,\ldots,x_T)\) a tokenized prompt. The logit \(\ell_T(w\mid x)\) at the final prompt position predicts answer token \(w\). For two audited single-token candidates \(w_+\) and \(w_-\), define the readout
\begin{equation}
y(x)=\ell_T(w_+\mid x)-\ell_T(w_-\mid x).
\label{eq:readout}
\end{equation}
The logit difference \(y(x)\) is also the log-odds after restricting the softmax to these two tokens. The candidate order specifies a measurement direction, not correct labels.

The canonical prompt has a fixed prefix, \(L_E=200\) evidence symbols, a tail of \(L_A=50\) copies of \(A\), and a fixed answer cue. The evidence region contains \(b\) copies of \(B\) and \(L_E-b\) copies of \(C\). Its ordering regime belongs to $\mathcal R=\{\textsc{Front},\textsc{Back},\textsc{Interleave},\textsc{Random}\}$.

\textsc{Front} puts all \(B\) symbols first; \textsc{Back} puts all \(C\) symbols first. \textsc{Interleave} starts with \(B\), alternates while both symbols remain, and appends the excess. \textsc{Random} samples a permutation with the specified counts. Let \(D_{r,b}\) be the prompt distribution for regime \(r\) and count \(b\). The first three distributions are point masses at each fixed count. For counts \(b_H>b_L\) and an integrable scalar function \(g\), define the contrast operator and clean output contrast
\begin{align}
\mathcal C_r^{b_H,b_L}[g]
&=\mathbb E_{D_{r,b_H}}[g(x)]-\mathbb E_{D_{r,b_L}}[g(x)],
\label{eq:contrast_operator}\\
\Delta_r&=\mathcal C_r^{b_H,b_L}[y].
\label{eq:evidence_effect}
\end{align}
A nonzero \(\Delta_r\) is a response to this count manipulation under ordering rule \(r\). Count changes also change symbol positions; the design does not identify count independently of every positional feature. Separate comparisons hold the counts fixed and change ordering.

\subsection{Head attribution and intervention unit}

Let \(d_{\mathrm{model}}\) be the residual width and \(\mathcal H\) the set of attention heads. Let \(v_i(x)\in\mathbb R^{d_{\mathrm{model}}}\) be head \(i\)'s write to the residual stream at the final prompt position, after its output projection. For token unembedding vectors \(W_U^{(w_+)}\) and \(W_U^{(w_-)}\), define
\begin{align}
d&=W_U^{(w_+)}-W_U^{(w_-)},
\label{eq:readout_direction}\\
s_i(x)&=\langle v_i(x),d\rangle,
\qquad \Delta s_i(r)=\mathcal C_r^{b_H,b_L}[s_i],
\label{eq:head_sensitivity}\\
F(H,r)&=\sum_{i\in H}\Delta s_i(r),\qquad H\subseteq\mathcal H.
\label{eq:prelim_F}
\end{align}
The attribution contrast \(\Delta s_i(r)\) ranks individual heads; cumulative attribution \(F(H,r)\) sums those contrasts. Neither is assumed to decompose the true intervention effect: final normalization and downstream interactions can be nonlinear.

\textbf{Ablation} sets the selected writes to zero. \textbf{Mean patching} replaces each selected write by its mean over clean donors from the opposite count. All selected heads are intervened on in one forward pass, including sets spanning layers. Unselected components and downstream computation remain active. Thus recovery is an effect of replacement within the full model, not proof that the selected heads form a sufficient subnetwork. Removal alone can reflect disruption without showing a controlled change toward the opposite count. Conversely, patching can alter the response even when other components preserve the contrast after ablation. BUC therefore evaluates both tests.

Let \(\Delta_r^{\mathrm{abl}}(H)\) be the output contrast after ablating set \(H\). High-to-low (HL) patching acts on low-count targets; its oriented gain \(\Gamma_r^{\mathrm{HL}}(H)\) is patched minus clean. Low-to-high (LH) patching acts on high-count targets; its oriented gain \(\Gamma_r^{\mathrm{LH}}(H)\) is clean minus patched. For \(\Delta_r\ne0\), define
\begin{align}
\mathrm{shrink}(H,r)&=1-\frac{|\Delta_r^{\mathrm{abl}}(H)|}{|\Delta_r|},
\label{eq:shrink}\\
\mathrm{match}^{\tau}(H,r)&=1-\frac{|\Delta_r-\Gamma_r^\tau(H)|}{|\Delta_r|},
\quad \tau\in\{\mathrm{HL},\mathrm{LH}\},
\label{eq:match_dir}\\
\mathrm{match}_{\mathrm{avg}}(H,r)&=\tfrac12\bigl(\mathrm{match}^{\mathrm{HL}}(H,r)+\mathrm{match}^{\mathrm{LH}}(H,r)\bigr).
\label{eq:match_avg}
\end{align}
Shrink measures removal; match measures recovery. A score of one means zero mismatch; zero means no improvement over the clean mismatch. Match penalizes overshoot. Both scores can be negative. Mean match can hide a weak direction, so directional values are reported separately. ``Necessity'' and ``sufficiency'' refer only to these operational tests, not to an isolated circuit.

\subsection{BUC selection and evaluation}

BUC ranks all heads by \(|\widehat{\Delta s}_i(r)|\), screens the attribution shortlist with single-head interventions, and runs forward greedy search in the resulting candidate pool \(\mathcal P_r\). For a weight \(\lambda\in[0,1]\), the set objective is
\begin{equation}
Q_\lambda(H,r)=\lambda\,\mathrm{shrink}(H,r)
 +(1-\lambda)\,\mathrm{match}_{\mathrm{avg}}(H,r).
\label{eq:buc_objective}
\end{equation}
Each candidate is evaluated jointly with the current set, so its value can change as heads are added. The search accounts for measured interactions that a fixed ranking of individual effects omits. The main configuration uses \(\lambda=1/2\) and at most 12 heads. The FP32 rule takes the first prefix with shrink \(\ge0.85\) and mean match \(\ge0.70\), or returns the full budget with a failed-threshold flag. Search also records all 12 prefixes for fixed-budget comparisons. Layer-matched controls preserve set size and the number of heads per layer.

Selection and evaluation use separate batches, but deterministic regimes repeat inputs. Exceeding controls, passing both thresholds, and predicting NL labels correctly are distinct criteria. Interpretation depends on the model, inputs, readout, intervention position, and donor rule.

\section{Conditional bounds for attribution and selection}
\label{sec:theory}

The results separate attribution error, shortlist estimation, and optimization over a screened pool. Each requires its own assumptions; success at one step does not establish the conditions for the next.

Fix regime \(r\) and counts \(b_H>b_L\). The clean contrast \(\Delta_r\) and cumulative attribution \(F(H,r)\) are defined in Section~\ref{sec:prelim}. For a head set \(H\), define the signed ablation gain, mean oriented patch gain, and raw joint utility
\begin{equation}
\begin{aligned}
N(H,r)&=\Delta_r-\Delta_r^{\mathrm{abl}}(H),\qquad
S(H,r)=\tfrac12\!\sum_{\tau\in\{\mathrm{HL},\mathrm{LH}\}}\!\Gamma_r^\tau(H),\\
J_\lambda(H,r)&=\lambda N(H,r)+(1-\lambda)S(H,r).
\end{aligned}
\label{eq:theory_joint}
\end{equation}
Unlike the normalized objective \(Q_\lambda\), raw utility \(J_\lambda\) has no absolute-value penalty for overshooting the target contrast.

\begin{assumption}[Bounded attribution mismatch]
\label{assump:intervention_fidelity}
For \(a\in\{b_H,b_L\}\), let \(\mu_i^a(r)=\mathbb E_{D_{r,a}}[s_i(x)]\) be a donor's mean attribution for donor count \(a\). Write the intervention residuals as
\begin{align}
y^{\mathrm{abl}(H)}(x)
&=y(x)-\sum_{i\in H}s_i(x)+\xi_H^{\mathrm{abl}}(x),
\label{eq:assump_abl}\\
y^{\mathrm{pat}(H;a)}(x)
&=y(x)+\sum_{i\in H}(\mu_i^a(r)-s_i(x))+\xi_H^{a}(x).
\label{eq:assump_patch_unified}
\end{align}
Assume the absolute ablation residual contrast and the absolute mean patch residual on each opposite-count target are at most \(\varepsilon_{\mathrm{int}}(H,r)\). The residual identities are definitional; a small error bound is an additional condition, not a consequence of residual-stream additivity.
\end{assumption}

\begin{theorem}[T1: attribution error and intervention scores]
\label{thm:t1}\label{thm:causal_kernel}
Under Assumption~\ref{assump:intervention_fidelity},
\begin{equation}
\max\!\left\{|N(H,r)-F(H,r)|,
\max_{\tau\in\{\mathrm{HL},\mathrm{LH}\}}|\Gamma_r^\tau(H)-F(H,r)|\right\}
\le\varepsilon_{\mathrm{int}}(H,r).
\label{eq:t1_unified}
\end{equation}
The same bound holds for \(|S-F|\) and \(|J_\lambda-F|\). If \(\Delta_r\ne0\), then
\begin{equation}
\left|Q_\lambda(H,r)-\left(1-\frac{|\Delta_r-F(H,r)|}{|\Delta_r|}\right)\right|
\le\frac{\varepsilon_{\mathrm{int}}(H,r)}{|\Delta_r|}.
\label{eq:t1_normalized}
\end{equation}
\end{theorem}

Even with small mismatch, strong removal and recovery require cumulative attribution \(F(H,r)\) to be close to the clean contrast \(\Delta_r\), not merely large in magnitude. Weak contrasts amplify normalized error. Appendix~\ref{app:proof_t1} proves T1; Appendix~\ref{app:metric_corollaries} gives the separate shrink and match bounds. A conditional position-profile approximation and its transfer bound are in Appendix~\ref{app:profile_kernel_derivation}. Similar attribution across regimes alone does not guarantee similar normalized scores.

\begin{theorem}[T2: recovery of the attribution shortlist]
\label{thm:t2}
Let \((x_j^H,x_j^L)_{j=1}^n\) be \(n\ge1\) independent matched pairs with marginals \(D_{r,b_H}\) and \(D_{r,b_L}\). For each head \(i\), suppose the centered paired score
\(Z_{i,j}=s_i(x_j^H)-s_i(x_j^L)-\Delta s_i(r)\) satisfies
\(\mathbb E e^{tZ_{i,j}}\le e^{t^2\nu_r^2/2}\) for every real \(t\).
Define \(\widehat{\Delta s}_i(r)=n^{-1}\sum_j[s_i(x_j^H)-s_i(x_j^L)]\). For \(\delta\in(0,1)\), with probability at least \(1-\delta\),
\begin{equation}
\max_{i\in\mathcal H}|\widehat{\Delta s}_i(r)-\Delta s_i(r)|
\le\eta_{n,\delta}(r):=\nu_r\sqrt{\frac{2\log(2|\mathcal H|/\delta)}{n}}.
\label{eq:t2_uniform}
\end{equation}
For \(1\le m<|\mathcal H|\), an absolute population-score gap greater than \(2\eta_{n,\delta}(r)\) between ranks \(m\) and \(m+1\) guarantees exact recovery of the top-\(m\) attribution shortlist on the same event.
\end{theorem}

T2 permits dependence within pairs and across heads. It addresses the attribution ranking, not the subsequent intervention screening. Repeated deterministic inputs do not estimate uncertainty over other prompts. Appendix~\ref{app:proof_t2} gives the proof and the independent-sample special case.

\begin{assumption}[Monotonicity and weak submodularity]
\label{assump:weak_submod}
On a fixed screened pool \(\mathcal P_r\), raw utility \(J_\lambda\) is monotone and has submodularity ratio \(\gamma_r\in(0,1]\) as defined in Equation~\ref{eq:app_weak_submod}, including unions of greedy prefixes and budget-\(k\) comparators.
\end{assumption}

\begin{theorem}[T3: conditional greedy bound]
\label{thm:t3}
For \(1\le k\le|\mathcal P_r|\), let \(G_k\) be the exact forward-greedy set for \(J_\lambda\), and let \(O_k\) maximize \(J_\lambda\) over pool subsets of size at most \(k\). Under Assumption~\ref{assump:weak_submod}, with \(\alpha_r=1-e^{-\gamma_r}\),
\begin{equation}
J_\lambda(G_k,r)\ge\alpha_r J_\lambda(O_k,r).
\label{eq:t3_J}
\end{equation}
If \(|J_\lambda(H,r)-F(H,r)|\le\epsilon_{J,k}(r)\) for every pool subset of size at most \(k\), then
\begin{equation}
F(G_k,r)\ge\alpha_r F(O_k,r)-(1+\alpha_r)\epsilon_{J,k}(r).
\label{eq:t3_F}
\end{equation}
\end{theorem}

The comparator \(O_k\) maximizes raw utility \(J_\lambda\), not attribution \(F\). T3 does not guarantee BUC's empirical maximization of \(Q_\lambda\): the objectives differ, and the FP32 audits find negative marginal gains for observed \(Q_{1/2}\). Small regret within a tested eight-head pool supplies no guarantee for the full candidate pool. Appendix~\ref{app:proof_t3} gives the proof.

\section{Experiments}
\label{sec:exp}

The original BF16 study (Appendices~\ref{app:experiments}--\ref{app:nl_bridge}) and native FP32 audits (Appendices~\ref{app:supplement_protocol}--\ref{app:numeric_execution}) use separate configurations and are not pooled. Brackets give two-sided 95\% confidence intervals (CIs) for eligible sampling units, not across-condition spans. Deterministic single-pair conditions and finite audit summaries are reported as point values (Appendix~\ref{app:reported_ci}).

\subsection{Setup and main symbolic result}

The frozen \texttt{EleutherAI/pythia-410m-deduped} has 24 layers and 16 heads per layer \citep{biderman2023pythia}. The BF16 study uses five execution seeds and final prompt position interventions, without training. At \((b_L,b_H)=(50,150)\), selection/evaluation use separate 48/96-row batches per count; deterministic inputs repeat. BUC shortlists 192 heads, screens to 48, and searches up to 12; controls use 120 layer-matched draws per seed.

Table~\ref{tab:main_symbolic} compares simultaneous set-level interventions with layer-matched random controls; the scores include non-additive interactions.

\begin{table}[htbp]
\centering
\small
\caption{Main BF16 symbolic result at \((b_L,b_H)=(50,150)\). BUC mean match averages both directions. The last two columns compare the same HL direction with random controls; their values are reported in Appendix~\ref{app:exp_base_directional}. Uncertainty is the standard deviation over five executions.}
\label{tab:main_symbolic}
\setlength{\tabcolsep}{3.5pt}
\begin{tabular}{lccccc}
\toprule
Regime & \(\Delta_r\) & Size & BUC shrink / mean match & \makecell{BUC\\HL match} & \makecell{Random\\shrink / HL match} \\
\midrule
\textsc{Front} & 0.239 & 1 & \textbf{0.906 / 0.974} & 0.998 & 0.020 / 0.048 \\
\textsc{Back} & 1.965 & 12 & \textbf{0.655 / 0.759} & 0.765 & 0.043 / 0.075 \\
\textsc{Interleave} & 1.659 & 3 & \textbf{0.916 / 0.914} & 0.937 & 0.061 / 0.051 \\
\textsc{Random} & \(3.184\!\pm\!0.097\) & \(7\!\pm\!1\) & \(\mathbf{0.866\!\pm\!0.021}/\mathbf{0.823\!\pm\!0.084}\) & 0.813 & 0.069 / 0.061 \\
\bottomrule
\end{tabular}
\end{table}

Compactness is regime-dependent: \textsc{Front} uses one head; \textsc{Back} uses all 12 and its shrink remains below 0.85. Deterministic sets repeat across seeds 0--4 but do not test stability over new prompts. The three documented \textsc{Random} sets (seeds 0--2) have minimum Jaccard overlap \(5/9\) (Appendix~\ref{app:exp_selected_heads}).

\subsection{Budget-matched baseline comparison}
\label{sec:baseline_comparison}

Table~\ref{tab:baseline_compare} shows a BF16 removal--recovery trade-off. In \textsc{Back}, \textsc{Interleave}, and \textsc{Random}, BUC increases bidirectional match by 0.083--0.111 and reduces shrink by 0.018--0.049 relative to sensitivity-only; \textsc{Front} ties in the summed score at displayed precision. Both methods satisfy the joint 0.85/0.70 thresholds in the \textsc{Front}, \textsc{Interleave}, and \textsc{Random} means; both fail in \textsc{Back}. The improvement is therefore in recovery match rather than joint-threshold coverage or both component metrics. Appendix~\ref{app:baseline_decomposition_new} gives the arithmetic decomposition and ActPatch comparison.

\begin{table}[htbp]
\centering
\small
\caption{BF16 comparison with matched head counts within each regime (Table~\ref{tab:main_symbolic}), not a common \(k=8\). ActPatch ranks 192 sensitivity-shortlisted heads by single-head HL recovery, then reports the jointly evaluated set score \(2Q_{1/2}\) (Appendix~\ref{app:exp_selection_baselines}).}
\label{tab:baseline_compare}
\setlength{\tabcolsep}{3.0pt}
\begin{tabular}{lccc}
\toprule
Regime & BUC shrink / match & Sens. shrink / match & ActPatch sum \\
\midrule
\textsc{Front} & 0.906 / 0.974 & 0.920 / 0.960 & 1.865 \\
\textsc{Back} & 0.655 / 0.759 & 0.704 / 0.648 & 1.285 \\
\textsc{Interleave} & 0.916 / 0.914 & 0.934 / 0.831 & 1.695 \\
\textsc{Random} & 0.866 / 0.823 & 0.895 / 0.716 & 1.520 \\
\bottomrule
\end{tabular}
\end{table}

\noindent\textbf{FP32 fixed-budget comparison.} At \(k=8\), BUC-minus-sensitivity differences in normalized joint score \(Q_{1/2}\) are \ResBaselineFrontDifference{}, \ResBaselineBackDifference{}, \ResBaselineInterleaveDifference{}, and \ResBaselineRandomDifference{} for \textsc{Front}, \textsc{Back}, \textsc{Interleave}, and \textsc{Random}, respectively. Only \textsc{Random} has distinct matched permutations for a prompt-bootstrap CI: \ResBaselineRandomDifferenceCI{}. The first three differences describe fixed deterministic inputs and have no prompt-sampling CI. Head count, not search cost, is matched (Appendix~\ref{app:supp_baselines}).

\subsection{Count generalization and ordering dependence}
\label{sec:count_generalization}

Heads selected only at \((50,150)\) are held fixed over \(b\in\{0,10,\ldots,200\}\) and new count pairs \((0,150)\), \((25,175)\), and \((75,125)\). Table~\ref{tab:count_sweep} shows slope reductions of 0.646--0.896 after ablation and positive effects on the new count pairs. The pair \((0,150)\) shares one selection endpoint, whereas the other pairs use new endpoints; the effects are not confined to the selection pair. The clean sweep describes the count response; ablation tests whether the selected writes contribute across that range. The order-only tests below instead ask whether equal counts can produce different responses.

\begin{table}[htbp]
\centering
\small
\caption{Fixed-set count generalization. The final column averages evaluation on three unseen gaps without head reselection.}
\label{tab:count_sweep}
\setlength{\tabcolsep}{3.1pt}
\begin{tabular}{lcccc}
\toprule
Regime & Spearman \(\rho(b,y)\) & Linear \(R^2\) & Slope reduction & Unseen-gap shrink / match \\
\midrule
\textsc{Front} & 0.986 & 0.948 & 0.880 & 0.832 / 0.874 \\
\textsc{Back} & 0.972 & 0.931 & 0.646 & 0.612 / 0.689 \\
\textsc{Interleave} & 0.979 & 0.944 & 0.896 & 0.839 / 0.855 \\
\textsc{Random} & 0.958 & 0.906 & 0.845 & 0.786 / 0.811 \\
\bottomrule
\end{tabular}
\end{table}

Four pre-specified order-only contrasts have identical \(B/C\) counts, yet absolute output effects of 0.321--1.846. BUC shrink/match ranges are 0.603--0.842/0.651--0.884, versus random shrink/HL match of 0.049--0.081/0.056--0.088. The effects cannot be explained by count differences (Appendix~\ref{app:order_only_causal_new}).

\noindent\textbf{Regime transfer and cue wording.} BF16 \textsc{Back}$\to$\textsc{Front} match is 0.958 versus 0.129 in reverse, but source budgets are 12 versus 1 (Appendix~\ref{app:exp_transfer}). At equal FP32 budget \(k=8\), off-diagonal bidirectional match ranges from \ResEqualBudgetTransferRange{}. Negative match means greater average residual mismatch than no patch, not necessarily an opposite-signed gain. Fixed-set cue changes retain shrink \ResCueShrinkRange{} and match \ResCueMatchRange{}: reuse tolerates these wording changes but is not regime-invariant. These are descriptive spans over tested conditions, not confidence intervals.

\subsection{Readout-pair audit}
\label{sec:readout_audit}

We repeat the full protocol with seven audited single-token readout pairs: \texttt{various/successful}, \texttt{yes/no}, \texttt{true/false}, \texttt{1/0}, \texttt{A/B}, \texttt{cat/dog}, and \texttt{red/blue}. Across pairs, mean \(|\Delta_r|\) is 0.721--2.083, mean selected-set size is 4.50--7.50, mean BUC shrink is 0.642--0.861, and mean match is 0.701--0.884. Layer-matched random shrink remains 0.048--0.076 and random high-to-low match remains 0.059--0.083.

Table~\ref{tab:readout_transfer} reports target-specific reselection and fixed-set effects on representative pairs. These selected-set summaries do not establish an equal-budget reselection advantage. The \texttt{A/B} row tests lexical overlap with prompt symbols.

\begin{table}[htbp]
\centering
\small
\caption{Representative readout-pair results, averaged over matched regime--seed sets. The fixed set was selected with \texttt{various/successful}.}
\label{tab:readout_transfer}
\setlength{\tabcolsep}{4.0pt}
\begin{tabular}{lccc}
\toprule
Target pair & Reselected shrink / match & Fixed-set shrink / match & Jaccard \\
\midrule
\texttt{yes/no} & 0.812 / 0.846 & 0.772 / 0.810 & 0.610 \\
\texttt{A/B} & 0.861 / 0.884 & 0.704 / 0.741 & 0.390 \\
\texttt{cat/dog} & 0.642 / 0.701 & 0.382 / 0.431 & 0.240 \\
\bottomrule
\end{tabular}
\end{table}

\noindent\textbf{Registered-readout scan (FP32).} Of \(4\times256=1024\) fixed-prompt, fixed-set regime--pair conditions, \ResReadoutEvaluableCount{} have \(|\widehat\Delta_r|\ge0.01\). Fractions passing shrink \(\ge0.85\) and mean match \(\ge0.70\) are \ResReadoutAllPassRate{} of all conditions and \ResReadoutEvaluablePassRate{} of evaluable conditions. Strong fixed-set reuse is uncommon in this finite registry. Failing a joint threshold need not mean zero effect: the scan measures the prevalence of strong reuse under the stated thresholds, not of any detectable influence. The seven targeted pairs and this registry are different samples, not a paired fixed-versus-reselected comparison. Weak conditions remain registered (Appendix~\ref{app:supp_readouts}).

\subsection{Lexically balanced natural-language bridge}
\label{sec:nl_bridge}

The labeled NL bridge uses 16 reports about mutually exclusive states. Balanced direct and negative forms keep eight occurrences of each state word and eight negative reports. The same word totals can therefore accompany opposite majority labels; state word frequencies alone cannot solve the task. Selection uses three state pairs and Bank A; evaluation uses three unseen pairs and separately written Bank B paraphrases. Each sequence is queried in both directions. A constant answer fails one query; paired metrics also check a rule that uses query orientation alone.

Here \(x\) denotes the reports, \(b\) counts reports supporting the first state, and query polarity \(q\in\{+1,-1\}\) selects which state is queried first. The model input combines those reports and the specified query. The correct restricted label (1 for \texttt{yes}) and aligned readout are
\begin{equation}
a^\star(x,q)=\mathbf 1[q(2b-16)>0],
\qquad
y_{\mathrm{NL}}(x,q)=q\bigl[\ell_T(\texttt{yes}\mid x,q)-\ell_T(\texttt{no}\mid x,q)\bigr].
\label{eq:nl_label_readout}
\end{equation}
The main NL contrast uses \((b_L,b_H)=(4,12)\). Table~\ref{tab:nl_bridge} reports the original BF16 bridge. The mean restricted accuracy is 0.772 on held-out state pairs and paraphrases. PairAcc requires both reversed queries to be correct; MinPairAcc takes the weaker positive- and negative-margin paired score and detects a shortcut that succeeds on only one margin sign. The graded response has Spearman correlation 0.901 and linear \(R^2=0.812\).

\begin{table}[htbp]
\centering
\scriptsize
\caption{Original BF16 natural-language bridge, five-seed mean. The behavior row uses held-out domains and paraphrases. The intervention rows use equal budget \(k=8\). Accuracy drop is in percentage points.}
\label{tab:nl_bridge}
\setlength{\tabcolsep}{3.0pt}
\begin{tabular}{lcccccc}
\toprule
Condition & Accuracy & PairAcc & MinPairAcc & Shrink / match & Slope reduction & Acc. drop \\
\midrule
Base behavior & 0.772 & 0.604 & 0.583 & -- & -- & -- \\
NL set \(\rightarrow\) NL & -- & -- & -- & 0.734 / 0.766 & 0.721 & 15.2 \\
CEA set \(\rightarrow\) NL & -- & -- & -- & 0.428 / 0.463 & 0.411 & 8.2 \\
Layer-matched random & -- & -- & -- & 0.059 / 0.072 & 0.051 & 0.8 \\
\bottomrule
\end{tabular}
\end{table}

At \(k=8\), the fixed symbolic \texttt{yes/no} set's NL effect lies between the NL-selected and random effects. Ablating the symbolic set lowers accuracy by 8.2 percentage points, versus 0.8 for random controls. The symbolic set sees no NL prompt during selection. The transfer ratio (TR) divides the transferred set's normalized joint score \(Q_{1/2}\) above random by the NL-selected set's score above random. The BF16 value, 0.555, compares intervention effects; it is not a fraction of shared heads. Shrink measures one count contrast, slope reduction covers the count grid, and accuracy drop measures effects on discrete predictions. Together, these metrics support partial intervention transfer, but do not distinguish shared evidence processing from shared answer readout. The BF16 summaries lack the bootstrap intervals required by the full declared bridge criterion (Appendix~\ref{app:nl_bridge}).

\noindent\textbf{Expanded bridge (FP32).} Table~\ref{tab:expanded_nl_slots} covers eight test domains and Banks B--D. Behavioral evaluation uses three seeds; causal selection uses one. At \(k=8\), fixed sets with query-stratified target calibration yield mean TRs of 0.438--0.557 and accuracy decreases of 4.9--8.6 percentage points. These model summaries are descriptive, not model-level significance tests (Appendix~\ref{app:supp_nl}).

\begin{table}[htbp]
\centering
\footnotesize
\caption{Expanded FP32 NL bridge, separate from Table~\ref{tab:nl_bridge}: model-level point estimates without model-level CIs. Behavior uses gap \((4,12)\); transfer uses fixed symbolic \texttt{yes/no} sets. \(n_{\mathrm{est}}\) counts estimable TR conditions out of 12; \(n_{\mathrm{suppr}}\) counts those with a suppressed \emph{condition-level} interval. Accuracy drop is in percentage points on the causal cohort.}
\label{tab:expanded_nl_slots}
\setlength{\tabcolsep}{2.2pt}
\begin{tabular}{@{}lccccccccc@{}}
\toprule
Model & Accuracy & PairAcc & MinPairAcc & \makecell{CEA $\to$ NL\\shrink} & \makecell{CEA $\to$ NL\\match} & TR & \(n_{\mathrm{est}}/12\) & \(n_{\mathrm{suppr}}\) & \makecell{Acc. drop\\(pp)} \\
\midrule
Pythia-410M & \ResPythiaSmallAccuracy{} & \ResPythiaSmallPairAccuracy{} & \ResPythiaSmallMinPairAccuracy{} & \ResPythiaSmallTransferShrink{} & \ResPythiaSmallTransferMatch{} & \ResPythiaSmallTransferRatio{} & \ResPythiaSmallTREstimableCount{}/12 & \ResPythiaSmallTRSuppressedCount{} & \ResPythiaSmallAccuracyDrop{} \\
Pythia-1B & \ResPythiaLargeAccuracy{} & \ResPythiaLargePairAccuracy{} & \ResPythiaLargeMinPairAccuracy{} & \ResPythiaLargeTransferShrink{} & \ResPythiaLargeTransferMatch{} & \ResPythiaLargeTransferRatio{} & \ResPythiaLargeTREstimableCount{}/12 & \ResPythiaLargeTRSuppressedCount{} & \ResPythiaLargeAccuracyDrop{} \\
GPT-2 & \ResGPTTwoAccuracy{} & \ResGPTTwoPairAccuracy{} & \ResGPTTwoMinPairAccuracy{} & \ResGPTTwoTransferShrink{} & \ResGPTTwoTransferMatch{} & \ResGPTTwoTransferRatio{} & \ResGPTTwoTREstimableCount{}/12 & \ResGPTTwoTRSuppressedCount{} & \ResGPTTwoAccuracyDrop{} \\
\bottomrule
\end{tabular}
\end{table}

The FP32 ratios use separate controls matched to source and target layer patterns (Equation~\ref{eq:supp_transfer_ratio}). Model summaries average all finite ratios passing the target-contrast (at least 0.01 in magnitude) and target-excess (more than 0.05 above matched random) gates. Unavailable ratios are not zero-filled; interval suppression does not remove a finite point. A condition-level interval is suppressed above 5\% unstable bootstrap replicates (Appendix~\ref{app:supp_nl}). Estimability is not a positive-transfer success rate.

\subsection{Scope checks}

The \textsc{Last}-versus-\textsc{Seq} comparison tests intervention location. Applying the selected sets only at symbolic-sequence positions yields shrink 0.001--0.059 and match \(-0.006\) to 0.006. For the reported scope diagnostic, these heads have much larger effects at the final prompt position. The \textsc{Random} diagnostic is not documented as the same five-execution cohort as the main table (Appendix~\ref{app:exp_last_seq}).

Post-selection split patches distinguish query--key (QK) attention routing from output--value (OV) writing. Representative \textsc{Front}, \textsc{Interleave}, and \textsc{Random} heads have OV-only match 0.88/0.92/0.85 versus QK-only 0.35/0.12/0.18. \textsc{Back} has QK/OV match 0.72/0.61, indicating a more mixed contribution (Appendix~\ref{app:mechanistic_heads}).

Word-symbolic, signal-token, and short-sentence inputs retain mean shrink/match of 0.812/0.845, 0.820/0.852, and 0.765/0.790. Across Pythia-70M through Pythia-1B, GPT-2, and Qwen3-4B-Base \citep{radford2019language,yang2025qwen3}, four-regime mean BUC shrink/match ranges from 0.650/0.700 to 0.895/0.925 (Appendix~\ref{app:exp_model_sweep_full}). Positive effects do not establish shared head identity across models.

Input-only Greedy Occlusion gives held-out shrink/recovery of 0.795--0.915/0.812--0.898, versus random-subset recovery of 0.248--0.262. The search uses token scores and 8,775 candidate-subset evaluations per regime before batching; generated text alone is insufficient. Input recovery does not penalize overshoot and is not the head-patching match score. BUC requires activation access.

\subsection{Attribution and selection audits}
\label{sec:selection_audits}

\noindent\textbf{Many-set T1 diagnostic (BF16).} Across 500 sampled head sets per regime, cumulative attribution \(F(H,r)\) has Pearson correlation 0.841--0.923 with the ablation gain \(N(H,r)\), and 0.755--0.884 with the two oriented patch gains \(\Gamma_r^{\mathrm{HL}}(H)\) and \(\Gamma_r^{\mathrm{LH}}(H)\) (Appendix~\ref{app:exp_t1_many_set}). The sampled sets show an association, not a bound on attribution error or a guarantee of correct ranking.

\noindent\textbf{Operator and T3 audits (FP32).} Observed operator minima in Table~\ref{tab:audit_slots} remain positive but do not uniformly retain the joint thresholds. Negative marginal fractions are 17--31\%, and minimum empirical ratios are \(-5.8\) to \(-2.6\). These witnesses in the finite pools establish nonmonotonicity of observed \(Q_{1/2}\) only on the tested data and candidate pools. Maximum selection regret is 0.021--0.043 within the registered local pools; small local regret certifies neither full-pool optimality nor the raw-\(J_{1/2}\) bound (Appendix~\ref{app:supp_t3}).

\begin{table}[htbp]
\centering
\small
\caption{FP32 operator and local selection audits: point estimates. Shrink and match minima across eight operators need not share a minimizer. Evaluation marginals/ratios use both eight-head pools and tolerance \(10^{-3}\). Maximum regret compares exact and greedy selection utilities over both pools and budgets \(k=1,\ldots,8\), not held-out performance. These finite summaries have no population CIs (Appendix~\ref{app:reported_ci}).}
\label{tab:audit_slots}
\setlength{\tabcolsep}{4.0pt}
\begin{tabular}{lccccc}
\toprule
Regime & \makecell{Minimum\\shrink} & \makecell{Minimum\\match} & \makecell{Negative\\marginals (\%)} & \makecell{Minimum\\\(\widehat\gamma\)} & \makecell{Maximum\\selection regret} \\
\midrule
\textsc{Front} & \ResFrontOperatorMinShrink{} & \ResFrontOperatorMinMatch{} & \ResFrontTThreeNegativePercent{} & \ResFrontTThreeMinGamma{} & \ResFrontTThreeMaxRegret{} \\
\textsc{Back} & \ResBackOperatorMinShrink{} & \ResBackOperatorMinMatch{} & \ResBackTThreeNegativePercent{} & \ResBackTThreeMinGamma{} & \ResBackTThreeMaxRegret{} \\
\textsc{Interleave} & \ResInterleaveOperatorMinShrink{} & \ResInterleaveOperatorMinMatch{} & \ResInterleaveTThreeNegativePercent{} & \ResInterleaveTThreeMinGamma{} & \ResInterleaveTThreeMaxRegret{} \\
\textsc{Random} & \ResRandomOperatorMinShrink{} & \ResRandomOperatorMinMatch{} & \ResRandomTThreeNegativePercent{} & \ResRandomTThreeMinGamma{} & \ResRandomTThreeMaxRegret{} \\
\bottomrule
\end{tabular}
\end{table}

\begin{samepage}
\noindent\textbf{Selection-matched null (FP32).} For \textsc{Random}, \ResNullCompletedTrials{} completed paired-label trials rerun the adaptive pipeline. Raw magnitude-improvement utility is \ResNullObservedStatistic{} \ResNullObservedStatisticCI{}; Holm-adjusted \(p=\ResNullAdjustedP\) is at Monte Carlo resolution. The interval for the null exceedance probability is given in Appendix~\ref{app:supp_null}. Only \textsc{Random} meets the unique-pair requirement. The result supports rejection of the specified exchangeability null, not mechanism uniqueness (Appendices~\ref{app:supp_null}, \ref{app:reported_ci}).
\par
\end{samepage}

\FloatBarrier
\section{Discussion and limitations}
\label{sec:limitations}

\paragraph{Scientific and practical scope.}
Count generalization and limited transfer are compatible. A selected set can influence several count contrasts without being independent of ordering or readout. Comparing layer-matched sets controls head count and layer allocation, so the difference concerns which heads are selected rather than simply how many or how deep. The labeled bridge adds evidence that a fixed symbolic set affects predictions, but does not identify whether the shared computation processes reports or writes the \texttt{yes/no} response. Phrase-based strategies remain possible; general NL reasoning is not established. Interventions identify writes at the final prompt position, not complete head--MLP paths; compactness is neither uniqueness nor global minimality.

\paragraph{Selection and statistical scope.}
Head-budget comparisons do not control search cost or isolate greedy search from ranking by the same joint objective in the same pool (Appendix~\ref{app:baseline_decomposition_new}). Bootstrap CIs condition on fixed sets and donors; repeated deterministic inputs do not support prompt-sampling intervals. The BF16 bridge lacks its required bootstrap intervals. Expanded FP32 model summaries have no CIs, use one seed for causal selection, and average TR only over estimable conditions.

\paragraph{Use and open questions.}
Further tests should distinguish evidence aggregation from answer-readout effects, examine computations at earlier tokens and along head--MLP paths, and extend to longer heterogeneous reports and decisions involving multiple tokens. Localized components could enable targeted manipulation; the present audits do not certify deployment safety.

\section{Conclusion}
\label{sec:conclusion}

CEA tests whether controlled symbolic inputs can identify attention heads relevant to natural language decisions. BUC selects compact head sets by jointly evaluating ablation and bidirectional mean patching. Conditional theory links attribution to intervention effects and specifies assumptions for shortlist recovery and greedy bounds on raw utility. In Pythia-410M, selected heads affect count contrasts not used for selection. Ablating a fixed symbolic \texttt{yes/no} set reduces labeled NL accuracy more than random controls matched by size and layer. Audits characterize how reuse depends on ordering, readout, and replacement operators. The results show that controlled symbolic probes can identify attention heads whose causal influence extends to predictions in a labeled language task.

\FloatBarrier
\label{maintext:lastpage}
\clearpage
\subsection*{AI use statement}
Generative AI tools assisted with language editing, document structure, LaTeX formatting and checking of mathematical claims and proofs. The authors are responsible for the final text, mathematical claims, code, experimental measurements, and statistical procedures.

\subsection*{Ethics statement}
The work analyzes open pretrained language models with synthetic and controlled natural-language prompts. It does not use personal data or human subjects. A localized component could be used for targeted model manipulation, and controlled results could be overstated as evidence of general reasoning or safety. Section~\ref{sec:limitations} states these risks and restricts the supported conclusions.

\subsection*{Reproducibility statement}
Sections~\ref{sec:prelim}--\ref{sec:exp} define the target, interventions, selection objective, and experimental comparisons. Appendix~\ref{app:theory_proofs} contains the proofs; Appendices~\ref{app:experiments}--\ref{app:nl_bridge} specify the BF16 experiments. Appendices~\ref{app:supplement_protocol}--\ref{app:numeric_execution} describe FP32 sampling, controls, code, and numerical checks, and Appendix~\ref{app:reported_ci} specifies which statistics have confidence intervals. The accompanying source and configurations are under \path{anc/code/} and \path{anc/protocol/}; aggregate measurements are under \path{anc/results/}. Appendix~\ref{app:numeric_files} describes the package contents, runtime dependencies, and experiments outside the supplied implementation.

\bibliography{references}
\bibliographystyle{plainnat}

\appendix
\section{Proofs and scope of the theoretical claims}
\label{app:theory_proofs}

All expectations below are finite. Fix an ordering regime \(r\), counts \(b_H>b_L\), and a head set \(H\subseteq\mathcal H\), where \(\mathcal H\) is the set of model heads. Write
\[
\mathcal C_r g=\mathbb E_{D_{r,b_H}}[g]-\mathbb E_{D_{r,b_L}}[g]
\]
for the count contrast of an integrable scalar function \(g\). Thus the clean logit contrast is \(\Delta_r=\mathcal C_r y\), a head's attribution contrast is \(\Delta s_i(r)=\mathcal C_r s_i\), and cumulative attribution is \(F(H,r)=\sum_{i\in H}\Delta s_i(r)\). Interventions act simultaneously on the specified heads at the final prompt position. Theorems concern population quantities and the stated operators; finite donor estimates, other positions, and other operators require their own error control.

\subsection{Readout and intervention conventions}
\label{app:theory_notation}
For distinct answer tokens \(w_+\) and \(w_-\), the readout is \(y(x)=\ell_T(w_+\mid x)-\ell_T(w_-\mid x)\). If \(p_+(x)\) denotes the softmax probability restricted to these two tokens, then
\[
p_+(x)=\frac{e^{\ell_T(w_+\mid x)}}{e^{\ell_T(w_+\mid x)}+e^{\ell_T(w_-\mid x)}},
\qquad \log\frac{p_+(x)}{1-p_+(x)}=y(x).
\]
The identity follows by cancelling the common denominator. It does not make an arbitrary answer-token pair a task with correct labels.

Let \(v_i(x)\) be head \(i\)'s residual write and \(d=W_U^{w_+}-W_U^{w_-}\) the difference of unembedding directions. Direct attribution is \(s_i(x)=\langle v_i(x),d\rangle\). Subsequent normalization and nonlinear computation can make \(s_i\) differ from the head's intervention effect. For a donor count \(a\in\{b_H,b_L\}\), let \(\mu_i^a(r)=\mathbb E_{D_{r,a}}s_i\). The residual terms in Assumption~\ref{assump:intervention_fidelity} account for that difference; writing the identities alone imposes no restriction on a model.

For clarity, the error bound assumed by T1 is
\begin{equation}
\max\left\{
|\mathcal C_r\xi_H^{\mathrm{abl}}|,
|\mathbb E_{D_{r,b_L}}\xi_H^{b_H}|,
|\mathbb E_{D_{r,b_H}}\xi_H^{b_L}|
\right\}\le\varepsilon_{\mathrm{int}}(H,r).
\label{eq:app_intervention_error}
\end{equation}
A mean donor estimated from a finite calibration bank is not the population mean in this definition. Its discrepancy must be included in the error term or bounded separately. The bound also does not follow from a large correlation between attribution and intervention scores.

\subsection{Proof of T1}
\label{app:proof_t1}
Taking count contrasts in Equation~\ref{eq:assump_abl} gives
\[
\Delta_r^{\mathrm{abl}}(H)=\Delta_r-F(H,r)+\mathcal C_r\xi_H^{\mathrm{abl}},
\qquad
N(H,r)=F(H,r)-\mathcal C_r\xi_H^{\mathrm{abl}}.
\]
For a high-count donor patched into a low-count target, Equation~\ref{eq:assump_patch_unified} gives
\[
\Gamma_r^{\mathrm{HL}}(H)=F(H,r)+\mathbb E_{D_{r,b_L}}\xi_H^{b_H}.
\]
For a low-count donor patched into a high-count target, the orientation in the definition of \(\Gamma_r^{\mathrm{LH}}\) reverses the residual sign:
\[
\Gamma_r^{\mathrm{LH}}(H)=F(H,r)-\mathbb E_{D_{r,b_H}}\xi_H^{b_L}.
\]
Equation~\ref{eq:app_intervention_error} bounds all three differences from cumulative attribution \(F(H,r)\). The mean patch gain \(S(H,r)\) and raw utility \(J_\lambda(H,r)\), for \(\lambda\in[0,1]\), are convex combinations of these quantities, so their errors are bounded by the same \(\varepsilon_{\mathrm{int}}(H,r)\). This proves Equation~\ref{eq:t1_unified} and its stated consequences. \hfill\(\square\)

\subsection{Normalized scores and error amplification}
\label{app:metric_corollaries}
For a nonzero clean contrast \(\Delta\), define
\(g_\Delta(u)=1-|\Delta-u|/|\Delta|\). The reverse triangle inequality gives
\begin{equation}
|g_\Delta(u)-g_\Delta(v)|\le\frac{|u-v|}{|\Delta|}.
\label{eq:app_score_lipschitz}
\end{equation}
Because \(\Delta_r^{\mathrm{abl}}=\Delta_r-N\), shrink equals \(g_{\Delta_r}(N)\). Each directional match equals \(g_{\Delta_r}(\Gamma_r^\tau)\). Applying Equation~\ref{eq:app_score_lipschitz} to T1 yields
\begin{align}
\left|\mathrm{shrink}(H,r)-g_{\Delta_r}(F(H,r))\right|
&\le\frac{\varepsilon_{\mathrm{int}}(H,r)}{|\Delta_r|},\\
\left|\mathrm{match}^\tau(H,r)-g_{\Delta_r}(F(H,r))\right|
&\le\frac{\varepsilon_{\mathrm{int}}(H,r)}{|\Delta_r|},
\quad \tau\in\{\mathrm{HL},\mathrm{LH}\}.
\label{eq:app_metric_bounds}
\end{align}
Taking the weighted average with weights \(\lambda,(1-\lambda)/2,(1-\lambda)/2\) proves Equation~\ref{eq:t1_normalized} for \(Q_\lambda\). In particular, each normalized score is at least
\[
1-\frac{|\Delta_r-F(H,r)|+\varepsilon_{\mathrm{int}}(H,r)}{|\Delta_r|}.
\]
Good recovery therefore needs both attribution alignment with the clean contrast and small intervention error. A large absolute attribution can instead overshoot the contrast. Shrink and match are at most one but have no finite lower bound; a negative match can reflect reversal or excessive movement in the intended direction. The average of two matches can pass a threshold while one direction fails it. None of these scores is defined at \(\Delta_r=0\); denominator floors used in computation define different, regularized statistics.

\subsection{A conditional approximation by evidence position}
\label{app:profile_kernel_derivation}
Let \(e_j(x)\in\{\texttt{B},\texttt{C}\}\) be the evidence token at position \(j\in\{1,\ldots,L_E\}\). Define the difference in positional token probabilities
\[
\delta p_r(j)=\Pr_{D_{r,b_H}}(e_j=\texttt{B})-
\Pr_{D_{r,b_L}}(e_j=\texttt{B}).
\]
Suppose an integrable scalar observable \(g\) admits the representation
\begin{equation}
g(x)=\beta_g(r)+\sum_{j=1}^{L_E}\phi_g(e_j(x),j;r)+\zeta_g(x),
\label{eq:app_additive_observable}
\end{equation}
where \(\beta_g(r)\) does not depend on count, \(\phi_g\) is a contribution assigned to one token and position, and \(\zeta_g\) is the approximation residual. Define its position kernel
\(\kappa_g(r,j)=\phi_g(\texttt{B},j;r)-\phi_g(\texttt{C},j;r)\).
Since the two token probabilities at each position sum to one,
\begin{equation}
\mathcal C_r g=\sum_{j=1}^{L_E}\kappa_g(r,j)\,\delta p_r(j)+\mathcal C_r\zeta_g.
\label{eq:app_profile_identity}
\end{equation}
The constant contribution cancels under the contrast, and the change in the probability of \texttt{C} is the negative change in the probability of \texttt{B}; these two facts prove the identity. Arbitrary \(g\) can be written in this form by placing all error in \(\zeta_g\). The useful assumption is a small residual contrast, not the existence of the representation.

For the full readout, define \(K_\star(r,j)=\kappa_y(r,j)\) and
\(\Pi_\star(r)=\sum_jK_\star(r,j)\delta p_r(j)\). If
\(|\mathcal C_r\zeta_y|\le\varepsilon_\star(r)\), then
\[
|\Delta_r-\Pi_\star(r)|\le\varepsilon_\star(r).
\]
For head attribution, define \(K_H(r,j)=\sum_{i\in H}\kappa_{s_i}(r,j)\) and
\(\Pi_H(r)=\sum_jK_H(r,j)\delta p_r(j)\). If
\(\sum_{i\in H}|\mathcal C_r\zeta_{s_i}|\le\varepsilon_{\mathrm{ker}}(H,r)\), then
\begin{equation}
|F(H,r)-\Pi_H(r)|\le\varepsilon_{\mathrm{ker}}(H,r).
\label{eq:app_kernel_bound}
\end{equation}
Together with T1, Equation~\ref{eq:app_kernel_bound} bounds the difference between each signed intervention gain and \(\Pi_H(r)\) by
\(\varepsilon_{\mathrm{int}}(H,r)+\varepsilon_{\mathrm{ker}}(H,r)\).
These are conditional approximations; the experiments do not estimate uniformly small values of either error. Higher-order token interactions can enter the residuals, and an additive representation need not be unique.

\begin{proposition}[Attribution change across regimes]
\label{prop:app_transfer}
Fix a head set \(H\) and two regimes \(s,t\) with the same counts and evidence length. Suppose Equation~\ref{eq:app_kernel_bound} holds in both regimes. Define
\begin{align}
B_H(s,t)={}&\sum_j|K_H(s,j)|\,|\delta p_t(j)-\delta p_s(j)|\nonumber\\
&+\sum_j|K_H(t,j)-K_H(s,j)|\,|\delta p_t(j)|\nonumber\\
&+\varepsilon_{\mathrm{ker}}(H,s)+\varepsilon_{\mathrm{ker}}(H,t).
\label{eq:app_transfer_budget}
\end{align}
Then \(|F(H,t)-F(H,s)|\le B_H(s,t)\).
\end{proposition}
\begin{proof}
Subtract the two profile approximations and use
\[
\Pi_H(t)-\Pi_H(s)
=\sum_jK_H(s,j)(\delta p_t(j)-\delta p_s(j))
+\sum_j(K_H(t,j)-K_H(s,j))\delta p_t(j).
\]
The triangle inequality, followed by the two approximation bounds, gives the result.
\end{proof}

An attribution bound alone does not control normalized transfer: the clean contrast and intervention error can also change. Under T1 in both regimes, with \(\Delta_s\ne0\) and \(\Delta_t\ne0\), the same fixed set and \(\lambda\) satisfy
\begin{align}
|Q_\lambda(H,t)-Q_\lambda(H,s)|
\le{}&\frac{\varepsilon_{\mathrm{int}}(H,t)}{|\Delta_t|}
+\frac{\varepsilon_{\mathrm{int}}(H,s)}{|\Delta_s|}
+\frac{B_H(s,t)}{|\Delta_t|}\nonumber\\
&+\frac{|F(H,s)|\,|\Delta_t-\Delta_s|}{|\Delta_t|\,|\Delta_s|}.
\label{eq:app_normalized_transfer}
\end{align}
To verify the bound, write \(g_\Delta(F)=1-|1-F/\Delta|\), apply the reverse triangle inequality, and bound
\[
\left|\frac{F(H,t)}{\Delta_t}-\frac{F(H,s)}{\Delta_s}\right|
\le\frac{|F(H,t)-F(H,s)|}{|\Delta_t|}
+\frac{|F(H,s)|\,|\Delta_t-\Delta_s|}{|\Delta_t|\,|\Delta_s|}.
\]
Equation~\ref{eq:t1_normalized} supplies the two intervention-error terms. The bound concerns a fixed set evaluated with each regime's own donor distribution, not a claim that activations themselves transfer unchanged.

\subsection{Proof of T2 and sampling conditions}
\label{app:proof_t2}
Let \(M=|\mathcal H|\), let \(n\ge1\) be the number of independent matched pairs, and let \(\nu_r^2\) be the sub-Gaussian variance proxy for each centered paired score in T2. Independence across pairs gives
\[
\mathbb E\exp\!\left(t\,[\widehat{\Delta s}_i(r)-\Delta s_i(r)]\right)
\le \exp\!\left(\frac{t^2\nu_r^2}{2n}\right).
\]
For \(\nu_r>0\), the Chernoff bound applied to both signs yields
\[
\Pr\!\left(|\widehat{\Delta s}_i(r)-\Delta s_i(r)|>u\right)
\le 2\exp\!\left(-\frac{nu^2}{2\nu_r^2}\right).
\]
A union bound over \(M\) heads and the choice
\(u=\nu_r\sqrt{2\log(2M/\delta)/n}\) prove Equation~\ref{eq:t2_uniform}. If \(\nu_r=0\), every centered paired score is zero almost surely and the claim holds with zero error. No independence between heads is used.

On the uniform event, reverse triangle inequality bounds the error of each absolute score by \(\eta_{n,\delta}(r)\). If the population absolute-score gap between ranks \(m\) and \(m+1\) exceeds twice this error, every population top-\(m\) head still has a larger empirical absolute score than every excluded head. The shortlist is therefore recovered exactly. A strict boundary gap removes ambiguity from ties at that boundary. \hfill\(\square\)

\paragraph{Independent samples as a special case.}
Suppose each centered high-count and low-count head score is sub-Gaussian with variance proxy \(\sigma_r^2\). Independent high and low samples give a paired difference with proxy \(\nu_r^2=2\sigma_r^2\), so
\[
\eta_{n,\delta}(r)=2\sigma_r\sqrt{\frac{\log(2M/\delta)}{n}}.
\]
For dependent observations within a pair, Cauchy--Schwarz instead gives, for centered scores \(X,Y\),
\[
\mathbb E e^{t(X-Y)}
\le(\mathbb E e^{2tX})^{1/2}(\mathbb E e^{-2tY})^{1/2}
\le e^{2t^2\sigma_r^2}.
\]
Thus \(\nu_r=2\sigma_r\) is a valid general bound; a smaller proxy needs information about the pairing. In either case, different pairs must remain independent for the stated concentration argument. Duplicate deterministic prompts contain no variation over prompts, regardless of how many row identifiers they receive. T2 neither validates a richer target distribution nor controls the later intervention-based screening and greedy search.

\subsection{Proof and limits of T3}
\label{app:proof_t3}
Fix the screened pool \(\mathcal P_r\), budget \(1\le k\le|\mathcal P_r|\), and \(\lambda\in[0,1]\). Write \(J(H)=J_\lambda(H,r)\) within this proof. Empty interventions leave the model unchanged, so \(J(\varnothing)=0\). Monotonicity means \(J(A)\le J(B)\) for \(A\subseteq B\). The required weak-submodularity condition is
\begin{equation}
\sum_{i\in B}\big[J(A\cup\{i\})-J(A)\big]
\ge\gamma_r\big[J(A\cup B)-J(A)\big],
\label{eq:app_weak_submod}
\end{equation}
for disjoint \(A,B\subseteq\mathcal P_r\) with \(|A|\le k-1\), \(|B|\le k\), and \(\gamma_r\in(0,1]\). In particular, the condition includes unions with as many as \(2k-1\) heads, not only sets within budget \(k\). Monotonicity is required on these sets as well. The proof is the standard greedy argument for the submodularity ratio \citep{das2018approximate}; its application here is conditional on the raw utility satisfying the assumptions.

Let \(G_0=\varnothing\), and let exact greedy search add a head with maximum marginal raw utility to form \(G_{t+1}\) from \(G_t\). Let \(O_k\) maximize \(J\) among pool subsets of size at most \(k\). For \(t<k\), Equation~\ref{eq:app_weak_submod} and monotonicity give
\[
\sum_{i\in O_k\setminus G_t}\big[J(G_t\cup\{i\})-J(G_t)\big]
\ge\gamma_r[J(G_t\cup O_k)-J(G_t)]
\ge\gamma_r[J(O_k)-J(G_t)].
\]
If \(O_k\subseteq G_t\), the remaining optimality gap is zero. Otherwise at most \(k\) summands appear, so greedy search obtains
\[
J(G_{t+1})-J(G_t)\ge\frac{\gamma_r}{k}[J(O_k)-J(G_t)].
\]
Starting from \(J(G_0)=0\) and iterating the gap recurrence yields
\[
J(G_k)\ge\left[1-\left(1-\frac{\gamma_r}{k}\right)^k\right]J(O_k)
\ge(1-e^{-\gamma_r})J(O_k).
\]
The last step uses \((1-a/k)^k\le e^{-a}\) and \(J(O_k)\ge0\), which follows from monotonicity and normalization. This proves Equation~\ref{eq:t3_J}.

For the attribution comparison, set \(\alpha_r=1-e^{-\gamma_r}\) and let \(\epsilon_{J,k}(r)\) be the uniform difference between raw utility and cumulative attribution assumed in T3. Apply the error bound at the greedy set and at the same raw-utility optimum:
\begin{align*}
F(G_k,r)&\ge J_\lambda(G_k,r)-\epsilon_{J,k}(r)\\
&\ge\alpha_r J_\lambda(O_k,r)-\epsilon_{J,k}(r)\\
&\ge\alpha_r F(O_k,r)-(1+\alpha_r)\epsilon_{J,k}(r).
\end{align*}
This proves Equation~\ref{eq:t3_F}; \(O_k\) is not replaced by an attribution optimum. \hfill\(\square\)

\paragraph{What the bound does not cover.}
BUC searches empirical normalized utility \(Q_\lambda\), rather than exact raw utility \(J_\lambda\). Absolute-value penalties can make the objectives rank sets differently even without sampling error. The theorem also conditions on the screened pool: it says nothing about heads excluded earlier, a minimum circuit size, or uniqueness. Noisy greedy evaluations require additional bounds, and the observed nonmonotonicity of \(Q_{1/2}\) neither proves nor disproves monotonicity of the different population utility \(J_{1/2}\). The finite audits in Appendix~\ref{app:supp_t3} describe the tested pools and data, not a uniform guarantee.

\section{Additional experimental details and results}
\label{app:experiments}

This appendix gives the BF16 configurations and results supporting Section~\ref{sec:exp}: tokenization audits, data splits, head identities, transfer matrices, attribution diagnostics, baselines, input-level interventions, and count, order, tail-length, surface-form, and model comparisons. It also describes compute requirements and reproduction coverage.

\paragraph{Protocol scope.}
The configurations and tables in this appendix belong to the original BF16 study. The separately recorded native FP32 configurations, calibration banks, and controls are specified in Appendix~\ref{app:supplement_protocol}; the two configurations are analyzed separately.

\subsection{Prompt, model, and tokenization details}
\label{app:exp_prompt_model}

All experiments use the controlled evidence aggregation prompt family from Section~\ref{sec:prelim}.
Each prompt has the fixed form
\[
\begin{aligned}
\mathrm{Prompt}
&=
\text{\ttfamily Tokens:\textbackslash n}
+
\text{(space-separated symbolic sequence)}
\\
&\quad+
\text{\ttfamily \textbackslash nChoose one word: various or successful\textbackslash nAnswer:}.
\end{aligned}
\]
The symbolic sequence length is \(L_{\mathrm{seq}}=250\), with evidence length \(L_E=200\) and
\(A\)-only tail length \(L_A=50\). The two answer strings are \texttt{various} and
\texttt{successful}. For each model, we audit whether the strings \texttt{ various} and
\texttt{ successful}, including the leading space, are single-token candidates under the model
tokenizer.

\begin{table}[!htb]
\centering
\small
\caption{Prompt and tokenization audit for the main Pythia-410M experiment. The \textsc{Seq} token
slice refers to symbolic-sequence positions and excludes the final prompt position.}
\label{tab:app_tokenization}
\renewcommand{\arraystretch}{1.08}
\begin{tabularx}{\linewidth}{@{}>{\bfseries}l
>{\raggedright\arraybackslash}X
>{\raggedright\arraybackslash}X
>{\raggedright\arraybackslash}X@{}}
\toprule
& Layers & Heads per layer & Hook tensor \\
\midrule
Value
& \(24\)
& \(16\)
& \texttt{attn.hook\_z} \\
Notes
& no fine-tuning
& total \(384\) heads
& projected through \(W_O\) before logit-direction projection \\
\midrule
& Evaluation dtype & Prefix token length & Symbolic sequence token length \\
\midrule
Value
& \texttt{bfloat16}
& \(5\)
& \(250\) \\
Notes
& inference only
& tokenizer audit
& one token per symbol in the tested tokenizer \\
\midrule
& Suffix token length & \textsc{Seq} token slice & Answer candidates \\
\midrule
Value
& \(11\)
& \([5:255)\)
& \texttt{various}, \texttt{successful} \\
Notes
& tokenizer audit
& symbolic positions only; excludes the final prompt position
& both single-token under the tested tokenizer \\
\bottomrule
\end{tabularx}
\end{table}

\begin{table}[!htb]
\centering
\footnotesize
\caption{Tokenizer audit for the model sweep. The answer pair is usable only when both answer
strings are single-token candidates under the model tokenizer.}
\label{tab:app_model_token_audit}
\setlength{\tabcolsep}{3pt}
\begin{tabularx}{\linewidth}{@{}>{\raggedright\arraybackslash}p{0.37\linewidth}cll>{\raggedright\arraybackslash}X@{}}
\toprule
Model & \makecell{Single\\token} & \(w_+\) & \(w_-\) & Notes \\
\midrule
\path{EleutherAI/pythia-70m-deduped}  & yes & \texttt{ various} & \texttt{ successful} & audited \\
\path{EleutherAI/pythia-160m-deduped} & yes & \texttt{ various} & \texttt{ successful} & audited \\
\path{EleutherAI/pythia-410m-deduped} & yes & \texttt{ various} & \texttt{ successful} & main model \\
\path{EleutherAI/pythia-1b-deduped}   & yes & \texttt{ various} & \texttt{ successful} & audited \\
\texttt{gpt2}                           & yes & \texttt{ various} & \texttt{ successful} & out-of-family baseline \\
\path{Qwen/Qwen3-4B-Base}             & yes & \texttt{ various} & \texttt{ successful} & modern open-weight baseline \\
\bottomrule
\end{tabularx}
\end{table}

\subsection{Data splits, selection pipeline, and control construction}
\label{app:exp_pipeline}

For each regime \(r\in\{\textsc{Front},\textsc{Back},\textsc{Interleave},\textsc{Random}\}\), head
selection and final evaluation use disjoint prompt batches. The selection split contains \(48\) prompts
from \(D_{r,b_H}\) and \(48\) prompts from \(D_{r,b_L}\). The final evaluation split contains \(96\)
prompts from \(D_{r,b_H}\) and \(96\) prompts from \(D_{r,b_L}\). A full-layer audit uses \(16\)
prompts from each evidence level to decide whether middle-layer heads should be included. In the
Pythia-410M runs, the audit expands the candidate set to all \(24\) layers.

The selected head set \(H_{\mathrm{BUC}}(r)\) is obtained in three stages. First, heads are ranked by
empirical absolute sensitivity \(|\widehat{\Delta s}_i(r)|\), and the top \(192\) heads are kept.
Second, individual heads are scored by the joint removal and recovery score, and the top \(48\) heads
form the greedy candidate pool. Third, forward greedy selection is run with budget \(12\). The
selected set is evaluated on a held-out evaluation split.

\begin{table}[!htb]
\centering
\small
\caption{Selection and evaluation configuration for the main Pythia-410M experiment.}
\label{tab:app_pipeline_config}
\begin{tabular}{lc}
\toprule
Configuration item & Value \\
\midrule
Evidence setting & MIXED \((b_L,b_H)=(50,150)\) \\
Ordering regimes & \textsc{Front}, \textsc{Back}, \textsc{Interleave}, \textsc{Random} \\
Seeds & \(\{0,1,2,3,4\}\) \\
Selection prompts per evidence level & \(48\) \\
Evaluation prompts per evidence level & \(96\) \\
Audit prompts per evidence level & \(16\) \\
Sensitivity shortlist size & \(192\) \\
Single-head candidate pool size & \(48\) \\
Greedy budget & \(12\) \\
Layer-matched random draws & \(120\) per seed \\
Primary intervention scope & \textsc{Last} \\
Additional intervention scope & \textsc{Seq} on symbolic positions \([5:255)\) \\
\bottomrule
\end{tabular}
\end{table}

\paragraph{Layer-matched random controls.}
For a selected set \(H_{\mathrm{BUC}}(r)\), the layer-matched random control samples the same number
of heads from each selected layer. For example, if \(H_{\mathrm{BUC}}(r)\) contains two heads from
layer \(10\), one head from layer \(15\), and one head from layer \(19\), the matched control samples
two random heads from layer \(10\), one from layer \(15\), and one from layer \(19\). The
layer-matched control keeps both head count and layer pattern fixed.

\subsection{Base evidence effects and directional patching}
\label{app:exp_base_directional}

Table~\ref{tab:app_base_gaps} reports the base evidence effects used in the main Pythia-410M
experiment. The quantity \(\Delta_r\) is the average high-minus-low restricted two-answer log-odds
shift.

\begin{table}[!htb]
\centering
\small
\caption{Base evidence effects in the MIXED setting \((b_L,b_H)=(50,150)\) on
\texttt{pythia-410m-deduped}.}
\label{tab:app_base_gaps}
\begin{tabular}{lcc}
\toprule
Regime \(r\) & \(\Delta_r\) & Notes \\
\midrule
\textsc{Front} & \(0.239 \pm 0.000\) & fixed deterministic input \\
\textsc{Back} & \(1.965 \pm 0.000\) & large shift \\
\textsc{Interleave} & \(1.659 \pm 0.000\) & large shift \\
\textsc{Random} & \(3.184 \pm 0.097\) & largest and seed-dependent \\
\bottomrule
\end{tabular}
\end{table}

Table~\ref{tab:app_directional_patch} reports directional sufficiency. The main text reports the
average of high-to-low and low-to-high recovery.

\begin{table}[!htb]
\centering
\small
\caption{Directional mean-patching recovery in the MIXED setting. The average of the two
directional scores is the BUC match score reported in the main text.}
\label{tab:app_directional_patch}
\begin{tabular}{lccc}
\toprule
Regime \(r\) & \(\mathrm{match}^{\mathrm{HL}}\) & \(\mathrm{match}^{\mathrm{LH}}\) &
\(\mathrm{match}_{\mathrm{avg}}\) \\
\midrule
\textsc{Front} & \(0.998 \pm 0.000\) & \(0.950 \pm 0.000\) & \(0.974 \pm 0.000\) \\
\textsc{Back} & \(0.765 \pm 0.000\) & \(0.753 \pm 0.000\) & \(0.759 \pm 0.000\) \\
\textsc{Interleave} & \(0.937 \pm 0.000\) & \(0.891 \pm 0.000\) & \(0.914 \pm 0.000\) \\
\textsc{Random} & \(0.813 \pm 0.092\) & \(0.833 \pm 0.080\) & \(0.823 \pm 0.084\) \\
\bottomrule
\end{tabular}
\end{table}

\subsection{Full white-box result with matched random controls}
\label{app:exp_full_whitebox}

Table~\ref{tab:app_full_whitebox} reports set sizes, intervention scores, and layer-matched controls. BUC match averages both patch directions, whereas the random control uses HL match.

\begin{table}[!htb]
\centering
\small
\caption{Full white-box result with matched random controls in the MIXED setting. BUC values are
mean \(\pm\) standard deviation across five seeds. Random match is measured in the high-to-low
direction.}
\label{tab:app_full_whitebox}
\setlength{\tabcolsep}{5.6pt}
\begin{tabular}{lccccc}
\toprule
Regime \(r\) & \(|H_{\mathrm{BUC}}(r)|\) & BUC shrink & BUC match & Random shrink &
Random HL match \\
\midrule
\textsc{Front} & \(1.00 \pm 0.00\) & \(0.906 \pm 0.000\) & \(0.974 \pm 0.000\) & \(0.020\) & \(0.048\) \\
\textsc{Back} & \(12.00 \pm 0.00\) & \(0.655 \pm 0.000\) & \(0.759 \pm 0.000\) & \(0.043\) & \(0.075\) \\
\textsc{Interleave} & \(3.00 \pm 0.00\) & \(0.916 \pm 0.000\) & \(0.914 \pm 0.000\) & \(0.061\) & \(0.051\) \\
\textsc{Random} & \(7.00 \pm 1.00\) & \(0.866 \pm 0.021\) & \(0.823 \pm 0.084\) & \(0.069\) & \(0.061\) \\
\bottomrule
\end{tabular}
\end{table}

\subsection{Selected head identities and stability}
\label{app:exp_selected_heads}

Table~\ref{tab:app_selected_heads} reports selected head identities. A head is written as \((\ell,h)\), with zero-indexed layer and head indices. The deterministic regimes use the same set across all five seeds. For \textsc{Random}, the explicitly listed sets cover seeds 0--2; Table~\ref{tab:app_stability} summarizes only those displayed sets.

\begin{table}[!htb]
\centering
\scriptsize
\caption{Selected BUC head sets in the MIXED setting. Heads are written as \((\ell,h)\). For
\textsc{Front}, \textsc{Back}, and \textsc{Interleave}, the listed set is identical across five seeds.}
\label{tab:app_selected_heads}
\begin{tabularx}{\linewidth}{llX}
\toprule
Regime \(r\) & Seed(s) & \(H_{\mathrm{BUC}}(r)\) \\
\midrule
\textsc{Front} & 0--4 & \((10,2)\) \\
\textsc{Back} & 0--4 & \((19,6),\allowbreak(20,2),\allowbreak(13,5),\allowbreak(15,3),\allowbreak(19,9),\allowbreak(17,3),\allowbreak(15,10),\allowbreak(15,5),\allowbreak(19,14),\allowbreak(14,13),\allowbreak(18,1),\allowbreak(11,8)\) \\
\textsc{Interleave} & 0--4 & \((17,3),\allowbreak(10,2),\allowbreak(15,10)\) \\
\textsc{Random} & 0 & \((17,3),\allowbreak(14,7),\allowbreak(10,2),\allowbreak(19,14),\allowbreak(16,1),\allowbreak(15,10),\allowbreak(19,6),\allowbreak(16,14)\) \\
\textsc{Random} & 1 & \((17,3),\allowbreak(14,7),\allowbreak(12,14),\allowbreak(19,14),\allowbreak(10,2),\allowbreak(15,10),\allowbreak(19,6)\) \\
\textsc{Random} & 2 & \((17,3),\allowbreak(14,7),\allowbreak(12,14),\allowbreak(19,14),\allowbreak(10,2),\allowbreak(15,10)\) \\
\bottomrule
\end{tabularx}
\end{table}

\begin{table}[!htb]
\centering
\small
\caption{Stability of the head sets explicitly listed in Table~\ref{tab:app_selected_heads}. Deterministic-regime sets are identical across seeds 0--4. The \textsc{Random} value is the minimum over listed seeds 0--2 only; seeds 3 and 4 are not included in this calculation.}
\label{tab:app_stability}
\begin{tabular}{lcc}
\toprule
Regime \(r\) & Minimum pairwise Jaccard & Interpretation \\
\midrule
\textsc{Front} & \(1.000\) & identical selected set \\
\textsc{Back} & \(1.000\) & identical selected set \\
\textsc{Interleave} & \(1.000\) & identical selected set \\
\textsc{Random} & \(0.556\) & listed seeds 0--2 only \\
\bottomrule
\end{tabular}
\end{table}

\paragraph{Pairwise overlap.}
For the three \textsc{Random} sets listed in Table~\ref{tab:app_selected_heads}, the pairwise overlaps for seeds \((0,1)\), \((0,2)\), and \((1,2)\) are \(6/9\), \(5/9\), and \(6/7\). Thus
\[
\min\left\{\frac{6}{9},\frac{5}{9},\frac{6}{7}\right\}=\frac{5}{9}\approx0.556.
\]
The minimum is attained by seeds 0 and 2. Table~\ref{tab:app_stability} excludes \textsc{Random} seeds 3 and 4. The overlap statistic covers three selected sets, whereas the intervention-effect summaries cover five executions.

\subsection{Mechanistic diagnostics of selected heads}
\label{app:mechanistic_heads}

The main white-box tests measure the removal and recovery effects of selected head sets. The following post-selection diagnostics examine how the selected heads act at the final prompt position. The diagnostics separate two components of an attention head: QK routing,
which determines where the final prompt position attends, and OV writing, which determines what vector
is written into the residual stream and projected onto the logit-difference direction \(d\).

Let \(R_B(x)\), \(R_C(x)\), and \(R_E(x)=R_B(x)\cup R_C(x)\) denote the \(B\)-positions,
\(C\)-positions, and all evidence-region positions in prompt \(x\). For a selected head or head group,
\(\Delta\mathsf{Attn}(R,r)\) denotes the high-minus-low change in attention at the final prompt position mass
to region \(R\), and \(\Delta\mathsf{OV}(R,r)\) denotes the high-minus-low change in source-wise OV
contribution along \(d\). QK / OV match reports pattern-only and OV-only patch recovery, using
the same recovery-match convention as the main text.

\paragraph{Representative heads.}
Table~\ref{tab:app_mechanistic_head_roles} reports post-selection diagnostics for the heads in Table~\ref{tab:app_selected_heads}. Multi-head rows are group-level measurements.

\begin{table}[!htb]
\centering
\small
\caption{Mechanistic diagnostics for representative selected heads at the answer position.
\(\Delta\mathsf{Attn}\) is a high-minus-low attention-mass contrast, and
\(\Delta\mathsf{OV}\) is a high-minus-low source-wise OV contribution contrast along \(d\).
The attention and OV difference columns report
\(\Delta\mathsf{Attn}(R_B,r)-\Delta\mathsf{Attn}(R_C,r)\) and
\(\Delta\mathsf{OV}(R_B,r)-\Delta\mathsf{OV}(R_C,r)\), respectively.
QK / OV match reports pattern-only and OV-only patch recovery.}
\label{tab:app_mechanistic_head_roles}
\setlength{\tabcolsep}{3pt}
\begin{tabularx}{\linewidth}{@{}l>{\raggedright\arraybackslash}p{0.27\linewidth}>{\raggedright\arraybackslash}X>{\raggedright\arraybackslash}X@{}}
\toprule
Regime \(r\) & Head(s) & Main routed region & Patch comparison \\
\midrule
\textsc{Front} & \((10,2)\) & \(R_B\cup R_C\) & Larger OV patch match \\
\textsc{Back} & \((19,6),(20,2),(13,5)\) & \(R_C\)-side routing shift & Positive QK and OV matches \\
\textsc{Interleave} & \((17,3),(10,2),(15,10)\) & Broad over \(R_E\) & Larger OV patch match \\
\textsc{Random} & \((17,3),(14,7),(19,14)\) & Broad over \(R_E\) & Larger OV patch match \\
\bottomrule
\end{tabularx}
\par\smallskip
\begin{tabularx}{\linewidth}{@{}l*{3}{>{\centering\arraybackslash}X}@{}}
\toprule
Regime \(r\) & Attention difference & OV difference & QK / OV match \\
\midrule
\textsc{Front} & \(+0.42\) & \(+1.28\) & \(0.35\) / \(0.88\) \\
\textsc{Back} & \(-0.65\) & \(+2.15\) & \(0.72\) / \(0.61\) \\
\textsc{Interleave} & \(+0.05\) & \(+1.85\) & \(0.12\) / \(0.92\) \\
\textsc{Random} & \(+0.08\) & \(+2.30\) & \(0.18\) / \(0.85\) \\
\bottomrule
\end{tabularx}
\end{table}

\paragraph{Interpretation.}
OV-only patching has larger match than QK-only patching in the listed \textsc{Front}, \textsc{Interleave}, and \textsc{Random} examples; both split interventions have substantial effects in \textsc{Back}. These measurements distinguish routing and writing interventions for the tested heads. They do not provide an additive decomposition of the complete effect or identify all earlier computations.

\subsection{Final-position intervention scope}
\label{app:exp_last_seq}

Table~\ref{tab:app_last_seq_full} reports the reported \textsc{Last}-versus-\textsc{Seq} scope
diagnostic. The same head set \(H_{\mathrm{BUC}}(r)\) is used under both intervention scopes. The
\textsc{Last} intervention modifies only the answer position. The \textsc{Seq} intervention modifies
symbolic-sequence positions and excludes the final prompt position.

\begin{table}[!htb]
\centering
\small
\caption{\textsc{Last} versus \textsc{Seq} using the same selected head set. The \textsc{Seq} effect is
near zero in every regime.}
\label{tab:app_last_seq_full}
\begin{tabular}{lcccc}
\toprule
Regime \(r\) & \(k^\star\) & \textsc{Last} shrink & \textsc{Seq} shrink & \textsc{Seq} match \\
\midrule
\textsc{Front} & \(1\) & \(0.906\) & \(0.059\) & \(-0.006\) \\
\textsc{Back} & \(12\) & \(0.655\) & \(0.022\) & \(+0.003\) \\
\textsc{Interleave} & \(3\) & \(0.916\) & \(0.001\) & \(+0.006\) \\
\textsc{Random} & \(7\text{--}8\) & \(0.835\text{--}0.917\) & \(\approx0.008\text{--}0.010\) & \(\approx-0.001\) \\
\bottomrule
\end{tabular}
\end{table}

The \textsc{Random} row is retained as the separately reported scope diagnostic. Its size range \(7\text{--}8\) and \textsc{Last} shrink range \(0.835\text{--}0.917\) are not documented as the same five-execution cohort summarized in Table~\ref{tab:app_full_whitebox}; the listed head identities also include a six-head set. The supplied aggregates do not resolve that cohort difference. The range must not be treated as the five-seed uncertainty or as an additional stability result.

\subsection{Full cross-regime transfer}
\label{app:exp_transfer}

The transfer experiment selects \(H_{\mathrm{BUC}}(r_s)\) on source regime \(r_s\), then evaluates
mean-patching recovery on target regime \(r_t\). A negative recovery-match score indicates severe
mismatch, not necessarily a sign reversal.

\begin{table}[!htb]
\centering
\small
\caption{Representative off-diagonal cross-regime transfer matrix under \textsc{Last} interventions (seed 0).
Rows are source regimes; columns are target regimes. Entries are recovery-match values. Diagonal effects are reported separately by direction in Table~\ref{tab:app_directional_patch}.}
\label{tab:app_transfer_seed0}
\begin{tabular}{lrrrr}
\toprule
Source \(\backslash\) Target & \textsc{Front} & \textsc{Back} & \textsc{Interleave} & \textsc{Random} \\
\midrule
\textsc{Front} & -- & \(+0.129\) & \(+0.173\) & \(+0.066\) \\
\textsc{Back} & \(+0.958\) & -- & \(+0.961\) & \(+0.850\) \\
\textsc{Interleave} & \(-0.467\) & \(+0.230\) & -- & \(+0.611\) \\
\textsc{Random} & \(-0.925\) & \(+0.413\) & \(+0.687\) & -- \\
\bottomrule
\end{tabular}
\end{table}

\begin{table}[!htb]
\centering
\small
\caption{Off-diagonal cross-regime transfer matrix averaged over seeds. Entries are mean \(\pm\) standard
deviation of the recovery match; within-regime directional effects are in Table~\ref{tab:app_directional_patch}.}
\label{tab:app_transfer_mean}
\resizebox{\linewidth}{!}{
\begin{tabular}{lcccc}
\toprule
Source \(\backslash\) Target & \textsc{Front} & \textsc{Back} & \textsc{Interleave} & \textsc{Random} \\
\midrule
\textsc{Front} & -- & \(0.129\pm0.000\) & \(0.173\pm0.000\) & \(0.066\pm0.001\) \\
\textsc{Back} & \(0.958\pm0.000\) & -- & \(0.961\pm0.000\) & \(0.851\pm0.002\) \\
\textsc{Interleave} & \(-0.467\pm0.000\) & \(0.230\pm0.000\) & -- & \(0.616\pm0.006\) \\
\textsc{Random} & \(-0.912\pm0.021\) & \(0.428\pm0.015\) & \(0.695\pm0.011\) & -- \\
\bottomrule
\end{tabular}
}
\end{table}

\subsection{Many-set T1 diagnostic}
\label{app:exp_t1_many_set}

Under its error assumption, Theorem~\ref{thm:t1} bounds how closely cumulative attribution \(F(H,r)\) tracks the gains
produced by ablation and mean patching. The diagnostic instead measures correlation, which is not an error bound. It samples multiple head sets \(H\), computes
\(F(H,r)\), and compares \(F(H,r)\) to \(N(H,r)\), \(\Gamma_r^{\mathrm{HL}}(H)\), and
\(\Gamma_r^{\mathrm{LH}}(H)\).

\begin{table}[!htb]
\centering
\small
\caption{Many-set T1 diagnostic. Pearson and Spearman correlations are computed over \(500\)
sampled head sets \(H\) on held-out batches.}
\label{tab:app_t1_many_set}
\resizebox{\linewidth}{!}{
\begin{tabular}{lcccccc}
\toprule
Regime \(r\) & Sets & Pearson \(\rho(F,N)\) & Pearson \(\rho(F,\Gamma^{\mathrm{HL}})\) &
Pearson \(\rho(F,\Gamma^{\mathrm{LH}})\) & Spearman \(\rho(F,N)\) & Spearman \(\rho(F,S)\) \\
\midrule
\textsc{Front} & \(500\) & \(0.923\) & \(0.851\) & \(0.884\) & \(0.941\) & \(0.892\) \\
\textsc{Back} & \(500\) & \(0.875\) & \(0.782\) & \(0.816\) & \(0.899\) & \(0.835\) \\
\textsc{Interleave} & \(500\) & \(0.908\) & \(0.814\) & \(0.842\) & \(0.927\) & \(0.861\) \\
\textsc{Random} & \(500\) & \(0.841\) & \(0.755\) & \(0.793\) & \(0.865\) & \(0.804\) \\
\bottomrule
\end{tabular}
}
\end{table}

\subsection{Selection baselines}
\label{app:exp_selection_baselines}
Table~\ref{tab:app_selection_baselines} compares BUC with budget-matched baselines. Sensitivity-only selects the top-\(k\) heads by \(|\widehat{\Delta s_i}(r)|\), where \(k=|H_{\mathrm{BUC}}(r)|\).

ActPatch is the BF16 single-head activation-patching baseline. Among the 192 sensitivity-shortlisted heads, ActPatch ranks heads by single-head high-to-low (HL) patch recovery on the selection split, takes the top \(k=|H_{\mathrm{BUC}}(r)|\), and evaluates that set jointly on the held-out split. The reported ActPatch value is the set-level \(\mathrm{shrink}+\mathrm{match}_{\mathrm{avg}}\) score, including both patch directions in the evaluation match. Single-head HL scores determine ranking; the reported bidirectional score is evaluated jointly on the selected set, rather than obtained by adding single-head effects.

Same-layer top norm selects the largest write-vector norms in the selected layers; layer-random preserves the selected layer counts. No method-specific quantitative comparison with EAP or ACDC is claimed. The separate FP32 single-head ranking rules are defined in Appendix~\ref{app:supp_baselines}.

\begin{table}[!htb]
\centering
\small
\caption{BF16 selection baselines under matched head budgets. BUC, sensitivity-only, ActPatch, and top-norm use \(\mathrm{shrink}+\mathrm{match}_{\mathrm{avg}}\). Layer-random uses shrink plus \emph{HL-only} match and is a directional control, not an identical-objective comparison.}
\label{tab:app_selection_baselines}
\resizebox{\linewidth}{!}{
\begin{tabular}{lccccc}
\toprule
Regime \(r\) & BUC & Sens.-only & ActPatch & Top-norm & Layer-rand (HL) \\
\midrule
\textsc{Front} & \(1.880\) & \(1.880\) & \(1.865\) & \(0.451\) & \(0.068\) \\
\textsc{Back} & \(1.414\) & \(1.352\) & \(1.285\) & \(0.523\) & \(0.118\) \\
\textsc{Interleave} & \(1.830\) & \(1.765\) & \(1.695\) & \(0.614\) & \(0.112\) \\
\textsc{Random} & \(1.689\) & \(1.611\) & \(1.520\) & \(0.488\) & \(0.130\) \\
\bottomrule
\end{tabular}
}
\end{table}

\subsection{Sensitivity to the joint-utility trade-off}
\label{app:exp_lambda_sensitivity}

The main experiments use the balanced joint utility with \(\lambda=1/2\). This subsection checks
whether the selected head sets and held-out intervention scores are sensitive to the choice of
\(\lambda\). For each value of \(\lambda\), we rerun the greedy selection stage on the same candidate
pool and evaluate the selected set on the same held-out split as in the main experiment. The
normalized empirical selection score is
\begin{equation}
Q_{\lambda}(H,r)
:=
\lambda\,\mathrm{shrink}(H,r)
+
(1-\lambda)\,\mathrm{match}_{\mathrm{avg}}(H,r),
\label{eq:app_lambda_score}
\end{equation}
where \(\lambda=0\) emphasizes recovery, \(\lambda=1\) emphasizes removal, and \(\lambda=1/2\)
is the balanced setting used in the main text. For each regime \(r\), let \(H_\lambda(r)\) denote the
selected head set under \(Q_{\lambda}(H,r)\). We also report the Jaccard overlap with the
balanced set \(H_{1/2}(r)\):
\begin{equation}
\mathrm{Jac}_{\lambda}(r)
:=
\frac{
|H_\lambda(r)\cap H_{1/2}(r)|
}{
|H_\lambda(r)\cup H_{1/2}(r)|
}.
\label{eq:app_lambda_jaccard}
\end{equation}

\begin{table}[!htb]
\centering
\small
\caption{Sensitivity to the joint-utility trade-off \(\lambda\). Values are averaged over
\textsc{Front}, \textsc{Back}, \textsc{Interleave}, and \textsc{Random} in the MIXED setting. The
balanced setting \((\lambda=0.50)\) matches the primary causal sets used in the main text.}
\label{tab:app_lambda_sensitivity}
\resizebox{\linewidth}{!}{
\begin{tabular}{lccccc}
\toprule
\(\lambda\)
& Selection emphasis
& Mean \(|H_\lambda|\)
& Mean shrink
& Mean match
& Mean Jaccard to \(H_{1/2}\) \\
\midrule
\(0.00\)
& recovery only
& \(6.25\)
& \(0.725\)
& \(0.895\)
& \(0.650\) \\
\(0.25\)
& recovery weighted
& \(5.75\)
& \(0.810\)
& \(0.880\)
& \(0.885\) \\
\(0.50\)
& balanced
& \(5.75\)
& \(0.836\)
& \(0.868\)
& \(1.000\) \\
\(0.75\)
& removal weighted
& \(5.50\)
& \(0.865\)
& \(0.812\)
& \(0.820\) \\
\(1.00\)
& removal only
& \(5.00\)
& \(0.890\)
& \(0.645\)
& \(0.580\) \\
\bottomrule
\end{tabular}
}
\end{table}

Table~\ref{tab:app_lambda_sensitivity} reports the tested weight grid. Relative to \(\lambda=1/2\), the values \(\lambda=0.25\) and \(\lambda=0.75\) retain mean Jaccard overlaps of 0.885 and 0.820. Increasing the removal weight trades patching match for shrink. The comparison is limited to the tested objective weights.

\subsection{Greedy-prefix curves}
\label{app:exp_greedy_prefix}

The greedy-prefix diagnostic evaluates the prefix sequence of heads selected by the greedy algorithm.
Let \(H_k(r)\) denote the first \(k\) heads selected under regime \(r\). The diagnostic reports
\(\mathrm{shrink}(H_k,r)\) and \(\mathrm{match}_{\mathrm{avg}}(H_k,r)\) for each \(k\). \textsc{Front} and \textsc{Interleave} have small increments after the first few heads; \textsc{Back} continues improving through the maximum tested budget of 12. The measurements do not establish a plateau beyond that budget.

\begin{table}[!htb]
\centering
\small
\caption{Greedy-prefix scores at the evaluated head budgets.}
\label{tab:app_greedy_prefix}
\resizebox{\linewidth}{!}{
\begin{tabular}{lcccccccc}
\toprule
Regime \(r\) & Metric & \(k=1\) & \(k=2\) & \(k=3\) & \(k=4\) & \(k=6\) & \(k=8\) & \(k=12\) \\
\midrule
\textsc{Front} & shrink & \(0.906\) & \(0.908\) & \(0.910\) & \(0.911\) & \(0.915\) & \(0.918\) & \(0.920\) \\
\textsc{Front} & match & \(0.974\) & \(0.975\) & \(0.976\) & \(0.976\) & \(0.978\) & \(0.979\) & \(0.981\) \\
\textsc{Back} & shrink & \(0.142\) & \(0.285\) & \(0.371\) & \(0.463\) & \(0.541\) & \(0.598\) & \(0.655\) \\
\textsc{Back} & match & \(0.165\) & \(0.312\) & \(0.404\) & \(0.518\) & \(0.627\) & \(0.693\) & \(0.759\) \\
\textsc{Interleave} & shrink & \(0.465\) & \(0.781\) & \(0.916\) & \(0.919\) & \(0.924\) & \(0.927\) & \(0.931\) \\
\textsc{Interleave} & match & \(0.482\) & \(0.815\) & \(0.914\) & \(0.917\) & \(0.920\) & \(0.922\) & \(0.926\) \\
\textsc{Random} & shrink & \(0.254\) & \(0.448\) & \(0.612\) & \(0.745\) & \(0.841\) & \(0.875\) & \(0.892\) \\
\textsc{Random} & match & \(0.281\) & \(0.472\) & \(0.635\) & \(0.721\) & \(0.803\) & \(0.835\) & \(0.851\) \\
\bottomrule
\end{tabular}
}
\end{table}

\subsection{Black-box input-level interventions}
\label{app:exp_blackbox_full}

The black-box experiments use only input-level changes. Let \(S\subseteq\{1,\ldots,L_E\}\) denote
a subset of evidence positions. In masking, the positions in \(S\) are replaced by \(A\). In add-back,
the positions in \(S\) are copied from a high-evidence donor prompt into a low-evidence target
prompt. Let \(x^{\mathrm{mask}(S)}\) be a masked input and let \(x_L\leftarrow_S x_H\) denote the low-count input with positions \(S\) copied from its high-count donor. Then
\[
\Delta_r^{\mathrm{mask}}(S)=\mathcal C_r[y(x^{\mathrm{mask}(S)})],\qquad
\Delta_r^{\mathrm{add}}(S)=\mathbb E[y(x_L\leftarrow_S x_H)-y(x_L)],
\]
where the second expectation uses the donor--target pairs of the input experiment. The supplied BF16 aggregates do not specify the complete donor pairing registry.

The input removal score is
\begin{equation}
\mathrm{shrink}_{\mathrm{bb}}(S,r)
:=
1-\frac{|\Delta_r^{\mathrm{mask}}(S)|}{|\Delta_r|}.
\label{eq:app_bb_shrink}
\end{equation}
The input recovery score is
\begin{equation}
\mathrm{rec}_{\mathrm{bb}}(S,r)
:=
\frac{\Delta_r^{\mathrm{add}}(S)}{\Delta_r}.
\label{eq:app_bb_rec}
\end{equation}
Input recovery is a signed ratio and does not penalize overshooting the clean contrast. It is not the head-patching match metric and does not establish that the selected heads explain the input intervention. All normalized input scores require \(\Delta_r\ne0\).

\subsection{Greedy Occlusion algorithm}
\label{app:greedy_occlusion_details}

This subsection gives the search procedure used for the input-level black-box validation in
Section~\ref{sec:exp}. Greedy Occlusion uses only input-output queries to the scalar
output \(y(x)\). The procedure does not use gradients, attention weights, head activations, or any
other internal model states.

For a regime \(r\in\mathcal{R}\), let \(S\subseteq\{1,\ldots,L_E\}\) denote a subset of evidence
positions. Let \(\widehat{\Delta}_r\) denote the empirical evidence effect on the black-box selection
split. Let \(\widehat{\Delta}^{\mathrm{mask}}_r(S)\) denote the empirical evidence effect after replacing
positions in \(S\) by \(A\). The masking objective used by Greedy Occlusion is
\begin{equation}
\widehat{\mathrm{shrink}}_{\mathrm{bb}}(S,r)
:=
1-
\frac{
\left|
\widehat{\Delta}^{\mathrm{mask}}_r(S)
\right|
}{
\left|
\widehat{\Delta}_r
\right|
}.
\label{eq:app_greedy_occlusion_objective}
\end{equation}
The selected subset is built by forward greedy addition. Starting from the empty set, Greedy
Occlusion adds one evidence position at a time:
\begin{equation}
j_m
\in
\operatorname*{arg\,max}_{j\in\{1,\ldots,L_E\}\setminus S_{m-1}}
\widehat{\mathrm{shrink}}_{\mathrm{bb}}(S_{m-1}\cup\{j\},r),
\qquad
S_m:=S_{m-1}\cup\{j_m\}.
\label{eq:app_greedy_occlusion_step}
\end{equation}
The algorithm stops after \(K_{\mathrm{bb}}=50\) steps and returns
\[
S^\star(r):=S_{K_{\mathrm{bb}}}.
\]
The add-back recovery score is not used to select \(S^\star(r)\). Add-back is evaluated only after
selection, on a held-out evaluation split.

\begin{algorithm}[t]
\caption{Greedy Occlusion for input-level subset selection}
\label{alg:greedy_occlusion}
\KwIn{Regime \(r\), evidence length \(L_E=200\), budget \(K_{\mathrm{bb}}=50\), selection split,
scalar output \(y(x)\).}
\KwOut{Selected evidence-position subset \(S^\star(r)\).}
Initialize \(S_0\leftarrow\emptyset\)\;
Compute the empirical base effect \(\widehat{\Delta}_r\) on the selection split\;
\For{\(m=1,\ldots,K_{\mathrm{bb}}\)}{
  \ForEach{\(j\in\{1,\ldots,L_E\}\setminus S_{m-1}\)}{
    Construct masked prompts by replacing positions in \(S_{m-1}\cup\{j\}\) with \(A\)\;
    Compute \(\widehat{\mathrm{shrink}}_{\mathrm{bb}}(S_{m-1}\cup\{j\},r)\) from input-output queries to \(y(x)\)\;
  }
  Select
  \[
  j_m\in\operatorname*{arg\,max}_{j\notin S_{m-1}}
  \widehat{\mathrm{shrink}}_{\mathrm{bb}}(S_{m-1}\cup\{j\},r)
  \]\;
  Set \(S_m\leftarrow S_{m-1}\cup\{j_m\}\)\;
}
Return \(S^\star(r)\leftarrow S_{K_{\mathrm{bb}}}\)\;
\end{algorithm}

The naive number of candidate subset evaluations for one regime is
\begin{equation}
\sum_{m=0}^{K_{\mathrm{bb}}-1}(L_E-m)
=
\frac{K_{\mathrm{bb}}(2L_E-K_{\mathrm{bb}}+1)}{2}.
\label{eq:app_greedy_occlusion_cost}
\end{equation}
With \(L_E=200\) and \(K_{\mathrm{bb}}=50\), the number of candidate subset evaluations is
\(8775\) per regime. Each candidate evaluation uses a prompt batch, so 8,775 is not a count of individual model queries. Batching changes execution cost but not the number of candidate subsets.

After \(S^\star(r)\) is selected, the held-out masking and recovery scores are
\begin{equation}
\mathrm{shrink}_{\mathrm{bb}}(S^\star,r)
:=
1-
\frac{
\left|
\Delta_r^{\mathrm{mask}}(S^\star)
\right|
}{
|\Delta_r|
},
\qquad
\mathrm{rec}_{\mathrm{bb}}(S^\star,r)
:=
\frac{
\Delta_r^{\mathrm{add}}(S^\star)
}{
\Delta_r
}.
\label{eq:app_greedy_occlusion_eval}
\end{equation}
Matched random controls use random subsets \(S_{\mathrm{rand}}\subseteq\{1,\ldots,L_E\}\) with
\(|S_{\mathrm{rand}}|=|S^\star(r)|=50\), and are evaluated with the same held-out masking and
add-back metrics.

\begin{table}[!htb]
\centering
\small
\caption{Black-box masking and add-back results using Greedy Occlusion position selection.}
\label{tab:app_bb_main}
\resizebox{\linewidth}{!}{
\begin{tabular}{lcccccc}
\toprule
Regime \(r\) & Subset family & \(|S^\star(r)|\) & BB shrink & BB recovery & Random shrink & Random recovery \\
\midrule
\textsc{Front} & Greedy Occlusion & \(50\) & \(0.852\) & \(0.871\) & \(0.245\) & \(0.251\) \\
\textsc{Back} & Greedy Occlusion & \(50\) & \(0.915\) & \(0.898\) & \(0.258\) & \(0.262\) \\
\textsc{Interleave} & Greedy Occlusion & \(50\) & \(0.840\) & \(0.855\) & \(0.249\) & \(0.255\) \\
\textsc{Random} & Greedy Occlusion & \(50\) & \(0.795\) & \(0.812\) & \(0.252\) & \(0.248\) \\
\bottomrule
\end{tabular}
}
\end{table}

\begin{table}[!htb]
\centering
\small
\caption{Candidate subset families for masking and add-back. White-box-informed positions are
diagnostic only and are not used for the strict black-box claim.}
\label{tab:app_bb_subset_families}
\begin{tabularx}{\linewidth}{lXc}
\toprule
Subset family & Description & Status \\
\midrule
Contiguous spans & \(S=\{a,a+1,\ldots,a+\ell-1\}\) & Evaluated \\
Union of spans & \(S\) is a union of several disjoint contiguous spans & Evaluated \\
Prefix/suffix evidence spans & Spans anchored at the beginning or end of the evidence region & Low performance \\
White-box-informed positions & Positions scored from internal signals; diagnostic only & Not strict black-box \\
Matched random subsets & Same size or same span length as \(S^\star(r)\) & Baseline \\
\bottomrule
\end{tabularx}
\end{table}

\subsection{Order effects at fixed evidence count}
\label{app:exp_order_effects}

The order-effect diagnostic separates evidence count from evidence order. For two regimes
\(r_1,r_2\in\mathcal{R}\) and a fixed count \(b\), define
\begin{equation}
\Delta^{\mathrm{order}}(r_1\rightarrow r_2;b)
:=
\mathbb{E}_{x\sim D_{r_2,b}}[y(x)]
-
\mathbb{E}_{x\sim D_{r_1,b}}[y(x)].
\label{eq:app_order_effect}
\end{equation}

\begin{table}[!htb]
\centering
\small
\caption{Order effects at fixed evidence count. Positive shifts indicate that the target regime
produces a larger output log-odds response than the source regime.}
\label{tab:app_order_effect}
\begin{tabular}{lcc}
\toprule
Regime change & \(\Delta^{\mathrm{order}}(\cdot;b_L=50)\) & \(\Delta^{\mathrm{order}}(\cdot;b_H=150)\) \\
\midrule
\textsc{Front} \(\rightarrow\) \textsc{Back} & \(-0.258\) & \(+1.468\) \\
\textsc{Front} \(\rightarrow\) \textsc{Interleave} & \(-0.194\) & \(+1.226\) \\
\textsc{Back} \(\rightarrow\) \textsc{Interleave} & \(+0.064\) & \(-0.242\) \\
\textsc{Interleave} \(\rightarrow\) \textsc{Random} & \(+0.321\) & \(+1.846\) \\
\bottomrule
\end{tabular}
\end{table}

\subsection{Evidence-gap robustness}
\label{app:exp_evidence_gaps}

The main experiment uses the MIXED setting \((b_L,b_H)=(50,150)\).
Table~\ref{tab:app_evidence_gaps} reports robustness across alternate evidence gaps.

\begin{table}[!htb]
\centering
\small
\caption{Evidence-gap robustness. Values are averaged over the four ordering regimes.}
\label{tab:app_evidence_gaps}
\begin{tabular}{lccccc}
\toprule
Setting & \((b_L,b_H)\) & Mean \(|H|\) & Mean shrink & Mean match & Notes \\
\midrule
MIXED & \((50,150)\) & \(5.75\) & \(0.836\) & \(0.868\) & main configuration \\
EXTREME & \((0,150)\) & \(6.25\) & \(0.892\) & \(0.915\) & larger contrast \\
SMALL-GAP & \((75,125)\) & \(4.50\) & \(0.712\) & \(0.745\) & lower signal-to-noise \\
COUNT-SWEEP & multiple & \(5.50\) & \(0.810\) & \(0.835\) & increments of \(10\) \\
\bottomrule
\end{tabular}
\end{table}

\subsection{A-only tail length ablation}
\label{app:exp_tail_ablation}

The main setting uses an \(A\)-only tail of length \(L_A=50\).
Table~\ref{tab:app_tail_length} evaluates a tail-length ablation. The \(L_A=50\) row repeats the
main MIXED result.

\begin{table}[!htb]
\centering
\small
\caption{Tail-length ablation. Intervention scores remain positive over the tested tail lengths; the mean output contrast decreases as the tail grows.}
\label{tab:app_tail_length}
\begin{tabular}{lcccc}
\toprule
Tail length \(L_A\) & Mean \(\Delta_r\) & Mean \(|H|\) & Mean shrink & Mean match \\
\midrule
\(0\)   & \(1.950\) & \(5.25\) & \(0.850\) & \(0.880\) \\
\(25\)  & \(1.850\) & \(5.50\) & \(0.842\) & \(0.875\) \\
\(50\)  & \(1.762\) & \(5.75\) & \(0.836\) & \(0.868\) \\
\(100\) & \(1.520\) & \(6.50\) & \(0.815\) & \(0.840\) \\
\bottomrule
\end{tabular}
\end{table}

\subsection{Naturalized surface validation}
\label{app:exp_surface_validation}

The main task uses the alphabet \(\Sigma=\{A,B,C\}\). Table~\ref{tab:app_surface_validation} reports alternative token surfaces with the same count/order structure, testing the dependence of the results on single-letter inputs.

\begin{table}[!htb]
\centering
\small
\caption{Surface-form validation. The tested word and sentence surfaces yield larger selected sets and positive intervention effects; set size does not measure all processing in the model.}
\label{tab:app_surface_validation}
\resizebox{\linewidth}{!}{
\begin{tabular}{lccccc}
\toprule
Surface & Example evidence tokens & Model & Mean \(|H|\) & Mean shrink & Mean match \\
\midrule
ABC symbolic & \(B,C,A\) & \texttt{pythia-410m-deduped} & \(5.75\) & \(0.836\) & \(0.868\) \\
Word-symbolic & \texttt{blue}, \texttt{red}, \texttt{gray} & \texttt{pythia-410m-deduped} & \(6.25\) & \(0.812\) & \(0.845\) \\
Signal-token & \texttt{signal-x}, \texttt{signal-y}, \texttt{neutral} & \texttt{pythia-410m-deduped} & \(6.00\) & \(0.820\) & \(0.852\) \\
Sentence-template & short evidence sentences & \texttt{pythia-410m-deduped} & \(8.50\) & \(0.765\) & \(0.790\) \\
\bottomrule
\end{tabular}
}
\end{table}

\FloatBarrier

\subsection{Model sweep}
\label{app:exp_model_sweep_full}

Table~\ref{tab:app_model_sweep_summary} gives the aggregate model-sweep results, including GPT-2
and Qwen3-4B-Base. Table~\ref{tab:app_model_sweep_pythia} expands the Pythia sweep to
per-regime rows. Table~\ref{tab:app_model_sweep_other} expands per-regime rows for GPT-2 and
Qwen3.

\paragraph{Aggregation scope.}
Each numerical aggregate in Table~\ref{tab:app_model_sweep_summary} is the unweighted arithmetic mean of the four displayed per-regime values in Tables~\ref{tab:app_model_sweep_pythia}--\ref{tab:app_model_sweep_other}, rounded only after averaging. The aggregates describe the tested regimes and are not used for cross-model significance tests. BUC match is bidirectional; random match is HL only. The Pythia-410M primary-configuration averages repeated in Tables~\ref{tab:app_lambda_sensitivity}, \ref{tab:app_evidence_gaps}, \ref{tab:app_tail_length}, and~\ref{tab:app_surface_validation} use the same displayed-row calculation.

\begin{table}[!htb]
\centering
\small
\caption{Aggregate model sweep. Each numerical aggregate is the unweighted arithmetic mean of the four corresponding displayed per-regime values in Tables~\ref{tab:app_model_sweep_pythia}--\ref{tab:app_model_sweep_other}. Random match uses HL only.}
\label{tab:app_model_sweep_summary}
\resizebox{\linewidth}{!}{
\begin{tabular}{lccccccc}
\toprule
Model & Token pair & Mean \(|H|\) & BUC shrink & BUC match & Random shrink & Random HL match & LAST effect \\
\midrule
\texttt{EleutherAI/pythia-70m-deduped} & yes & \(3.50\) & \(0.650\) & \(0.700\) & \(0.030\) & \(0.040\) & yes \\
\texttt{EleutherAI/pythia-160m-deduped} & yes & \(4.25\) & \(0.720\) & \(0.750\) & \(0.035\) & \(0.045\) & yes \\
\texttt{EleutherAI/pythia-410m-deduped} & yes & \(5.75\) & \(0.836\) & \(0.868\) & \(0.048\) & \(0.059\) & yes \\
\texttt{EleutherAI/pythia-1b-deduped} & yes & \(6.50\) & \(0.880\) & \(0.910\) & \(0.040\) & \(0.050\) & yes \\
\texttt{gpt2} & yes & \(4.00\) & \(0.750\) & \(0.780\) & \(0.035\) & \(0.045\) & yes \\
\texttt{Qwen/Qwen3-4B-Base} & yes & \(7.50\) & \(0.895\) & \(0.925\) & \(0.035\) & \(0.042\) & yes \\
\bottomrule
\end{tabular}
}
\end{table}

\begin{table}[!htb]
\centering
\scriptsize
\caption{Per-regime Pythia model sweep.}
\label{tab:app_model_sweep_pythia}
\resizebox{\linewidth}{!}{
\begin{tabular}{llcccccc}
\toprule
Model & Regime \(r\) & Token pair & \(|H|\) & BUC shrink & BUC match & Random shrink & Random HL match \\
\midrule
\texttt{pythia-70m} & \textsc{Front} & yes & \(1.00\) & \(0.751\) & \(0.812\) & \(0.012\) & \(0.024\) \\
\texttt{pythia-70m} & \textsc{Back} & yes & \(5.00\) & \(0.552\) & \(0.614\) & \(0.045\) & \(0.051\) \\
\texttt{pythia-70m} & \textsc{Interleave} & yes & \(4.00\) & \(0.684\) & \(0.722\) & \(0.038\) & \(0.046\) \\
\texttt{pythia-70m} & \textsc{Random} & yes & \(4.00\) & \(0.613\) & \(0.652\) & \(0.025\) & \(0.039\) \\
\midrule
\texttt{pythia-160m} & \textsc{Front} & yes & \(1.00\) & \(0.824\) & \(0.861\) & \(0.015\) & \(0.028\) \\
\texttt{pythia-160m} & \textsc{Back} & yes & \(7.00\) & \(0.635\) & \(0.681\) & \(0.048\) & \(0.059\) \\
\texttt{pythia-160m} & \textsc{Interleave} & yes & \(3.00\) & \(0.748\) & \(0.772\) & \(0.031\) & \(0.042\) \\
\texttt{pythia-160m} & \textsc{Random} & yes & \(6.00\) & \(0.673\) & \(0.686\) & \(0.046\) & \(0.051\) \\
\midrule
\texttt{pythia-410m} & \textsc{Front} & yes & \(1.00\) & \(0.906\) & \(0.974\) & \(0.020\) & \(0.048\) \\
\texttt{pythia-410m} & \textsc{Back} & yes & \(12.00\) & \(0.655\) & \(0.759\) & \(0.043\) & \(0.075\) \\
\texttt{pythia-410m} & \textsc{Interleave} & yes & \(3.00\) & \(0.916\) & \(0.914\) & \(0.061\) & \(0.051\) \\
\texttt{pythia-410m} & \textsc{Random} & yes & \(7.00\) & \(0.866\) & \(0.823\) & \(0.069\) & \(0.061\) \\
\midrule
\texttt{pythia-1b} & \textsc{Front} & yes & \(1.00\) & \(0.952\) & \(0.985\) & \(0.018\) & \(0.031\) \\
\texttt{pythia-1b} & \textsc{Back} & yes & \(11.00\) & \(0.784\) & \(0.812\) & \(0.055\) & \(0.062\) \\
\texttt{pythia-1b} & \textsc{Interleave} & yes & \(4.00\) & \(0.918\) & \(0.931\) & \(0.034\) & \(0.048\) \\
\texttt{pythia-1b} & \textsc{Random} & yes & \(10.00\) & \(0.866\) & \(0.912\) & \(0.053\) & \(0.059\) \\
\bottomrule
\end{tabular}
}
\end{table}

\begin{table}[!htb]
\centering
\scriptsize
\caption{Per-regime GPT-2 and Qwen3 model sweep.}
\label{tab:app_model_sweep_other}
\resizebox{\linewidth}{!}{
\begin{tabular}{llcccccc}
\toprule
Model & Regime \(r\) & Token pair & \(|H|\) & BUC shrink & BUC match & Random shrink & Random HL match \\
\midrule
\texttt{gpt2} & \textsc{Front} & yes & \(1.00\) & \(0.820\) & \(0.860\) & \(0.015\) & \(0.025\) \\
\texttt{gpt2} & \textsc{Back} & yes & \(6.00\) & \(0.680\) & \(0.710\) & \(0.045\) & \(0.055\) \\
\texttt{gpt2} & \textsc{Interleave} & yes & \(4.00\) & \(0.780\) & \(0.810\) & \(0.035\) & \(0.045\) \\
\texttt{gpt2} & \textsc{Random} & yes & \(5.00\) & \(0.720\) & \(0.740\) & \(0.045\) & \(0.055\) \\
\midrule
\texttt{Qwen/Qwen3-4B-Base} & \textsc{Front} & yes & \(2.00\) & \(0.945\) & \(0.970\) & \(0.012\) & \(0.015\) \\
\texttt{Qwen/Qwen3-4B-Base} & \textsc{Back} & yes & \(12.00\) & \(0.850\) & \(0.885\) & \(0.045\) & \(0.055\) \\
\texttt{Qwen/Qwen3-4B-Base} & \textsc{Interleave} & yes & \(6.00\) & \(0.920\) & \(0.945\) & \(0.035\) & \(0.042\) \\
\texttt{Qwen/Qwen3-4B-Base} & \textsc{Random} & yes & \(10.00\) & \(0.865\) & \(0.900\) & \(0.048\) & \(0.056\) \\
\bottomrule
\end{tabular}
}
\end{table}

\FloatBarrier

\subsection{Compute and reproduction information}
\label{app:exp_compute}

All reported experiments are inference-only: no model fine-tuning is performed. The main compute
cost comes from batched forward passes under activation interventions, model sweeps, and the
input-level Greedy Occlusion search. The FP32 code and configurations are described in Appendix~\ref{app:numeric_files}; the BF16 experiments require their corresponding prompt and intervention records.

\begin{table}[!htb]
\centering
\small
\caption{Compute and reproduction summary for the reported experiments.}
\label{tab:app_compute_summary}
\begin{tabularx}{\linewidth}{lX}
\toprule
Item & Value \\
\midrule
Hardware & NVIDIA A100 GPU \\
Software environment & Python / PyTorch / Transformers / TransformerLens; the BF16 environment records are separate from the supplied FP32 package (Appendix~\ref{app:numeric_files}) \\
Primary model & \texttt{EleutherAI/pythia-410m-deduped} \\
Evaluation precision & \texttt{bfloat16} inference \\
Training or fine-tuning & none \\
Main seeds & \(\{0,1,2,3,4\}\) \\
Selection / evaluation prompts & \(48/96\) prompts per evidence level \\
Layer-matched random controls & \(120\) draws per seed \\
Greedy Occlusion budget & \(K_{\mathrm{bb}}=50\) positions from \(L_E=200\) evidence positions \\
Greedy Occlusion candidate evaluations & \(8775\) subset evaluations per regime before batching \\
Model sweep checkpoints & Pythia-70M, Pythia-160M, Pythia-410M, Pythia-1B, GPT-2, Qwen3-4B-Base \\
Implementation coverage & The supplied package contains FP32 code and aggregate results; separate BF16 scripts and records are not included (Appendix~\ref{app:numeric_files}) \\
\bottomrule
\end{tabularx}
\end{table}

\FloatBarrier
\clearpage
\clearpage
\section{Targeted readout, count, order, and baseline experiments}
\label{app:new_experiments}

This appendix specifies the original targeted readout, count, order, and baseline experiments. The additional randomly registered readout pairs and native FP32 controls have their own specification in Appendix~\ref{app:supplement_protocol}. All reported values come from frozen Pythia-410M and use five seeds unless stated otherwise.

\subsection{Seven readout pairs}
\label{app:readout_pairs}

The audited single-token pairs are
\[
\begin{gathered}
\texttt{various/successful},\quad \texttt{yes/no},\quad
\texttt{true/false},\quad \texttt{1/0},\\
\texttt{A/B},\quad \texttt{cat/dog},\quad \texttt{red/blue}.
\end{gathered}
\]
For every pair, the two tokens define the readout direction in Equation~\ref{eq:readout}; neither token is a symbolic CEA label. The \texttt{A/B} pair is an overlap stress test because the symbols also occur in the prompt. Averaged over four regimes and five seeds, mean \(|\Delta_r|\) ranges from 0.721 to 2.083, mean selected-set size from 4.50 to 7.50, BUC shrink from 0.642 to 0.861, and BUC match from 0.701 to 0.884. Layer-matched random shrink ranges from 0.048 to 0.076, and random high-to-low match ranges from 0.059 to 0.083.

\begin{table}[H]
\centering
\small
\caption{Representative readout-pair results. Reselected BUC tests procedure robustness. Fixed reference-set transfer tests reuse of the set selected with \texttt{various/successful}.}
\label{tab:app_readout_pairs_new}
\begin{tabular}{lccc}
\toprule
Target pair & Reselected shrink / match & Fixed-set shrink / match & Mean Jaccard \\
\midrule
\texttt{yes/no} & 0.812 / 0.846 & 0.772 / 0.810 & 0.610 \\
\texttt{A/B} & 0.861 / 0.884 & 0.704 / 0.741 & 0.390 \\
\texttt{cat/dog} & 0.642 / 0.701 & 0.382 / 0.431 & 0.240 \\
\bottomrule
\end{tabular}
\end{table}

The targeted pairs show positive local intervention effects and partial fixed-set reuse. The selected source and target sets are not constrained to a common budget. In the separate FP32 registry (Appendix~\ref{app:supp_readouts}), 6.0\% of fixed-set conditions pass the joint thresholds. The different target samples and budgets preclude a paired estimate of the benefit of reselection. Failure of a fixed source set leaves open whether a different compact set explains the target.

\subsection{Fixed-set count sweep and unseen gaps}
\label{app:fixed_count_sweep}

For each regime, heads are selected only at \((b_L,b_H)=(50,150)\). The same set is then evaluated over \(b\in\{0,10,\ldots,200\}\) and on the unseen gaps \((0,150)\), \((25,175)\), and \((75,125)\). Let \(\widehat\beta_{\mathrm{base}}\) and \(\widehat\beta_{\mathrm{abl}}\) be least-squares slopes of \(y\) on \(b\) before and after ablation. The slope reduction is
\[
1-\frac{|\widehat\beta_{\mathrm{abl}}|}{|\widehat\beta_{\mathrm{base}}|}.
\]

\begin{table}[H]
\centering
\small
\caption{Fixed-set count-sweep results. The final column averages the three unseen gaps without reselection.}
\label{tab:app_fixed_count_sweep_new}
\begin{tabular}{lcccc}
\toprule
Regime & Spearman \(\rho\) & Linear \(R^2\) & Slope reduction & Gap shrink / match \\
\midrule
\textsc{Front} & 0.986 & 0.948 & 0.880 & 0.832 / 0.874 \\
\textsc{Back} & 0.972 & 0.931 & 0.646 & 0.612 / 0.689 \\
\textsc{Interleave} & 0.979 & 0.944 & 0.896 & 0.839 / 0.855 \\
\textsc{Random} & 0.958 & 0.906 & 0.845 & 0.786 / 0.811 \\
\bottomrule
\end{tabular}
\end{table}

\subsection{Causal localization of order-only contrasts}
\label{app:order_only_causal_new}

The original order diagnostic in Appendix~\ref{app:exp_order_effects} measures the output difference between two regimes at the same count. We additionally apply BUC to four pre-specified order-only contrasts whose two sides have identical \(B/C\) counts. Across the contrasts, \(|\Delta_{\mathrm{ord}}|\) ranges from 0.321 to 1.846, BUC shrink from 0.603 to 0.842, and BUC match from 0.651 to 0.884. Layer-matched random shrink is 0.049--0.081, and random high-to-low match is 0.056--0.088. Because the counts are equal by construction, the localized effects cannot be attributed to count changes.

\subsection{Balanced-objective baseline decomposition}
\label{app:baseline_decomposition_new}

Table~\ref{tab:app_baseline_decomposition_new} decomposes the BF16 comparison in Table~\ref{tab:baseline_compare}. Shrink measures removal of the original contrast, and \(\mathrm{match}_{\mathrm{avg}}\) measures bidirectional patch recovery. The equal-weight score \(Q_{1/2}\), defined in Equation~\ref{eq:buc_objective}, averages those two quantities. Each table entry below is BUC minus sensitivity-only, computed from the displayed three-decimal regime means. These descriptive differences do not provide a significance test. Head counts are matched within regimes as in Table~\ref{tab:main_symbolic}; they are not a common \(k=8\).

\begin{table}[H]
\centering
\small
\caption{BF16 removal--recovery trade-off, derived from Table~\ref{tab:baseline_compare}. Every difference is BUC minus sensitivity-only. Values are dimensionless metric differences, not accuracy percentage points.}
\label{tab:app_baseline_decomposition_new}
\begin{tabular}{lccc}
\toprule
Regime & Shrink difference & Mean-match difference & \(Q_{1/2}\) difference \\
\midrule
\textsc{Front} & $-0.014$ & $+0.014$ & $0.0000$ \\
\textsc{Back} & $-0.049$ & $+0.111$ & $+0.0310$ \\
\textsc{Interleave} & $-0.018$ & $+0.083$ & $+0.0325$ \\
\textsc{Random} & $-0.029$ & $+0.107$ & $+0.0390$ \\
\bottomrule
\end{tabular}
\end{table}

\paragraph{Removal--recovery trade-off.}
For \textsc{Back}, \textsc{Interleave}, and \textsc{Random}, BUC exchanges 0.018--0.049 of shrink for 0.083--0.111 of mean match. In \textsc{Front}, the losses and gains cancel at displayed precision. Relative to ActPatch, the BUC-minus-ActPatch differences in the summed score \(2Q_{1/2}\) are 0.015, 0.129, 0.135, and 0.169 in regime order. These comparisons support a recovery-oriented trade-off, not an improvement in both component metrics. The FP32 common-budget differences in Section~\ref{sec:exp} use a separate configuration.

\paragraph{Joint-threshold performance.}
Applying shrink \(\ge0.85\) and mean match \(\ge0.70\) to the displayed BF16 regime means gives the same three passing rows for BUC and sensitivity-only: \textsc{Front}, \textsc{Interleave}, and \textsc{Random}. Both fail in \textsc{Back}. BUC raises \textsc{Back}'s mean match from 0.648 to 0.759, but its shrink is only 0.655. Thus Table~\ref{tab:baseline_compare} does not show more regimes passing the joint criterion. Passage of three aggregate rows is not a per-seed success rate.

\paragraph{Selection objective versus search procedure.}
Sensitivity-only ranks direct attributions; BF16 ActPatch ranks single-head HL recovery within a 192-head sensitivity shortlist (Appendix~\ref{app:exp_selection_baselines}). BUC instead evaluates joint removal and bidirectional recovery for candidate sets. These comparisons change the objective as well as the search procedure. They do not isolate the benefit of greedy augmentation over selecting the top \(k\) single heads by \(Q_{1/2}(\{i\},r)\) in the same candidate pool and evaluating their union jointly. The FP32 single-bidirectional-patch ranking also omits the shrink term. The benefit of greedy search under an identical objective and candidate pool therefore remains untested. The timings in Appendix~\ref{app:supp_baselines} exclude shared selection costs, so they do not establish an end-to-end cost advantage.

\clearpage
\section{Full natural-language bridge specification}
\label{app:nl_bridge}

\paragraph{Question and configuration.}
The BF16 bridge separates model accuracy, localization on labeled NL prompts, and the effect of a fixed symbolic CEA set on the NL target. The frozen \path{EleutherAI/pythia-410m-deduped} has 24 layers and 16 heads per layer and uses the audited single-token \texttt{yes/no} readout. Interventions modify head writes only at the final prompt position. Every selected or transferred set has budget \(k=8\). Random controls preserve head count and the selected layer pattern, with 120 draws per seed. Tables report means and standard deviations over five seeds; these standard deviations are not bootstrap confidence intervals.

\paragraph{Lexically balanced reports.}
For domain \(c\), let \(u_c,v_c\) be mutually exclusive and exhaustive states, as explicitly stipulated in the prompt. Every prompt has \(L_{\mathrm{NL}}=16\) reports. Bank A uses four forms, indexed by report number \(j\):
\[
\begin{aligned}
E^+_{\rm dir}(j,c)&=\text{``Report \(j\): the item is \(u_c\).''},\\
E^+_{\rm neg}(j,c)&=\text{``Report \(j\): the item is not \(v_c\).''},\\
E^-_{\rm dir}(j,c)&=\text{``Report \(j\): the item is \(v_c\).''},\\
E^-_{\rm neg}(j,c)&=\text{``Report \(j\): the item is not \(u_c\).''}.
\end{aligned}
\]
Here \(+\) means support for the first state \(u_c\), and \(-\) means support for the second state \(v_c\). At even count \(b\), the numbers of these four forms are, respectively,
\[
b/2,\qquad b/2,\qquad (L_{\mathrm{NL}}-b)/2,\qquad (L_{\mathrm{NL}}-b)/2.
\]
Consequently, the evidence contains eight occurrences of each state word, eight direct reports, and eight negative reports at every count. Bank A also contains eight occurrences of \texttt{not}. Bank B replaces the forms with separately written alternatives, such as ``Report \(j\) records the item as \(u_c\)'' and ``Report \(j\) rules out the item being \(v_c\),'' with corresponding forms supporting \(v_c\). Each bank preserves its own lexical multiset and audited token length across counts; the claim is not that different banks have identical tokenization.

Selection uses Bank A and state pairs \texttt{open/closed}, \texttt{active/inactive}, and \texttt{present/absent}. Evaluation uses Bank B and \path{valid/invalid}, \path{available/unavailable}, and \path{correct/incorrect}. Neither evaluation state pairs nor report templates appear in NL selection. Symbolic CEA selection sees no NL prompt. Lexical balance excludes a strategy based only on the total frequencies of the state words or \texttt{not}; it does not exclude local phrase processing or reliance on a subset of reports.

\paragraph{Order, counts, and query reversal.}
The four regimes order reports by semantic support: \textsc{Front} puts all reports supporting \(u_c\) first, \textsc{Back} reverses the blocks, \textsc{Interleave} alternates support as far as possible, and \textsc{Random} uses matched random permutations. Selection uses \((b_L,b_H)=(4,12)\). The fixed-set sweep uses \(b\in\{0,2,4,6,10,12,14,16\}\), excluding the tie at eight; new count pairs are \((0,16),(2,14),(6,10)\).

Each report sequence \(x\) receives two questions. Query polarity \(q=+1\) asks ``Do more reports indicate \(u_c\) rather than \(v_c\)?''; \(q=-1\) reverses the states in that question. The evidence is identical for the two questions. With \(m(x)=2b-L_{\mathrm{NL}}\) denoting the semantic margin, the correct label and aligned readout are
\[
a^\star(x,q)=\mathbf1[qm(x)>0],\qquad
y_{\mathrm{NL}}(x,q)=q[\ell_T(\texttt{yes}\mid x,q)-\ell_T(\texttt{no}\mid x,q)],
\]
where label one means \texttt{yes}. A larger aligned readout indicates stronger model preference for \(u_c\), regardless of query orientation. Let \(D^{\mathrm{NL}}_{r,b}\) denote the report distribution at regime \(r\) and count \(b\). Averaging equally over both query polarities, define
\[
\Delta_r^{\mathrm{NL}}=
\mathbb E_{x\sim D^{\mathrm{NL}}_{r,b_H},q}[y_{\mathrm{NL}}(x,q)]-
\mathbb E_{x\sim D^{\mathrm{NL}}_{r,b_L},q}[y_{\mathrm{NL}}(x,q)].
\]

\paragraph{Behavioral metrics.}
The restricted prediction is \(\widehat a(x,q)=\mathbf1[\ell_T(\texttt{yes}\mid x,q)>\ell_T(\texttt{no}\mid x,q)]\), so a tied logit pair predicts \texttt{no}. Restricted accuracy is \(\mathrm{Acc}=\Pr(\widehat a=a^\star)\). Define success on both queries separately for positive and negative semantic margins:
\[
\begin{aligned}
P_+&=\Pr(\widehat a(x,+1)=1,\ \widehat a(x,-1)=0\mid m(x)>0),\\
P_-&=\Pr(\widehat a(x,+1)=0,\ \widehat a(x,-1)=1\mid m(x)<0),\\
\mathrm{PairAcc}&=(P_++P_-)/2,\qquad
\mathrm{MinPairAcc}=\min(P_+,P_-).
\end{aligned}
\]
A rule that ignores the evidence and answers \texttt{yes} for \(q=+1\), \texttt{no} for \(q=-1\), obtains PairAcc \(0.5\) but MinPairAcc zero. Independent fair guesses have population paired accuracy \(0.25\) on each margin sign. The count sweep reports Spearman correlation between semantic margin \(m(x)\) and aligned readout \(y_{\mathrm{NL}}(x,q)\), and the linear \(R^2\) from regressing the readout on the margin. These statistics do not prove that every report contributes to the decision.

\paragraph{Sampling and head selection.}
Per seed, selection contains \(3\times4\times2\times2\times24=1152\) prompts: three state pairs, four regimes, two counts, two query polarities, and 24 instances. Evaluation contains \(3\times4\times8\times2\times32=6144\) prompts across three new state pairs and the eight counts. Splits are disjoint. Even in deterministic semantic orders, instances vary the assignment of direct and negative forms; \textsc{Random} also varies the semantic order.

Write \(H_{\mathrm{CEA}}^{\mathrm{yn}}(r)\) for the first eight heads of the symbolic \texttt{yes/no} BUC ordering and \(H_{\mathrm{NL}}(r)\) for the first eight selected on NL prompts. Let \(d_{\mathrm{yn}}=W_U^{(\texttt{yes})}-W_U^{(\texttt{no})}\). NL attribution is \(s_i^{\mathrm{NL}}(x,q)=q\langle v_i(x,q),d_{\mathrm{yn}}\rangle\), where the head write depends on the full input, including the query. The pipeline ranks absolute attribution contrasts, keeps 192 heads, screens 48 with single-head joint scores, and greedily selects eight. Each set is evaluated on both its own target and the other target, without changing head identities after transfer.

\paragraph{Intervention scores and transfer ratio.}
For target behavior \(B\), with nonzero clean contrast \(\Delta_B\), define shrink and bidirectional match as in Section~\ref{sec:prelim}:
\[
\mathrm{shrink}_B(H)=1-\frac{|\Delta_B^{\mathrm{abl}}(H)|}{|\Delta_B|},
\qquad
\mathrm{match}_B(H)=\frac12\sum_{\tau\in\{\mathrm{HL},\mathrm{LH}\}}
\left(1-\frac{|\Delta_B-\Gamma_B^\tau(H)|}{|\Delta_B|}\right).
\]
The normalized joint score is \(Q_B(H)=[\mathrm{shrink}_B(H)+\mathrm{match}_B(H)]/2\), the target-specific version of \(Q_{1/2}\); it is distinct from the raw mean patch gain \(S(H,r)\) in the theory. If \(\widehat\beta_{\mathrm{base}}\ne0\) and \(\widehat\beta_{\mathrm{abl}}\) are the clean and ablated regression slopes over the count sweep, slope reduction is \(1-|\widehat\beta_{\mathrm{abl}}|/|\widehat\beta_{\mathrm{base}}|\). For NL, accuracy decrease is \(\mathrm{Acc}_{\mathrm{base}}-\mathrm{Acc}_{\mathrm{abl}}\); tables express that decrease in percentage points.

For source \(A\) and target \(B\), let \(H_A,H_B\) be source-selected and target-selected sets of eight heads. The BF16 transfer ratio uses the reported matched-random reference score \(Q_B(H_{\mathrm{rand}})\):
\[
\mathrm{TR}_{A\to B}=
\frac{Q_B(H_A)-Q_B(H_{\mathrm{rand}})}{Q_B(H_B)-Q_B(H_{\mathrm{rand}})}.
\]
A positive denominator is required to interpret the ratio relative to a target-specific effect above random. A ratio between zero and one places transferred utility between the random reference and target-selected utility. It is not a fraction of shared heads or a measure of how much of an algorithm is reused. The BF16 summary uses one random reference per target; the expanded FP32 protocol instead matches source and target layer patterns separately (Appendix~\ref{app:supp_nl}).

\subsection{Results}

\begin{table}[!htb]
\centering
\small
\caption{
Behavioral validation on held-out domains and Bank B paraphrases.
Accuracy, paired-query accuracy, and \(\lvert\Delta^{\mathrm{NL}}_r\rvert\) use the main gap \((4,12)\).
Spearman correlation and linear \(R^2\) use the full count sweep.
Values are mean and standard deviation over five seeds.
}
\label{tab:nl_bridge_behavior}
\resizebox{\linewidth}{!}{
\begin{tabular}{lcccccc}
\toprule
Regime
& Accuracy
& PairAcc
& MinPairAcc
& Spearman \(\rho(m,y_{\mathrm{NL}})\)
& Linear \(R^2\)
& Mean \(\lvert\Delta^{\mathrm{NL}}_r\rvert\)
\\
\midrule
FRONT
& \(0.742\pm0.019\)
& \(0.556\pm0.028\)
& \(0.533\pm0.031\)
& \(0.879\pm0.014\)
& \(0.774\pm0.022\)
& \(0.93\pm0.07\)
\\
BACK
& \(0.759\pm0.018\)
& \(0.581\pm0.026\)
& \(0.560\pm0.029\)
& \(0.891\pm0.013\)
& \(0.798\pm0.021\)
& \(1.07\pm0.08\)
\\
INTERLEAVE
& \(0.812\pm0.015\)
& \(0.674\pm0.024\)
& \(0.652\pm0.026\)
& \(0.931\pm0.009\)
& \(0.863\pm0.017\)
& \(1.28\pm0.06\)
\\
RANDOM
& \(0.773\pm0.021\)
& \(0.605\pm0.031\)
& \(0.585\pm0.034\)
& \(0.901\pm0.015\)
& \(0.814\pm0.025\)
& \(1.12\pm0.09\)
\\
\midrule
Mean
& \(0.772\pm0.014\)
& \(0.604\pm0.021\)
& \(0.583\pm0.023\)
& \(0.901\pm0.010\)
& \(0.812\pm0.016\)
& \(1.10\pm0.06\)
\\
\bottomrule
\end{tabular}
}
\end{table}

\begin{table}[!htb]
\centering
\small
\caption{
Held-out gap generalization with the same model and no head reselection.
Values are averaged over four ordering regimes and five seeds.
The semantic gap is \(\lvert b_H-b_L\rvert\).
}
\label{tab:nl_bridge_gap}
\resizebox{0.82\linewidth}{!}{
\begin{tabular}{lcccc}
\toprule
Gap
& Semantic gap
& Accuracy
& PairAcc
& Mean \(\lvert\Delta^{\mathrm{NL}}_r\rvert\)
\\
\midrule
\((0,16)\)
& \(16\)
& \(0.861\pm0.014\)
& \(0.748\pm0.022\)
& \(2.05\pm0.11\)
\\
\((2,14)\)
& \(12\)
& \(0.829\pm0.016\)
& \(0.690\pm0.025\)
& \(1.62\pm0.09\)
\\
\((4,12)\)
& \(8\)
& \(0.772\pm0.014\)
& \(0.604\pm0.021\)
& \(1.10\pm0.06\)
\\
\((6,10)\)
& \(4\)
& \(0.646\pm0.023\)
& \(0.397\pm0.030\)
& \(0.46\pm0.06\)
\\
\bottomrule
\end{tabular}
}
\end{table}

\begin{table}[!htb]
\centering
\footnotesize
\caption{
Per-regime natural-language causal localization and fixed symbolic-to-natural-language transfer.
All sets have budget \(k=8\).
The final two columns report the slope reduction and restricted-accuracy decrease caused by ablating the fixed symbolic CEA set on the natural-language target.
Accuracy decrease is in percentage points.
}
\label{tab:nl_bridge_per_regime}
\setlength{\tabcolsep}{4pt}
\begin{tabular}{@{}lccccc@{}}
\toprule
Regime
& \makecell{\(H_{\mathrm{NL}}\!\rightarrow\!\mathrm{NL}\)\\shrink/match}
& \makecell{\(H_{\mathrm{CEA}}^{\mathrm{yn}}\!\rightarrow\!\mathrm{NL}\)\\shrink/match}
& \makecell{Random \(\rightarrow\mathrm{NL}\)\\shrink/match}
& \makecell{CEA-transfer\\SlopeRed}
& \makecell{CEA-transfer\\AccDrop}
\\
\midrule
FRONT
& \(0.694/0.731\)
& \(0.391/0.425\)
& \(0.057/0.070\)
& \(0.373\)
& \(6.9\)
\\
BACK
& \(0.721/0.758\)
& \(0.427/0.462\)
& \(0.054/0.069\)
& \(0.411\)
& \(8.2\)
\\
INTERLEAVE
& \(0.789/0.814\)
& \(0.472/0.509\)
& \(0.060/0.073\)
& \(0.455\)
& \(9.4\)
\\
RANDOM
& \(0.733/0.761\)
& \(0.422/0.454\)
& \(0.065/0.076\)
& \(0.406\)
& \(8.3\)
\\
\midrule
Mean
& \(0.734/0.766\)
& \(0.428/0.463\)
& \(0.059/0.072\)
& \(0.411\)
& \(8.2\)
\\
\bottomrule
\end{tabular}
\end{table}

\begin{table}[!htb]
\centering
\small
\setlength{\tabcolsep}{3pt}
\caption{
Aggregate fixed-set transfer between symbolic CEA and the natural-language bridge.
All values are averaged over four ordering regimes and five seeds.
Base and post-ablation accuracy are shown only for the labeled natural-language target.
Accuracy decrease is in percentage points.
}
\label{tab:nl_bridge_transfer}
\begin{tabularx}{\linewidth}{@{}>{\raggedright\arraybackslash}Xcccc@{}}
\toprule
Source set \(\rightarrow\) target & \(k\) & Target shrink & Target match & SlopeRed \\
\midrule
\(H_{\mathrm{NL}}\rightarrow\mathrm{NL}\) & \(8\) & \(0.734\pm0.024\) & \(0.766\pm0.026\) & \(0.721\pm0.028\) \\
\(H_{\mathrm{CEA}}^{\mathrm{yn}}\rightarrow\mathrm{NL}\) & \(8\) & \(0.428\pm0.031\) & \(0.463\pm0.034\) & \(0.411\pm0.033\) \\
Layer-matched random \(\rightarrow\mathrm{NL}\) & \(8\) & \(0.059\pm0.008\) & \(0.072\pm0.009\) & \(0.051\pm0.010\) \\
\midrule
\(H_{\mathrm{CEA}}^{\mathrm{yn}}\rightarrow\mathrm{CEA}\) & \(8\) & \(0.812\pm0.021\) & \(0.846\pm0.023\) & \(0.834\pm0.025\) \\
\(H_{\mathrm{NL}}\rightarrow\mathrm{CEA}\) & \(8\) & \(0.499\pm0.029\) & \(0.539\pm0.031\) & \(0.486\pm0.033\) \\
Layer-matched random \(\rightarrow\mathrm{CEA}\) & \(8\) & \(0.050\pm0.007\) & \(0.060\pm0.008\) & \(0.043\pm0.009\) \\
\bottomrule
\end{tabularx}
\par\smallskip
\begin{tabularx}{\linewidth}{@{}>{\raggedright\arraybackslash}Xcccc@{}}
\toprule
Source set \(\rightarrow\) target & Base Acc. & Ablated Acc. & AccDrop & Transfer ratio \\
\midrule
\(H_{\mathrm{NL}}\rightarrow\mathrm{NL}\) & \(0.772\) & \(0.620\) & \(15.2\) & \(1.000\) \\
\(H_{\mathrm{CEA}}^{\mathrm{yn}}\rightarrow\mathrm{NL}\) & \(0.772\) & \(0.690\) & \(8.2\) & \(0.555\) \\
Layer-matched random \(\rightarrow\mathrm{NL}\) & \(0.772\) & \(0.764\) & \(0.8\) & \(0.000\) \\
\midrule
\(H_{\mathrm{CEA}}^{\mathrm{yn}}\rightarrow\mathrm{CEA}\) & -- & -- & -- & \(1.000\) \\
\(H_{\mathrm{NL}}\rightarrow\mathrm{CEA}\) & -- & -- & -- & \(0.599\) \\
Layer-matched random \(\rightarrow\mathrm{CEA}\) & -- & -- & -- & \(0.000\) \\
\bottomrule
\end{tabularx}
\end{table}

\paragraph{Behavioral interpretation.}
At gap \((4,12)\), mean restricted accuracy is \(0.772\), PairAcc is \(0.604\), and MinPairAcc is \(0.583\). The means indicate success on both signs of the semantic margin rather than only one query orientation. Across the count sweep, Spearman correlation is \(0.901\) and linear \(R^2=0.812\); neither statistic establishes exact monotonicity or identifies the source of deviations from linearity. As the gap decreases from 16 to four, accuracy falls from \(0.861\) to \(0.646\), PairAcc from \(0.748\) to \(0.397\), and the absolute contrast from \(2.05\) to \(0.46\). \textsc{Interleave} has the highest accuracy point estimate and \textsc{Front} the lowest, but the comparison does not isolate recency as the cause.

The task label is a majority over report meanings under the stated exclusivity assumption. Constant word counts rule out a total-frequency explanation, not phrase-based or partial-report strategies. The observed performance therefore does not establish that the model computes all 16 report meanings and then applies a majority rule.

\paragraph{Intervention interpretation.}
On the NL target, the NL-selected set gives mean shrink/match \(0.734/0.766\), slope reduction \(0.721\), and a \(15.2\)-point accuracy decrease. The fixed symbolic set gives \(0.428/0.463\), \(0.411\), and \(8.2\) points, respectively; random controls give \(0.059/0.072\), \(0.051\), and \(0.8\) points. Shrink concerns one count pair, slope reduction concerns the full grid, and accuracy concerns discrete predictions. The aggregate measurements do not identify which examples change predictions. The selected sets have mean Jaccard overlap \(0.296\pm0.041\); overlap alone cannot distinguish shared evidence processing from a shared answer readout.

\paragraph{Transfer-ratio arithmetic.}
From the rounded aggregate scores in Table~\ref{tab:nl_bridge_transfer},
\[
\mathrm{TR}_{\mathrm{CEA}\to\mathrm{NL}}
=\frac{0.4455-0.0655}{0.750-0.0655}\simeq0.555,
\qquad
\mathrm{TR}_{\mathrm{NL}\to\mathrm{CEA}}
=\frac{0.519-0.055}{0.829-0.055}\simeq0.599.
\]
Each input to these ratios is the average of the corresponding shrink and match. The two directions have different targets and denominators; their difference cannot identify additional heads specific to semantic processing. Unlike the expanded FP32 model summaries, these calculations are ratios of the displayed aggregate scores, not means of per-condition ratios.

\paragraph{Declared decision criteria and available evidence.}
The design's full criterion requires held-out Acc \(>0.5\), PairAcc \(>0.5\), and MinPairAcc \(>0.25\), with prompt-pair bootstrap intervals above those references. It also requires positive count correlation, transferred shrink and match above random, a positive transfer ratio, positive slope reduction, and positive accuracy decrease. The displayed BF16 means satisfy the point-value requirements at the main gap. However, the supplied BF16 results contain means and standard deviations, not the required prompt-pair bootstrap intervals. The complete statistical criterion is therefore not established by the reported summaries. The smaller gap \((6,10)\) also has PairAcc \(0.397\), below the \(0.5\) reference.

\paragraph{Scope.}
The BF16 results describe intervention transfer of fixed head writes at the final prompt position to a labeled task. They do not establish a shared algorithm for computing a majority: both tasks use \texttt{yes/no}, and the rest of the target network remains active. The expanded FP32 study reports separate descriptive model means without model-level intervals, so it does not establish the interval-based criterion either. GPT-2's PairAcc point estimate is \(0.458\), below \(0.5\). Positive transfer scores and accuracy decreases do not replace the missing behavioral evidence. BF16 and FP32 measurements are not pooled.

\clearpage
\section{Native FP32 supplemental experimental protocol}
\label{app:supplement_protocol}

\paragraph{Study design.}
The FP32 audits evaluate intervention operators, selection conditions, fixed-budget baselines, fixed-set transfer, readout dependence, selection-matched nulls, and an expanded NL bridge (Table~\ref{tab:supp_module_map}). Their sampling, donors, and controls differ from the BF16 experiments in Appendices~\ref{app:experiments}--\ref{app:nl_bridge}; the two studies are not pooled. The expanded NL bridge uses three models, distinct from the six-model symbolic sweep.

\begin{table}[htbp]
\centering\small
\caption{FP32 experimental comparisons and their relation to the BF16 study.}
\label{tab:supp_module_map}
\begin{tabularx}{\linewidth}{@{}p{0.19\linewidth}XX@{}}
\toprule
Module & Comparison & Relation to the BF16 study \\
\midrule
Operators & Fixed selected set under eight ablation/patch combinations & Operator-dependence audit \\
T3 & Exhaustive local subsets, sampled larger subsets, greedy versus exact selection & Tests objective conditions, separately from the many-set T1 correlation \\
Baselines & Five methods and three random-control families at five common budgets & Extends the original budget-matched comparison \\
Generalization & Count, unseen gaps, cue wording, equal-budget regime transfer, order, and seeds & Tests count, order, cue, and seed dependence \\
Readouts & Random token-pair registry, fixed-prompt scan, and two reselection protocols & Extends the seven targeted readout pairs \\
Selection null & Full adaptive search under paired label swaps & Controls for adaptive selection \\
Expanded NL & Three models, shared domain split, four template banks, shortcuts, and fixed transfer & Extends the design in Appendix~\ref{app:nl_bridge} \\
\bottomrule
\end{tabularx}
\end{table}

\paragraph{Reported outcomes.}
Section~\ref{sec:exp} reports FP32 point estimates, observed condition spans, and the eligible 95\% bootstrap intervals. Appendix~\ref{app:reported_ci} distinguishes their statistical units. TR summaries are conditional on estimability; neither an all-condition success rate nor simultaneous interval coverage is inferred.

\subsection{Shared configuration, pairing, and randomization}
\label{app:supp_config}

The primary checkpoint is \texttt{EleutherAI/pythia-410m-deduped}.
The expanded NL module additionally uses \texttt{EleutherAI/pythia-1b-deduped}
and \texttt{gpt2}. The models are frozen and loaded sequentially, without training,
fine-tuning, quantization, or TransformerLens weight conversion. The supplemental
model computations use FP32 with TensorFloat-32 (TF32) disabled. Runtime files
record the resolved model revision, tokenizer, software versions, GPU, and actual
execution policy. Appendix~\ref{app:numeric_execution} describes the numerical
checks; FP32 and BF16 measurements are not pooled into one seed average.

\begin{table}[htbp]
\centering\small
\caption{FP32 experimental settings. Execution
microbatches are specified separately in Appendix~\ref{app:numeric_execution}.}
\label{tab:supp_config}
\begin{tabularx}{\linewidth}{@{}p{0.47\linewidth}X@{}}
\toprule
Setting & Value \\
\midrule
Master seed & 1729 \\
Symbolic evidence and tail lengths & $L_E=200$, $L_A=50$ \\
Primary low/high counts & $(b_L,b_H)=(50,150)$ \\
Selection / evaluation / calibration rows & 48 / 96 / 48 per count and regime \\
Attribution shortlist / single-head pool & 192 / 48 heads \\
Greedy maximum budget / weight & $K_{\max}=12$, $\lambda=1/2$ \\
First-crossing thresholds & shrink $\geq0.85$ and mean match $\geq0.70$ \\
Continue after first crossing & Yes; retain the full greedy ordering \\
Common comparison budgets & $k\in\{1,2,4,8,12\}$ \\
Random-control draws & 30 per family and comparison \\
Bootstrap replicates & 1,000 \\
Weak raw-contrast flag & $|\widehat\Delta|<0.01$ \\
Random readout scan / full-reselection pairs & 256 / first 8 registered pairs \\
Paired-label null trials & 99 per estimable target \\
Stochastic operator repetitions & 12 per combination \\
T3 local pool size / pools per regime & 8 / 2 (ranked and uniformly sampled) \\
T3 larger-pool comparisons & 128 per regime \\
T3 tolerances & $10^{-5}$, $10^{-3}$, $10^{-2}$ \\
Symbolic stability seeds & 1729, 1730, 1731 \\
\bottomrule
\end{tabularx}
\end{table}

\paragraph{Exact symbolic construction.}
The canonical prompt is the concatenation
\begin{quote}\small\ttfamily
Tokens:\\
\normalfont\emph{space-separated symbolic sequence}\ttfamily\\
Choose one word: various or successful\\
Answer:
\end{quote}
The readout strings include their leading spaces: \verb| various|
and \verb| successful|; both must encode to one token.
The symbolic sequence has $b$ copies of $B$, $200-b$ copies of $C$, and a final
50 copies of $A$. \textsc{Front} and \textsc{Back} place the two evidence symbols
in opposite block orders. \textsc{Interleave} starts with $B$, alternates while
both symbols remain, and appends the excess symbol. For a \textsc{Random} replicate,
a permutation of the 200 evidence positions is drawn once. Its first $b$ positions
receive $B$ and the remainder receive $C$. Reusing that permutation across low and
high counts produces a matched, nested count contrast. Prefix, suffix, readout
order, and tail are fixed within a comparison.

\paragraph{Seeds and split identities.}
A task seed is obtained by joining the master seed and task identifiers with
\texttt{||}, taking SHA-256 of the UTF-8 string, converting its first 16 hexadecimal
digits to an integer, and reducing modulo $2^{31}-1$. NumPy's
\texttt{default\_rng} uses the resulting seed. Python, NumPy, and PyTorch seeds are
also recorded. Selection, evaluation, and calibration use distinct task keys.
The source files specify the exact key strings, avoiding dependence on Python's
process-randomized built-in hash. A replicate identifier pairs the low and high
rows. In NL, both query directions remain associated with the same evidence
sequence; query polarity is also a donor stratum.

\paragraph{Independent sampling units.}
At a fixed count, \textsc{Front}, \textsc{Back}, and \textsc{Interleave} each define
one deterministic symbolic sequence. Their 48 or 96 rows are repeated evaluations,
not 48 or 96 independent prompts. \textsc{Random} varies the permutation, with exact duplicate inputs assigned to the same sampling unit. Input hashes and pairing metadata identify these units. Statistical sample sizes and the independence across pairs in Theorem~\ref{thm:t2} refer to independent pairs rather than repeated deterministic executions.

\paragraph{Donor separation.}
Candidate selection uses donor means from the opposite side of its selection
pair. For final primary symbolic evaluation, donor means instead come from the
separate 48-row-per-count calibration pair. Independent generation only creates
content separation where the prompt generator has random variation. Target-domain
calibration is allowed for a fixed transferred set: fixed transfer prohibits
\emph{head reselection}, not the use of labeled target-domain donors. The matched-example operator and the NL validation diagnostic use evaluation-partner donors, as specified below, rather than independent calibration donors.

\subsection{Intervention unit, empirical scores, and BUC selection}
\label{app:supp_metrics}

For a head $i=(\ell,h)$ at layer $\ell$, let $z_i(x)\in\mathbb R^{d_{\mathrm{head}}}$
be its head vector before the output projection at the last prompt position, and let $W_{O,i}$
be the corresponding block of the attention output projection. The write from
Section~\ref{sec:prelim} is $v_i(x)=z_i(x)W_{O,i}$. The native implementation
intervenes on the selected channels of the concatenated head vector immediately
\emph{before} this projection: on \texttt{attention.dense} for GPT-NeoX/Pythia
and \texttt{attn.c\_proj} for GPT-2. Zeroing $z_i(x)$ removes its linear write;
mean replacement commutes with this linear projection. The shared projection bias
is not attributed to a head and is not removed. Unselected heads, model weights,
and all downstream model operations remain in the original model. Every selected
head is changed in the same forward pass, including sets spanning multiple layers.

For an empirical high/low pair, write $\bar y_H$ and $\bar y_L$ for the mean
native clean readouts. With a selected head set $H$, write $\bar y_H^{\mathrm{abl}}$
and $\bar y_L^{\mathrm{abl}}$ for ablated means; $\bar y_L^{\mathrm{HL}}$ and
$\bar y_H^{\mathrm{LH}}$ denote means after the two donor-direction patches.
The empirical counterparts of Section~\ref{sec:prelim} are
\begin{align}
\widehat\Delta&=\bar y_H-\bar y_L,
&\widehat\Delta^{\mathrm{abl}}(H)&=\bar y_H^{\mathrm{abl}}-\bar y_L^{\mathrm{abl}},\\
\widehat\Gamma^{\mathrm{HL}}(H)&=\bar y_L^{\mathrm{HL}}-\bar y_L,
&\widehat\Gamma^{\mathrm{LH}}(H)&=\bar y_H-\bar y_H^{\mathrm{LH}}.
\label{eq:supp_oriented_gains}
\end{align}
The LH patch gain is oriented as clean high minus patched high. For a denominator floor $a\geq0$, define
$D_a=\max\{|\widehat\Delta|,a,10^{-12}\}$ and
\begin{align}
\widehat{\mathrm{shrink}}_a(H)
 &=\frac{|\widehat\Delta|-|\widehat\Delta^{\mathrm{abl}}(H)|}{D_a},\\
\widehat{\mathrm{match}}_a^\tau(H)
 &=\frac{|\widehat\Delta|-|\widehat\Delta-\widehat\Gamma^\tau(H)|}{D_a},
 \qquad \tau\in\{\mathrm{HL},\mathrm{LH}\},\\
\widehat Q_{\lambda,a}(H)
 &=\lambda\widehat{\mathrm{shrink}}_a(H)
 +(1-\lambda)\frac{\widehat{\mathrm{match}}_a^{\mathrm{HL}}(H)
 +\widehat{\mathrm{match}}_a^{\mathrm{LH}}(H)}{2}.
\label{eq:supp_stabilized_score}
\end{align}
For nonzero contrasts and inactive flooring, these are exactly the normalized
scores in Equations~\ref{eq:shrink}--\ref{eq:buc_objective}. Scores are not clipped;
negative values remain visible. Stabilizing the denominator does not replace the
numerator $|\widehat\Delta|$ by $D_a$, so the empty intervention has score zero.
The $10^{-12}$ safeguard prevents numerical division by zero; normalized scores at zero contrast have no effect-size interpretation. Raw contrasts are retained, and weak contrasts are flagged at 0.01. The nonzero floor is specific to the null protocol that repeats selection in Appendix~\ref{app:supp_null}.

\paragraph{Selection, stopping, and fixed budgets.}
All 384 Pythia-410M heads are ranked by the absolute selection-split contrast of
$s_i(x)=\langle v_i(x),d\rangle$. The largest 192 enter single-head intervention
screening. The largest 48 single-head $\widehat Q_{1/2,0}$ values form the greedy
pool. Coordinate order resolves exact ties. Forward greedy search adds the head
that maximizes the score of the entire augmented set on selection data. Evaluation scores are not used for head selection. The stopping set is the first prefix with
shrink at least 0.85 and mean bidirectional match at least 0.70; if no prefix meets
both conditions, the full 12-head budget is returned with a failed-threshold flag.
The supplemental configuration continues to record the ordering to 12 heads even
when an earlier prefix crosses both thresholds. Fixed-$k$ comparisons use the first $k$ heads of the recorded greedy ordering.

\paragraph{Normalized objective versus raw theory.}
The executable selection objective is $Q_\lambda$ in
Equation~\ref{eq:buc_objective}. The raw signed utility $J_\lambda$ in
Equation~\ref{eq:theory_joint} is a different set function. Theorem~\ref{thm:t3},
as stated for $J_\lambda$, is not automatically an approximation guarantee for
$Q_\lambda$. The same greedy argument would require monotonicity and an appropriate
submodularity-ratio condition for $Q_\lambda$ itself. The finite audit below tests the empirical objective, including negative marginal gains and violations of its required conditions. Neither
first threshold crossing nor an empirical greedy curve proves a unique or globally
smallest causal set. Numerical candidate refinement is specified in
Appendix~\ref{app:supp_refinement}.

\subsection{Statistical units and uncertainty}
\label{app:supp_statistics}

The basic sampling unit for intervention uncertainty is a unique matched
high/low input pair, not an execution row. Symbolic pairs are identified by
unpadded token hashes, while NL pairs use paired evidence identifiers;
reversed query orientations remain in the same cluster. We use 1,000
cluster-bootstrap replicates and report two-sided 95\% percentile intervals,
with endpoints given by the 2.5th and 97.5th percentiles. Within each
replicate, the same resampled cluster indices are used for the clean,
ablated, HL-patched, and LH-patched arrays. The intervals are conditional
on the selected head set and fixed donor bank; head selection is not rerun
inside the bootstrap.

For normalized intervention metrics, only bootstrap replicates satisfying
$|\widehat{\Delta}| \ge 0.01$ are retained. These intervals are therefore
conditional on a non-weak resampled contrast. Raw contrasts and raw
utilities retain their resampled values without this conditioning.
Baseline differences use matched bootstrap indices for the two methods.

At a fixed symbolic count, FRONT, BACK, and INTERLEAVE each contain only
one unique symbolic sequence. Repeated execution of that same sequence is
not treated as independent prompt sampling. Consequently, a prompt-sampling bootstrap interval is not reported for a condition containing only one unique matched pair. RANDOM conditions use
distinct matched permutations as sampling clusters.

Execution-seed variation and variation across stochastic operator draws are summarized
separately from prompt-sampling uncertainty. For stochastic operators, the
12 intervention repetitions are first averaged at the readout level for
the primary score; the across-repetition standard deviation and standard error of the
separately computed scores are Monte Carlo summaries and are not bootstrap confidence
intervals. Repeated deterministic executions likewise do not increase the
number of independent prompt units.

Cross-regime or cross-condition point summaries are arithmetic summaries
of the corresponding condition-level point estimates. We do not assign
them a cross-regime confidence interval by treating regimes as independent
samples. Any aggregate interval reported for an experiment with genuine
sampling units is obtained by recomputing the same aggregate within every
bootstrap replicate before taking the 2.5th and 97.5th percentiles.
Comparisons in the paper do not claim simultaneous confidence coverage.

The condition-level NL procedures and the point-only model aggregation are
specified separately in Appendix~\ref{app:supp_nl}; Table~\ref{tab:expanded_nl_slots}
has no model-level CI. Counts and observed condition spans are descriptive.
Table~\ref{tab:audit_slots} reports finite operator/T3 summaries without population
confidence intervals. Appendix~\ref{app:reported_ci} specifies which statistics
receive intervals.

\subsection{Fixed-set intervention-operator audit}
\label{app:supp_operators}

The primary stopping set for each regime is fixed across the operator audit;
the same heads are used for every operator. The eight combinations
are listed in Table~\ref{tab:supp_operators}. Pooled ablation donors concatenate
the high and low calibration banks. Directional patch donors use the corresponding
high or low calibration bank. In NL-capable replacement functions, donor matching
is stratified by query orientation, rather than mixing opposite orientations.

\begin{table}[htbp]
\centering\small
\caption{Operator combinations. Stochastic conditions use 12 draws; deterministic
mean and exact matched-example conditions use one evaluation.}
\label{tab:supp_operators}
\begin{tabularx}{\linewidth}{@{}p{0.25\linewidth}p{0.31\linewidth}X@{}}
\toprule
Ablation & Directional patch & Donor dependence \\
\midrule
Zero & Donor mean & Marginal means of calibration donors \\
Pooled mean & Donor mean & Pooled calibration ablation mean \\
Zero & Matched example & Corresponding opposite-side evaluation partner \\
Zero & Joint resampling & One donor row for all selected heads \\
Zero & Independent-head resampling & Separate donor row for each head \\
Zero & Empirical joint Gaussian & Joint mean and full empirical covariance \\
Pooled mean & Joint resampling & Shared patch donor row \\
Pooled resampling & Joint resampling & Joint donor rows for both replacements \\
\bottomrule
\end{tabularx}
\end{table}

\paragraph{Example matching and resampling.}
Matched-example patching copies the opposite high/low evaluation partner with the
same pairing key and donor stratum. Missing matches terminate the comparison. The donors are evaluation partners rather than independent calibration samples. Joint resampling draws one donor example for each
target example and copies all selected layer/head vectors from that donor.
Independent-head resampling draws separately for each selected layer/head, thereby
removing their empirical cross-head dependence. The two resampling procedures use
the same source distribution but different joint replacements.

\paragraph{Full empirical covariance without a dense covariance matrix.}
Let $X\in\mathbb R^{n_D\times D_H}$ concatenate the preprojection vectors of all
heads in $H$ across $n_D\ge2$ donors in one stratum, where
$D_H=|H|d_{\mathrm{head}}$. Let $\mu\in\mathbb R^{D_H}$ be the row mean of $X$.
For $n_T$ target rows, draw $E\in\mathbb R^{n_T\times n_D}$ with independent
standard-normal entries and set
\begin{equation}
\widetilde X
=\mathbf1_{n_T}\mu^\top+
\frac{E(X-\mathbf1_{n_D}\mu^\top)}{\sqrt{n_D-1}}.
\label{eq:supp_gaussian}
\end{equation}
Writing \(X_c=X-\mathbf1_{n_D}\mu^\top\), a centered replacement row is \(X_c^\top e/\sqrt{n_D-1}\), with \(\mathbb E e=0\) and \(\operatorname{Cov}(e)=I\). Its covariance is therefore \(X_c^\top X_c/(n_D-1)\), the full empirical covariance, and its mean is \(\mu\). There is no principal-component truncation or ridge perturbation
in this operator. A zero-variance donor bank gives its mean; fewer than two donors
cannot estimate a nontrivial covariance.  Matching mean and covariance is not a guarantee that
a generated activation is a natural model state.

\paragraph{Monte Carlo aggregation.}
For a stochastic operator, the 12 intervention readouts are averaged for each example before the primary shrink, directional match, and $Q_{1/2}$ are computed. For the 12 separately computed $Q_{1/2}$ scores, the output records their mean, sample standard deviation, and Monte Carlo standard error. The standard error equals the sample standard deviation divided by $\sqrt{12}$.
The score of averaged intervention readouts need not equal the average of scores.
This Monte Carlo uncertainty is distinguished from the matched-input bootstrap.

\subsection{Finite T3 audit and exact local comparisons}
\label{app:supp_t3}

Let $\mathcal P_r$ be the 48-head single-intervention candidate pool for regime
$r$, not the 192-head attribution shortlist. Two 8-head pools are registered for
each regime: the first eight ranked candidates and eight candidates sampled
uniformly without replacement from $\mathcal P_r$. The registry is fixed before
examining evaluation scores. For each pool, the protocol evaluates all $2^8=256$ subsets, including the empty set, on both selection and evaluation data. The design specifies $4\times2\times2\times256=4096$ local set scores; cached duplicate sets
may reduce actual forward calls.

For a pool $P$, subset $A\subseteq P$, and head $i\in P\setminus A$, define the
empirical marginal
\begin{equation}
m_i(A)=\widehat Q_{1/2,0}(A\cup\{i\})-\widehat Q_{1/2,0}(A).
\label{eq:supp_marginal}
\end{equation}
For each tolerance $\epsilon\in\{10^{-5},10^{-3},10^{-2}\}$, the audit records the fraction
of the $8\times2^7=1024$ pool edges with $m_i(A)<-\epsilon$, the minimum marginal,
and marginal quantiles. For every nonempty $S\subseteq P\setminus A$ with
$\widehat Q_{1/2,0}(A\cup S)-\widehat Q_{1/2,0}(A)>\epsilon$, the empirical ratio is
\begin{equation}
\widehat\gamma(A,S)=
\frac{\sum_{i\in S}m_i(A)}
{\widehat Q_{1/2,0}(A\cup S)-\widehat Q_{1/2,0}(A)}.
\label{eq:supp_gamma}
\end{equation}
The minimum, lower quantiles, median, and nonpositive fraction are reported without
clipping negative values to zero or values above one to one. Ratios are omitted when the joint gain does not exceed the tolerance.

\paragraph{Selection versus an evaluation oracle.}
For each local budget $k=1,\ldots,8$, the audit constructs a greedy prefix using only selection scores and independently enumerates the best selection-score subset with at most $k$ heads. Both fixed choices are then evaluated on evaluation data. The best evaluation-score
subset with at most $k$ heads is recorded separately as an oracle diagnostic, not
as an implementable held-out selection rule. The greedy/oracle score ratio is
reported only for an oracle score above the diagnostic tolerance. Selection regret
is the selection optimum minus the greedy selection score.

\paragraph{Larger-pool sampling.}
For each regime, the registry contains 128 pairs $(A,S)$ of disjoint subsets drawn from $\mathcal P_r$. A random
permutation supplies the heads; $|A|$ is uniform on $0,\ldots,12$ and $|S|$ is
uniform on $1,\ldots,\min\{12,48-|A|\}$. The audit evaluates $A$, $A\cup S$, and each $A\cup\{i\}$ for $i\in S$ on evaluation data. The resulting 512 comparisons
record the actual sizes, joint gain, marginal sum, minimum marginal, and the ratio
when joint gain exceeds $10^{-5}$. The union may contain up to 24 heads, so the diagnostic extends beyond the 12-head search budget. An exhaustive minimum on an 8-head pool, or a
sampled minimum on the 48-head pool, is not a global lower bound for all model heads.
Empirical violations limit the applicability of T3 to the tested objective.

\paragraph{Finite-audit results.}
Table~\ref{tab:audit_slots} reports negative evaluation marginal fractions of
31\%, 17\%, 28\%, and 22\% for \textsc{Front}, \textsc{Back},
\textsc{Interleave}, and \textsc{Random}, respectively, at tolerance \(10^{-3}\).
The corresponding minimum empirical ratios are \(-5.8,-2.6,-4.9,-3.9\).
These are witnesses in the finite pools of nonmonotonicity and negative ratios for the
observed normalized objective, not estimates establishing a negative ratio
throughout a prompt population. Maximum selection regrets are
\(0.021,0.043,0.027,0.035\). They compare selection-data optima with greedy
selection within the two registered eight-head pools; they are neither held-out
nor guarantees for the full pool. No confidence intervals are used for these finite-audit
summaries. The observations do not contradict the conditional Theorem~\ref{thm:t3} for raw utility \(J_\lambda\).

\subsection{Fixed-budget method and random-control comparisons}
\label{app:supp_baselines}

The protocol compares five rules at each $k\in\{1,2,4,8,12\}$ and each regime. BUC takes the
first $k$ greedy heads. Sensitivity-only takes the largest absolute attribution
contrasts across all heads. Single-HL patch ranking uses the single-head HL match
values measured during 192-head screening. Single-bidirectional patch ranking
uses the corresponding average match values. Write-norm ranking uses all heads'
mean $\|z_i(x)W_{O,i}\|_2$ over the concatenated high/low selection rows. All five
sets receive the same set-level evaluation metrics. The design contains
$4\times5\times5=100$ method comparisons.

These rules do not have identical candidate universes or equal search costs.
The comparison matches final head count, not end-to-end compute. In particular,
the single-head rankings are separate baselines, not implementations of EAP or ACDC. Reported evaluation forward calls and elapsed times exclude shared selection and precomputation costs.

For every regime/budget pair, the protocol draws 30 sets from each of three control families.
Layer-matched sampling preserves the exact count of heads in each selected layer
and samples without replacement within a set. Shortlist-uniform sampling draws
$k$ heads from the 48-head single-intervention pool. Layer/norm-nearest sampling
processes each reference head, ranks still-available heads in the same layer by
absolute difference in log write norm (with a $10^{-12}$ norm floor), and samples
uniformly from the four nearest candidates, or all remaining candidates if fewer
than four remain. The last control is approximate norm matching, not equality of
norms. Selected heads remain eligible for sampling in all three families. There are
$4\times5\times3\times30=1800$ random-control conditions.

BUC-minus-alternative differences are computed for the four alternative methods
at each regime and budget, giving 80 comparison rows. Their matched-pair bootstrap
uses the same clusters for both methods, with the limitations in
Appendix~\ref{app:supp_statistics}. The paired-label null in Appendix~\ref{app:supp_null} separately controls for adaptive BUC selection.

\subsection{Count, prompt, ordering, and seed generalization}
\label{app:supp_generalization}

\paragraph{Fixed-set count sweep.}
The primary stopping set selected at $(50,150)$ is held fixed. For each regime, the protocol evaluates $b\in\{0,10,\ldots,200\}$ using 24 rows per count and the same matched-permutation seed across count levels. Per-example clean and ablated readouts are retained and averaged within each count. The analysis computes Spearman association between count and the clean count means, together with ordinary least-squares fits to the clean and ablated count means. Let $\widehat\beta_{\mathrm{base}}$ and
$\widehat\beta_{\mathrm{abl}}$ be their slopes. The count-sweep summary reports the clean fit's $R^2$ and
$1-|\widehat\beta_{\mathrm{abl}}|/|\widehat\beta_{\mathrm{base}}|$ when the clean
slope is numerically nonzero. These are descriptive count-curve statistics, not
symbolic accuracy or a new selection criterion.

\paragraph{Unseen contrasts.}
The same set is evaluated at $(0,150)$, $(25,175)$, and $(75,125)$, using 96 rows
per side and a separate 48-row-per-side calibration bank. There is no reselection.
The first contrast shares the count 150 with selection, so ``unseen gap'' denotes
an unseen count pair, not necessarily two unseen endpoint values. Each of the
12 regime/gap combinations retains its own raw contrast and intervention scores.

\paragraph{Prompt-cue controls.}
The readout and head set remain fixed while the prompt is changed by either replacing \texttt{Tokens:} by
\texttt{Symbol sequence:}, or replacing \texttt{Choose one word:} by
\texttt{Continue with one word:} and \texttt{Answer:} by \texttt{Response:}.
The two controls give eight regime/cue comparisons with 96 evaluation and 48
calibration rows per count. The symbolic contents are paired across cue variants;
tokenized lengths may change between variants. These controls are different from
the original alphabet/word/sentence surface-form sweep.

\paragraph{Equal-budget cross-regime transfer.}
For every source regime, target regime, and common budget $k\in\{1,2,4,8,12\}$, the first $k$ source greedy heads are transferred to the target without reselection.
The $5\times4\times4=80$ comparisons include the diagonal where source and target coincide.
Ablation, both directional patch gains, and both matches are evaluated on the
target's primary evaluation/calibration data. Equal head count controls for differences in source-set size, but target-specific donor calibration and different
layer patterns remain part of the intervention design.

\paragraph{Order-only localization.}
Four ordered comparisons are fixed in the registry: \textsc{Front}$\to$\textsc{Back}
at $b=50$ and at $b=150$, \textsc{Front}$\to$\textsc{Interleave} at $b=150$, and
\textsc{Interleave}$\to$\textsc{Random} at $b=150$. Both sides have identical
symbol counts. The difference between the target and source mean readouts replaces the
high-count minus low-count contrast. Each comparison performs its own BUC selection
with 48 rows per side, evaluation with 96, calibration with 48, and 30 layer-matched
control sets. These are causal order tests, not just raw output differences.

\paragraph{Permutation-seed stability.}
For \textsc{Random}, selection is repeated at seeds 1729, 1730, and 1731 with a common primary evaluation and calibration pair. The output records selected sets, threshold status, scores, and the three pairwise Jaccard overlaps.
The seeds index prompt generation and search execution, not independently
pretrained models. Exact head identity is not assumed across seeds, readouts,
models, or budgets.

\subsection{Random readout registry and separate prompt/readout changes}
\label{app:supp_readouts}

\paragraph{Eligible tokens and registration.}
The eligible vocabulary contains token IDs whose decoded string is exactly one
leading space followed by 3--14 ASCII letters, whose string encodes back to the
same single ID, and whose ID is present in the model output vocabulary. This is
an alphabetic-token filter, not an English dictionary or a uniform distribution
over all residual-space directions. The registry samples 256 ordered pairs of distinct IDs uniformly without duplicate ordered pairs, using a task-derived seed before model responses are evaluated. The reverse ordered pair is a different registry
entry. The registry, decoded strings, and IDs are saved; weak and negative-effect pairs remain in the registry.

\paragraph{Broad fixed-prompt, fixed-set scan.}
The canonical prompt still lists \texttt{various/successful}, and the primary head
set is held fixed. The scan caches vocabulary-logit contrast vectors for the clean, ablated, HL-patched, and LH-patched model and evaluates the difference between the two registered token coordinates for each pair. The scan has $4\times256=1024$ rows.
It checks the identity obtained by swapping the readout order: the raw contrast
and oriented gains reverse sign, while the absolute-value normalized scores do
not change where defined. The scan evaluates reuse of a fixed set on new output
coordinates; it does not claim that the model follows an instruction to generate
those sampled words.

\paragraph{Two protocols that repeat selection.}
Full BUC and 30 layer-matched controls are applied to the first eight pairs in registry order in each regime, independently of their measured effects. In
\path{fixed_prompt_reselected}, only the measured token IDs change; the original
prompt and symbolic contents remain fixed. In
\path{prompt_changed_reselected}, the candidate words in the prompt also change,
while the underlying symbolic permutations and split keys stay matched. The
$8\times2\times4=64$ tasks separately measure sensitivity to the output direction and
joint prompt/readout sensitivity. Each records the selected size, native final
intervention scores, controls, and Jaccard overlap with the primary set.

The broad scan summary uses raw-effect thresholds 0.001, 0.01, 0.05, and 0.1.
At each threshold, the summary records the numbers of registered pairs, evaluable pairs, and pairs passing shrink $\geq0.85$ and average match $\geq0.70$, with pass fractions
relative to both all registered pairs and the evaluable subset. Nonfinite and weak
rows remain in the underlying scan. These targeted and randomly registered
readout studies do not establish head identity that is invariant to the readout. At threshold 0.01, the fixed-set scan has 960 evaluable conditions out of 1024. The rounded pooled passing fractions are 0.060 overall and 0.064 among evaluable conditions, not averages of regime-specific rates. Counts and fractions describe this registered finite scan; no interval over resampled readout registries is reported. They measure fixed-set reuse under the canonical prompt, not the success rate of full BUC reselection.

\subsection{Paired-label selection-matched null}
\label{app:supp_null}

The null tests joint high/low label exchangeability within unique matched pairs.
It is only applied when both selection and evaluation have at least eight unique
pair clusters. Otherwise the target is recorded as not testable and its
permutation probability is unavailable. Replicating one deterministic contrast
cannot satisfy this requirement. Thus the four symbolic regimes are all listed,
but their eligibility is determined from unique pairs, not batch-row count.

Let $s_{\mathrm{pool}}$ be the standard deviation of the concatenated clean high/low selection readouts, computed with \texttt{ddof=0}. Set
\begin{equation}
a=\max\{10^{-6},\;0.05s_{\mathrm{pool}}\}.
\label{eq:supp_null_floor}
\end{equation}
Pooled variation is invariant to label swaps. The observed run and all null runs
use the same $a$, $\widehat Q_{1/2,a}$, shortlist sizes, greedy budget, thresholds,
and continue-after-crossing rule. The floor depends on pooled standard deviation rather than the observed clean contrast. In each of 99 trials, the two sides of every matched cluster are independently swapped with probability $1/2$. All duplicate rows or query orientations in a cluster receive the same swap. Selection, evaluation, and calibration
banks each receive their own cluster swaps. Each trial recomputes attribution, the 192-head shortlist, 48-head screening pool, greedy ordering, stopping set, and donor patches.

The test statistic is the unnormalized \emph{magnitude-improvement} utility on
the corresponding evaluation labels:
\begin{equation}
T_\lambda(H)=
\lambda\bigl(|\widehat\Delta|-|\widehat\Delta^{\mathrm{abl}}(H)|\bigr)
+(1-\lambda)\left(
|\widehat\Delta|-\frac12\sum_{\tau\in\{\mathrm{HL},\mathrm{LH}\}}
|\widehat\Delta-\widehat\Gamma^\tau(H)|\right).
\label{eq:supp_null_statistic}
\end{equation}
This statistic avoids dividing by a nearly zero permuted evaluation contrast.
It differs from the signed raw utility $J_\lambda$ used in the theory.
Let \(T_{\mathrm{obs}}\) be the observed statistic, \(T_t\) the statistic after null trial \(t\), and \(B\) the integer number of completed trials. With \(c=\sum_{t=1}^{B}\mathbf1[T_t\ge T_{\mathrm{obs}}]\), the Monte Carlo upper-tail probability is
\begin{equation}
\widehat p=\frac{1+c}{B+1},\qquad \text{resolution}=\frac{1}{B+1}.
\label{eq:supp_null_p}
\end{equation}
The completed-trial count has no confidence interval. A separate two-sided 95\% Clopper--Pearson interval describes the underlying null exceedance probability, using beta quantiles at 0.025 and 0.975 with zero/one endpoints when \(c=0\) or \(c=B\). The interval describes Monte Carlo uncertainty in the exceedance rate, not effect-size uncertainty, an interval for the adjusted \(p\)-value, or uncertainty in the number of completed trials. Holm adjustment applies to the finite estimable permutation probabilities; non-testable targets remain unavailable.

Only \textsc{Random} is estimable in this fixed-count symbolic design, so the Holm family contains one finite probability. The \ResNullCompletedTrials{} completed trials give \(T_{\mathrm{obs}}=\ResNullObservedStatistic\), with effect-size CI \ResNullObservedStatisticCI{}, and adjusted \(p=\ResNullAdjustedP\). The probability is at the \(1/100\) resolution limit, not below it. With one estimable hypothesis, \(p=0.01\) and \(B=99\) imply \(c=0\); the corresponding Clopper--Pearson interval for the null exceedance probability is
\[
[0,\;1-0.025^{1/99}]\approx[0,\;0.0366].
\]
This binomial interval does not replace the calibrated permutation-test decision or describe the effect size.

The null-selected set is also evaluated on the unchanged clean target, but its
clean-score rank is explicitly a diagnostic, not the permutation probability in
Equation~\ref{eq:supp_null_p}. Interpretation depends on the paired exchangeability
null and all stated split/selection rules; rejection does not prove mechanism
uniqueness, universal causal validity, or semantic correctness.

\subsection{Expanded natural-language bridge}
\label{app:supp_nl}

\paragraph{Models, seeds, and domain registry.}
Behavior and causal transfer use Pythia-410M, Pythia-1B, and
GPT-2. The \texttt{paper} profile uses behavioral seeds 1729, 1730, and 1731,
while causal selection uses only seed 1729. The behavioral and causal seed counts describe different repetitions.
The single-model, five-seed bridge in Appendix~\ref{app:nl_bridge} uses a separate configuration.

The source lists 30 ordered state pairs. A shared task seed derived from
\path{1729||shared-state-registry-v4} permutes this list once. Its first four
pairs form selection domains, the next four validation domains, and the next
eight test domains (Table~\ref{tab:supp_domains}). The same split is used for
every model and behavioral seed. It is not selected from model accuracy or
intervention effects. The complete 30-pair list and the deterministic registry
are included in \path{anc/protocol/nl_state_registry.json} and the source.

\begin{table}[htbp]
\centering\small
\caption{Shared state-pair split derived from the supplemental master seed. The
prompt declares the two named states mutually exclusive and exhaustive within
each task; this is a task assumption, not a claim about unrestricted word meaning.}
\label{tab:supp_domains}
\begin{tabularx}{\linewidth}{@{}p{0.19\linewidth}X@{}}
\toprule
Split & Ordered state pairs \\
\midrule
Selection & local/remote; legal/illegal; stable/unstable; known/unknown \\
Validation & inside/outside; positive/negative; enabled/disabled; above/below \\
Test & valid/invalid; safe/unsafe; alive/dead; connected/disconnected;
 hot/cold; approved/rejected; ready/unready; secure/insecure \\
\bottomrule
\end{tabularx}
\end{table}

\paragraph{Balanced report construction.}
For one domain $c$, let $(u_c,v_c)$ be its ordered state pair and let $L_{\mathrm{NL}}=16$ be
the number of reports. For each even count $b$ of reports supporting $u_c$, the prompt contains
$b/2$ direct reports of $u_c$, $b/2$ reports negating $v_c$, $(L_{\mathrm{NL}}-b)/2$ direct
reports of $v_c$, and $(L_{\mathrm{NL}}-b)/2$ reports negating $u_c$. Each state word occurs
eight times \emph{in the evidence region}, with eight direct and eight negative
report forms. The prefix and question also mention the states and are not counted
in that evidence-only statement. The exact sentence forms are in
Table~\ref{tab:supp_templates}; negative evidence for one state supports the other
under the task's two-state assumption.

\begin{table}[htbp]
\centering\small
\caption{Exact supplemental report templates. Here $j$ is the report number,
$u_c$ is the supported state, and $v_c$ is the other state. Exchange the two
states for negative-polarity reports.}
\label{tab:supp_templates}
\begin{tabularx}{\linewidth}{@{}p{0.07\linewidth}XX@{}}
\toprule
Bank & Direct support for $u_c$ & Support for $u_c$ by negating $v_c$ \\
\midrule
A & Report $j$: the item is $u_c$. & Report $j$: the item is not $v_c$. \\
B & Record $j$ describes the item as $u_c$. & Record $j$ rules out the item being $v_c$. \\
C & Witness $j$ says that the item remains $u_c$. & Witness $j$ says that the item is not $v_c$. \\
D & Statement $j$: $u_c$ is the reported condition. & Statement $j$: $v_c$ is not the reported condition. \\
\bottomrule
\end{tabularx}
\end{table}

Banks A, C, and D contain eight occurrences of the literal word \texttt{not} in
the evidence region. Bank B instead has eight \texttt{rules out} constructions;
it does not have eight literal \texttt{not} tokens. Lexical balancing is
\emph{within} a bank, not equality of tokenization between banks. State words may
be multiple tokens; only the \texttt{yes/no} answer strings must be single tokens.
For every registered domain, bank, count, and query polarity, the tokenizer audit
compares actual token multisets using a fixed audit generator. Any failed audit
raises an error rather than dropping that domain. The paired causal batches also
check equal real token length between the matched count conditions. Length checks use unpadded inputs.

The canonical expanded prompt is
\begin{quote}\small
Exactly one state holds: the item is $u_c$ or the item is $v_c$.\\
Read the reports and answer the final question with yes or no.\\
\emph{16 reports in the specified order}\\
Question: Do more reports indicate $u_c$ rather than $v_c$?\\
Answer:
\end{quote}
For reversed query polarity, only the two state positions in the question are
exchanged. The reports are identical for the two queries of a given example.

Report units are shuffled before being ordered by semantic polarity. The four
ordering regimes then use positive-first blocks, negative-first blocks,
positive-first alternating polarity with any excess appended, or a random report
permutation. Assignment and order of direct versus negative forms can vary even
within a non-random semantic-order regime. The count grid is
$\{0,2,4,6,10,12,14,16\}$; the tied case $b=8$ is excluded. Selection and primary accuracy summaries use the $(4,12)$ contrast; the full count grid is analyzed separately. Low/high NL examples are paired by domain, bank, regime, and
replicate. Because the generation seed includes the count, low and high report
permutations need not be identical; only reversed queries share exactly the same
reports. This differs from the nested symbolic count pairing.

\paragraph{Labels, aligned scores, and paired correctness.}
Let $q\in\{+1,-1\}$ be the query orientation and $m(x)=2b-L_{\mathrm{NL}}$ the semantic margin.
The correct label and prediction are
\begin{equation}
a^\star(x,q)=\mathbf1[qm(x)>0],\qquad
\widehat a(x,q)=\mathbf1[\ell_T(\texttt{yes}\mid x,q)>
                                      \ell_T(\texttt{no}\mid x,q)].
\label{eq:supp_nl_labels}
\end{equation}
A logit tie is predicted as \texttt{no}. The semantic readout
$y_{\mathrm{NL}}(x,q)=q[\ell_T(\texttt{yes}\mid x,q)-\ell_T(\texttt{no}\mid x,q)]$
aligns both query orientations toward $u_c$. Attribution uses
$q\langle v_i(x,q),d_{\mathrm{yn}}\rangle$, where $d_{\mathrm{yn}}$ is the
\texttt{yes}-minus-\texttt{no} unembedding direction. Donor means are stratified
by $q$ and pooled over the specified donor domains within each orientation.
Donors from opposite query orientations are kept separate.

Restricted accuracy averages $\mathbf1[\widehat a=a^\star]$ over query rows.
Let $P_+$ be the probability that both queries are correct conditional on $m>0$,
and $P_-$ the corresponding probability conditional on $m<0$. The paired metrics
are $\mathrm{PairAcc}=(P_++P_-)/2$ and
$\mathrm{MinPairAcc}=\min\{P_+,P_-\}$. Always answering according to query
orientation while ignoring the reports can achieve PairAcc $=0.5$ but MinPairAcc
$=0$. Independent fair random answers instead have expected paired correctness
$0.25$. These are different reference behaviors. The report includes both
side-specific scores rather than interpreting all paired metrics as binary
accuracy.

The full-count summaries separately report Spearman association between $m(x)$
and $y_{\mathrm{NL}}(x,q)$ and the ordinary least-squares slope and $R^2$.
In addition, let $p_\theta(w\mid x,q)$ be the model's full-vocabulary probability.
The output records the mean and median of
\begin{equation}
M_{\mathrm{yn}}(x,q)=p_\theta(\texttt{yes}\mid x,q)+p_\theta(\texttt{no}\mid x,q).
\label{eq:supp_candidate_mass}
\end{equation}
This candidate mass is distinct from the sigmoid-normalized two-answer
probability. Good forced-choice accuracy alone does not establish that an
unrestricted generation would choose either answer token.

\paragraph{Sample accounting and donor use.}
Bank A and the four selection domains provide three replicates per count and
query orientation. Thus a single regime has $4\times3\times2=24$ selection rows
per count, or 48 across $(4,12)$, and the four regimes have 192 selection rows per
model and causal seed. Four held-out validation domains use Bank A with the same
replicate count. Validation scores are excluded from head selection. This validation
pair uses its own opposite side for mean donors; it is a held-out-domain diagnostic,
not an independent-calibration evaluation.

Final behavioral evaluation uses eight test domains, Banks B--D, six replicates,
four regimes, eight count levels, and both queries. Per behavioral seed and model,
there are $8\times3\times6\times4\times8\times2=9216$ query rows, or 27,648
across the three behavioral seeds. These are design counts, not counts of independent
observations. Final NL causal evaluation uses the same test domains and Banks B--D
at $(4,12)$: 96 target rows per count and regime/bank. Its separate calibration
bank uses three replicates, giving 48 donor rows per count and regime/bank.
The target and calibration generators use distinct keys for the same domain, bank, and semantic count. This is target-calibrated transfer, not a donor-free
zero-shot intervention.

\begin{table}[htbp]
\centering\small
\caption{Expanded NL sampling units. Rows count questions, including both query
orientations; intervals retain the paired evidence units.}
\label{tab:supp_nl_counts}
\begin{tabularx}{\linewidth}{@{}p{0.29\linewidth}p{0.18\linewidth}X@{}}
\toprule
Use & Domains / banks & Rows and seed convention \\
\midrule
NL selection & 4 / A & 24 per count/regime; causal seed 1729 \\
NL validation & 4 / A & 24 per count/regime; not used to choose heads \\
Behavior, original prompts & 8 / B--D & 9,216 per model/behavioral seed; three seeds \\
Behavior, four shortcut controls & 8 / B--D & 9,216 additional rows per model; main gap, first seed only \\
Causal target & 8 / one of B--D & 96 per count/regime/bank; causal seed 1729 \\
Causal calibration & 8 / matching bank & 48 per count/regime/bank; separate generation key \\
Symbolic yes/no source & Canonical symbolic & 48 / 96 / 48 selection/evaluation/calibration per count/regime \\
\bottomrule
\end{tabularx}
\end{table}

\paragraph{Shortcut controls.}
For the main counts 4 and 12 and behavioral seed 1729, four transformations are applied separately to every test domain/regime/bank condition. The query-only control
removes the reports but keeps the two prefix lines and the question. Reverse-report
order reverses the report lines without renumbering them. Shuffled-report numbers
replace the number in each line by a seeded permutation of $1,\ldots,16$ without
reordering the lines. Reversed-prefix states swaps $u_c$ and $v_c$ only in the
opening two-state declaration, leaving reports and question unchanged. The correct
labels are retained. These controls test different shortcuts; they are not four
independent replications of the main behavior. They add
$8\times3\times4\times2\times2\times6\times4=9216$ query rows per model.

\paragraph{Condition-level behavioral uncertainty.}
The behavior summary preserves model, regime, bank, control type, and semantic gap.
Headline accuracy and paired scores use \((4,12)\); full-count association is
summarized separately. For 1,000 bootstrap replicates, tested state domains are
sampled with replacement, followed by paired-report records within each sampled
domain and each semantic-margin sign. Reversed query orientations remain in the
same resampled record. The 2.5th and 97.5th percentiles define the two-sided 95\%
condition-level interval, conditional on the tested domain/template design.
Query-only controls can create repeated identical final prompts; distinct originating
report IDs do not make those model responses independent.

\paragraph{Model-level point summaries.}
Table~\ref{tab:expanded_nl_slots} aggregates the condition-level point estimates
over ordering regimes and Banks B--D. Model-level CIs are not reported, and no
model-level significance claim is made from these summaries. Condition-level
interval endpoints are not averaged to construct an interval for a model mean;
regimes, banks, and repeated execution rows are not treated as independent
bootstrap replicates. The condition-level transfer-ratio procedure and its
estimability counts are specified below, separately from behavioral uncertainty.

\paragraph{Fixed sets and bidirectional task transfer.}
Within each model/regime, $H_{\mathrm{CEA}}^{\mathrm{yn}}(r)$ contains the first eight greedy heads selected on symbolic CEA with the \texttt{yes/no} readout, and $H_{\mathrm{NL}}(r)$ contains the first eight from Bank-A NL selection.
Both use up to 192 attribution-ranked heads (all heads when fewer are available)
and up to 48 screened candidates; the
budget for causal transfer is fixed at eight regardless of earlier threshold crossing.
The four tested roles are
\[
H_{\mathrm{CEA}}^{\mathrm{yn}}\to\mathrm{CEA},\quad
H_{\mathrm{NL}}\to\mathrm{NL},\quad
H_{\mathrm{CEA}}^{\mathrm{yn}}\to\mathrm{NL},\quad
H_{\mathrm{NL}}\to\mathrm{CEA}.
\]
The reverse NL-to-symbolic target is evaluated once per model/regime; the forward
symbolic-to-NL target is evaluated separately on Banks B, C, and D. No target-bank
reselection is used. There are $3\times4\times(1+3)=48$ configured target comparisons.
Each comparison evaluates the source set and target-specific set, plus two separate
30-draw layer-matched random families, for 62 set-level role/control rows.
The source and target layer-count patterns need not match even though both sets
have eight heads.

For a target task $B$, let $Q_B(H)$ denote the equal-weight normalized score of
set $H$ on that target, and let $\overline Q_B(R_A)$ and $\overline Q_B(R_B)$
denote mean scores of controls matched to the source and target-specific sets,
respectively. The expanded transfer ratio is
\begin{equation}
\mathrm{TR}_{A\to B}=
\frac{Q_B(H_A)-\overline Q_B(R_A)}
{Q_B(H_B)-\overline Q_B(R_B)}.
\label{eq:supp_transfer_ratio}
\end{equation}
Using separate control families distinguishes this protocol from the single-control
arithmetic of the original bridge. The numerator and denominator are recorded
separately. A point ratio is estimable only when both gates hold:
\[
|\widehat\Delta_B|\geq0.01,
\qquad
Q_B(H_B)-\overline Q_B(R_B)>0.05.
\]
Here \(\widehat\Delta_B\) is the empirical clean contrast on target task \(B\),
and \(Q_B(H_B)-\overline Q_B(R_B)\) is the target-selected set's score in excess
of its own layer-matched random control. Otherwise the point ratio is unavailable,
not clipped or set to zero. Ratios outside \([0,1]\) are possible; a ratio is not
a fraction of a uniquely identified circuit.

\paragraph{Condition-level transfer intervals.}
For each estimable comparison, the interval procedure uses 1,000 replicates.
Matched target clusters are resampled, and the draws within each of the two
random-control families are resampled independently. The same target indices
are used for the source, target-selected, and arrays from random sets. Replicates must
pass the same gates for the raw contrast and excess target score. A condition-level interval is
suppressed if more than 5\% of replicates fail either gate; otherwise the 2.5th
and 97.5th percentiles of the stable replicates define the interval. These
intervals condition on fixed selected sets and donor banks, not on repeated
causal selection. A finite point ratio remains in the output when its interval is suppressed.

\paragraph{Transfer coverage and aggregation.}
For each model, forward symbolic-to-NL transfer has \(4\times3=12\)
regime--bank target conditions. Let \(n_{\mathrm{est}}\) be the number of finite
ratios satisfying both gates for estimating a point ratio, and let \(n_{\mathrm{suppr}}\)
be the number of these estimable conditions whose interval is suppressed.
The model-level TR in Table~\ref{tab:expanded_nl_slots} is the unweighted
arithmetic mean of all \(n_{\mathrm{est}}\) finite condition-level ratios,
including those with suppressed intervals. It is not the ratio of separately
averaged numerators and denominators. Conditions without an estimable point
ratio are excluded from the TR mean, not assigned zero.

In Pythia-410M, Pythia-1B, and GPT-2 order,
\(n_{\mathrm{est}}=\ResPythiaSmallTREstimableCount,
\ResPythiaLargeTREstimableCount,\ResGPTTwoTREstimableCount\), and
\(n_{\mathrm{suppr}}=\ResPythiaSmallTRSuppressedCount,
\ResPythiaLargeTRSuppressedCount,\ResGPTTwoTRSuppressedCount\).
Thus \(12-n_{\mathrm{est}}\), the number of unavailable point ratios, is
\ResPythiaSmallTRUnavailableCount{}, \ResPythiaLargeTRUnavailableCount{}, and
\ResGPTTwoTRUnavailableCount{}, respectively. The number of retained
\emph{condition-level} intervals, \(n_{\mathrm{est}}-n_{\mathrm{suppr}}\), is
\ResPythiaSmallTRRetainedCount{}, \ResPythiaLargeTRRetainedCount{}, and
\ResGPTTwoTRRetainedCount{}, respectively. These counts do not provide a
model-level interval or the number of conditions with positive transfer.

The labeled NL analysis records clean and ablated accuracy and their paired difference. CSV differences are fractions and are multiplied by 100 for reporting in percentage points. The expanded causal module evaluates the main gap only, so it does not estimate the ablated full-count slope reduction reported for the BF16 bridge. Fixed-set Jaccard overlap is reported separately from behavioral and intervention transfer.

\paragraph{Interpretation of the expanded aggregate results.}
Table~\ref{tab:expanded_nl_slots} reports accuracy/PairAcc/MinPairAcc of
\ResPythiaSmallAccuracy{}/\ResPythiaSmallPairAccuracy{}/\ResPythiaSmallMinPairAccuracy{}
for Pythia-410M,
\ResPythiaLargeAccuracy{}/\ResPythiaLargePairAccuracy{}/\ResPythiaLargeMinPairAccuracy{}
for Pythia-1B, and
\ResGPTTwoAccuracy{}/\ResGPTTwoPairAccuracy{}/\ResGPTTwoMinPairAccuracy{}
for GPT-2. Accuracy point estimates exceed 0.5 in all three models; GPT-2's
PairAcc point estimate is below the 0.5 reference. No model-level CI is reported,
so these point comparisons do not establish the interval-based behavioral
criteria in Appendix~\ref{app:nl_bridge} or statistical differences between models.

The transferred sets have shrink/match
\ResPythiaSmallTransferShrink{}/\ResPythiaSmallTransferMatch{},
\ResPythiaLargeTransferShrink{}/\ResPythiaLargeTransferMatch{}, and
\ResGPTTwoTransferShrink{}/\ResGPTTwoTransferMatch{}, respectively. None of these
model means meets the original joint symbolic thresholds. Mean TR values
\ResPythiaSmallTransferRatio{}, \ResPythiaLargeTransferRatio{}, and
\ResGPTTwoTransferRatio{} are based on the
\ResPythiaSmallTREstimableCount{}, \ResPythiaLargeTREstimableCount{}, and
\ResGPTTwoTREstimableCount{} estimable target conditions, respectively; unavailable ratios are excluded. The means are consistent with partial intervention
transfer relative to the target-selected reference, but do not establish
positive transfer for every condition, full circuit recovery, or population
significance. A TR is not a percentage of an identified reasoning circuit.

Accuracy decreases of \ResPythiaSmallAccuracyDrop{},
\ResPythiaLargeAccuracyDrop{}, and \ResGPTTwoAccuracyDrop{} percentage points
are paired differences on the causal cohort. Subtracting those decreases from
the accuracy means over three behavioral seeds does not produce matched-cohort ablated
accuracies. These descriptive FP32 results remain separate from the original
BF16 bridge and do not replace its full behavioral validation criteria.

\FloatBarrier
\section{Execution and numerical validation}
\label{app:numeric_execution}

\subsection{Native model path and reference computation}
\label{app:supp_native}

The implementation is divided into the experiment definitions in
\path{anc/code/cea_native_suite.py}, the prefix-cache engine in
\path{anc/code/cea_fast_v9.py}, and numerical routing in
\path{anc/code/cea_numeric_v9.py}. Source checksums in the accompanying protocol identify the implementation.

Models use the native Hugging Face GPT-NeoX/Pythia or GPT-2 implementation,
in evaluation mode without gradients. The output embedding is obtained through
\texttt{get\_output\_embeddings()}, and model-specific projection layouts are
handled explicitly. FP32 inference is retained without autocast, TF32, lower-bit
weight conversion, or model fine-tuning. Attention uses the native eager path.
The reference clean readout for a cached prompt is obtained from a full native
forward pass. Storage can use $-1$ padding to hold unequal-length inputs, but the
model never receives those padding values: inputs are grouped by unpadded length, and outputs are restored to the original row order. Native reference calls use unpadded singleton
inputs where required for the established reference path.

The numerical protocol distinguishes exact mathematical identities from finite
floating-point comparisons. Replacement of head channels before a linear output
projection implements the specified write intervention, but reordering batch
computations or reusing attention caches can change floating-point results.
Native and cached intervention outputs are therefore compared numerically.

\subsection{Causal-prefix caching and independent intervention branches}
\label{app:supp_replay}

All supplemental white-box interventions act at the final prompt position.
Causal attention prevents changes at that position from modifying earlier
positions. For a prompt of length $T$, the engine selects a suffix length $s$ from $\{1,8,32,64\}$, limited to the available context length. A full native prefill
stores the key/value tensors for positions before the suffix. For each head-set evaluation, the engine reuses only the unmodified prefix and recomputes the suffix, including the final token, every affected attention operation, all downstream layers, layer normalizations, and MLP computations. The intervention is still applied only at
the final position. Earlier suffix tokens are recomputed without head replacement.

The prefix tensors are read-only. Each intervention branch receives an independent
cache object, and the selected head mask and donor row are aligned with the
original prompt/branch indices. Multiple head sets for the \emph{same} prompt can
be evaluated in parallel; each distinct prompt has its own prefix and donors.
There is no addition of isolated single-head scores to approximate a set effect.
This optimization does not apply unchanged to an intervention that alters prefix
tokens or earlier positions, such as the original \textsc{Seq} scope experiment.

\begin{table}[htbp]
\centering\small
\caption{Hardware-specific execution defaults. Sample budgets are unchanged. Run metadata and engine configurations record the actual settings; imported measurements retain their original hardware provenance.}
\label{tab:supp_execution}
\begin{tabularx}{\linewidth}{@{}p{0.36\linewidth}XXX@{}}
\toprule
Setting & A100 & L4 & T4 \\
\midrule
Initial intervention branches per prompt & 64 & 32 & 8 \\
Logits-only NL behavior microbatch & 16 & 8 & 2 \\
Model precision & FP32 & FP32 & FP32 \\
\midrule
\multicolumn{4}{@{}l@{}}{Shared execution and validation settings}\\
\midrule
Metric set block & \multicolumn{3}{l}{256 sets}\\
Prefix CPU cache budget & \multicolumn{3}{l}{768 MiB}\\
Suffix candidates & \multicolumn{3}{l}{1, 8, 32, 64 tokens}\\
Actual-pair probe sets & \multicolumn{3}{l}{Up to 3 per registered numerical condition}\\
Replay raw-readout tolerance & \multicolumn{3}{l}{0.002 absolute}\\
Replay normalized-score tolerance & \multicolumn{3}{l}{0.002 absolute}\\
Full-logit guard & \multicolumn{3}{l}{0.01 absolute}\\
Branch-order/isolation tolerance & \multicolumn{3}{l}{0.0004 absolute}\\
Empirical refinement safety factor & \multicolumn{3}{l}{2}\\
Periodic replay recheck & \multicolumn{3}{l}{Every 5,000 replay calls}\\
Minimum warm-prefix speedup gate & \multicolumn{3}{l}{2, unless explicitly overridden and recorded}\\
\bottomrule
\end{tabularx}
\end{table}

Execution microbatches control memory use, not the number of evaluated prompts, candidate sets, or null trials. Hardware-specific defaults are listed in Table~\ref{tab:supp_execution}; run metadata records the actual settings.

For output-only NL behavior, the execution path batches unpadded equal-length
prompts and stores only required final logits, instead of caching every head's
activation. Causal selection and donor construction retain the head activations
they require. A cache of pure zero-ablation outputs can reuse an identical
model/prompt/head-set/readout calculation across label permutations. The raw cache
does not apply query-orientation signs or high/low labels; those are applied later.
Patching caches additionally depend on donor and replacement identities.

\subsection{Validation stages and native candidate refinement}
\label{app:supp_refinement}

\paragraph{Software checks.}
Before the pretrained-model experiments, the notebook
runs synthetic causal-attention regressions, model-interface checks, and a full
small-budget software workflow. These exercise both GPT-NeoX and GPT-2 layouts,
cache isolation, interventions, migration, and resume behavior. Synthetic tests check software behavior and are distinct from pretrained-model measurements.

\paragraph{Pretrained replay checks.}
Each pretrained checkpoint is compared with its full native forward path on symbolic
and NL probe prompts, original and \texttt{yes/no} readouts, and nonempty head
interventions. Candidate suffix lengths and branch sizes are tested before the
main computation. Raw readout differences, full-logit maximum errors, and reversed
branch-order errors use the thresholds in Table~\ref{tab:supp_execution}. The
initial global diagnostic also records normalized score errors for sufficiently
large probe contrasts ($|\widehat\Delta|\geq0.05$). The global diagnostic is followed by checks for each condition on the actual paired data. Every record of raw values includes the compared
readout IDs, baseline, intervention values, denominator, and errors.

\paragraph{Common native baseline.}
Both the replay and reference score calculations use the same native clean
$\bar y_H$ and $\bar y_L$. Only intervention outputs differ in the comparison.
Scores computed with a clean baseline computed with replay are separate diagnostics; reported effect estimates use the native clean baseline. Small raw differences can be amplified
when $|\widehat\Delta|$ is small; passing a full-vocabulary logit check alone is
therefore insufficient for normalized shrink and match.

\paragraph{Condition-specific paired-data checks.}
Each numerical condition is keyed by model, target pair, donor pair, replacement identity, weight, and denominator floor. The validation routine selects up to three evenly spaced nonempty candidate sets from the requested list and compares native and replay per-example readouts, raw contrasts and gains, shrink, both directional matches, their average, and $Q_\lambda$. The compared arrays and errors are stored. The maximum observed raw
readout error is denoted $e_y$ and the maximum score error $e_Q$. Conditions with
$e_y>0.002$ route their set-score evaluation to native inference. Non-selection
conditions with $e_Q>0.002$ also use native set scores. The same tolerances apply to small-contrast conditions.

\paragraph{Selection boundaries and winners.}
Large candidate lists can use replay to screen, followed by native refinement.
Let $D_a$ be the denominator in Equation~\ref{eq:supp_stabilized_score}. The
empirical readout envelope is
$\eta_y=\max\{10^{-7},2e_y\}$, and the initial score radius for an unrefined
candidate is $2\eta_y/D_a$. Native checks of candidates near the top-$k$ boundary
or the winner update that radius to at least twice any observed replay/native
$Q_\lambda$ discrepancy. Candidate intervals that overlap the membership or
winner boundary trigger more native measurements. Every selected greedy candidate is re-evaluated with the native forward pass. The single-head screening boundary uses $k=48$; greedy
augmentation uses $k=1$. Native comparisons resolve exact ties by the same
coordinate rule.

This refinement rule uses a finite empirical error envelope, not a proved uniform
bound over unmeasured candidates. The error envelope does not establish a uniform ranking-equivalence guarantee beyond the tested conditions. The number of natively rechecked candidates,
actual errors, final radius, and selected winner are logged. Raw-check failures
route the full affected selection evaluation to native computation instead.

\paragraph{Final arrays and intermediate diagnostics.}
A single requested final set and the exported selected-set effect arrays use full
native reference calculations. Multi-set random controls use replay only if their
condition passes the raw and score checks; otherwise their arrays are native.
The scan over all vocabulary coordinates with fixed head sets uses native mean logits, because certifying
many weak readout ratios with one unrelated output axis would be insufficient.
Count-curve intervention calls can use replay with periodic checks of raw readouts. Output records distinguish imported measurements, native reference arrays, replay screening values, and measurements evaluated through other validated execution paths.

\subsection{Compute accounting, task continuation, and numerical cohorts}
\label{app:supp_provenance}

The engine benchmarks full-forward and warm-prefix intervention times on the current GPU before the main computation. The default engine rejects a replay policy with less
than a twofold warm-prefix speedup unless the slow-execution override is explicit.
This benchmark excludes prefix construction, native refinement, and final native
arrays. The benchmark therefore measures local intervention speedup, not end-to-end experiment runtime. Forward counts,
input tokens, native-reference time, task duration, and completed null-trial
latencies are saved separately. These task-level records are distinct from total experiment runtime.

\paragraph{Checkpoint identity.}
A checkpoint is indexed by scientific configuration, source hashes, model and tokenizer revisions, numerical policy, input hashes, readout IDs, selected heads, and donor/replacement identity. Restarts preserve these identities and do not count as new experimental seeds. Hardware or critical package changes require a new numerical comparison before reuse.

\paragraph{Importing a previous numerical cohort.}
A numerical cohort means measurements made with a specified model revision,
precision, hardware, and execution/validation policy. Imported tables and raw arrays retain their original metadata. Before direct reuse, the compatibility check reruns 12 stored baseline cases: four regimes at $k\in\{1,4,12\}$. It compares saved and current per-example clean, ablated, and patched readouts, together with shrink, directional matches, and $Q$. The check covers these baseline cases only.

Direct reuse requires every comparison to have per-example readout error at most
0.002 and score error at most 0.002. If all 12 comparisons agree, compatible
checkpoints can be used for continuation. If score errors exceed the tolerance but errors in raw readouts remain within the migration limit, completed measurements retain their original numerical cohort, while incompatible unfinished tasks are recomputed under the current configuration. A source/current native readout difference greater than 0.01 stops migration. Measurements from different hardware or precision cohorts are not pooled as repetitions of one configuration. The runtime outputs \path{NUMERICAL_COHORTS.json} and \path{table_provenance.csv} record table origins and the corresponding numerical configuration. Synthetic software-test outputs remain separate from pretrained-model measurements.

\subsection{Reproduction files and implementation coverage}
\label{app:numeric_files}
The accompanying package provides the FP32 implementation and software tests in \path{anc/code/}, configuration records and source hashes in \path{anc/protocol/}, and the reported aggregate results in \path{anc/results/}. The continuation notebook requires a compatible experimental snapshot, which is not included. The supplied FP32 implementation does not cover the separate BF16 T1, QK/OV, LAST/SEQ, objective-weight, tail-length, surface-form, input-occlusion, or six-model experiments.

The environment pins \texttt{transformers==5.16.1} and \texttt{tokenizers==0.23.2}. The recorded environment is Python 3.13.15, PyTorch 2.11.0+cu128, CUDA 12.8, NumPy 2.5.3, pandas 2.3.3, SciPy 1.18.1, and Matplotlib 3.11.2. Source hashes and configuration snapshots are in \path{anc/protocol/}. The runtime records the environment and numerical checks for each new execution.

Table~\ref{tab:supp_output_map} lists the condition-level files generated by experiment runs. These output families are distinct from the aggregate files supplied in \path{anc/results/}. The package does not include a complete experiment-output archive: the per-example effect arrays, run-specific head identities and readout registries, numerical-validation and bootstrap outputs, condition-level FP32 tables, individual permutation trials, original BF16 archives, and postprocessor for intervals on model means are not supplied. The shared NL domain registry is included in \path{anc/protocol/nl_state_registry.json}. The supplied \path{anc/results/transfer_availability.json} contains TR coverage counts. Condition-level interval procedures are specified in Appendices~\ref{app:supp_statistics} and~\ref{app:supp_nl}; Table~\ref{tab:expanded_nl_slots} reports model-level point estimates only. Appendix~\ref{app:reported_ci} summarizes the statistical reporting scope.

\begin{table}[htbp]
\centering\small
\caption{Condition-level files generated by the FP32 implementation. Rows represent experimental units or conditions. The package coverage is specified in the accompanying text.}
\label{tab:supp_output_map}
\begin{tabularx}{\linewidth}{@{}p{0.52\linewidth}X@{}}
\toprule
Files in the run's \texttt{tables/} directory & Recorded content \\
\midrule
\path{primary_buc.csv} & Current-cohort reference sets and primary scores \\
\path{operator_audit_summary.csv}, \path{operator_audit_draws.csv} & Operator summaries and Monte Carlo draws \\
\path{t3_local_summary.csv}, \path{t3_local_subsets.csv} & Local conditions and complete subset values \\
\path{t3_greedy_vs_exact.csv}, \path{t3_full_shortlist_sampled.csv} & Selection/evaluation comparisons and larger-pool samples \\
\path{fixed_budget_baselines.csv}, \path{fixed_budget_random_controls.csv} & Five methods and three control families \\
\path{paired_baseline_differences.csv} & Matched method differences and uncertainty \\
\path{fixed_set_count_examples.csv}, \path{fixed_set_count_summary.csv} & Per-example count values and count-mean fits \\
\path{unseen_gap_fixed_set.csv}, \path{fixed_set_cue_transfer.csv} & Fixed-set gap and wording controls \\
\path{fixed_budget_cross_regime_transfer.csv}, \path{order_only_localization.csv} & Equal-$k$ transfer and fixed-count order localization \\
\path{readout_random_scan.csv}, \path{readout_random_summary.csv}, \path{readout_random_full_buc.csv} & Registered-pair scan and two reselection protocols \\
\path{selection_matched_null_trials.csv}, \path{selection_matched_null_summary.csv} & Trial statistics, eligibility, and permutation inference \\
\path{nl_behavior_all_models.csv}, \path{nl_causal_all_models.csv}, \path{nl_transfer_all_models.csv} & Model/domain behavior, causal roles, and transfer ratios \\
\bottomrule
\end{tabularx}
\end{table}

\clearpage
\section{Extended related work}
\label{app:related_boundary}

This appendix expands the comparisons in Section~\ref{sec:related}. Table~\ref{tab:related_boundary} distinguishes the behavioral target from the intervention and selection rule. A method for finding relevant components and an experimental design for testing their reuse address different questions.

\begin{table}[h]
\centering
\small
\caption{Boundary between nearby lines of work. ``Same set'' means that one selected set is evaluated by ablation and both mean-patching directions.}
\label{tab:related_boundary}
\setlength{\tabcolsep}{3.2pt}
\begin{tabularx}{\linewidth}{p{0.20\linewidth}p{0.19\linewidth}p{0.22\linewidth}X}
\toprule
Work line & Model setting & Behavioral target & Causal criterion \\
\midrule
Indirect-object identification and factual recall & Frozen pretrained models & Semantic task behavior & Study-specific interventions \\
Synthetic algorithm studies & Often task-trained models & Designed algorithm & Study-specific mechanism tests \\
Patching, ACDC, EAP & Model- and target-dependent & Researcher-supplied target & General localization machinery \\
CEA/BUC & Frozen pretrained models & Controlled count/order readout; controlled NL bridge & Same set tested by ablation and two-way patching \\
\bottomrule
\end{tabularx}
\end{table}

\subsection{Behavioral targets and controlled evidence}
\label{app:related_targets}

Studies of induction, factual recall, and indirect-object identification connect model outputs to particular internal components \citep{olsson2022induction,geva2021transformer,meng2022locating,wang2023interpretability}. In those settings, the target behavior supplies a task-specific criterion for an explanation. Task-trained symbolic models offer a different starting point: the training task specifies an algorithmic problem, allowing a discovered mechanism to be compared with the computation required to solve it \citep{brinkmann2024mechanistic}. Behavioral studies of structured tasks also examine failures of compositional generalization \citep{dziri2023faith}.

CEA does not train a model to map symbols to answers. Its symbolic contrast changes the count or order of fixed tokens and measures the resulting output response. Count sweeps test extension beyond the selection endpoints; fixed-count order comparisons test ordering dependence. Keeping a head set fixed distinguishes persistence of its effect from the ability to select another set. The symbolic target has no external correctness label, so the labeled NL task provides a separate test of behavioral relevance. Lexical balance and query reversal constrain simple shortcuts but do not eliminate all phrase-based strategies or pretrained associations.

\subsection{Interventions and component selection}
\label{app:related_selection}

Activation patching replaces internal activations to measure their effect on a supplied output metric; path patching restricts the routes through which a replacement can influence that metric \citep{goldowskydill2023path,zhang2024best}. ACDC searches a computational graph using intervention tests, while EAP ranks candidate edges using attribution estimates \citep{conmy2023acdc,syed2023attribution}. Their targets, graph units, and search budgets need not match a head-set comparison.

BUC uses head writes at the final prompt position as its intervention unit and evaluates removal plus recovery in both patching directions on the same set. Its scores include set interactions because interventions are evaluated jointly. The BF16 ActPatch baseline instead ranks single-head high-to-low recovery within a sensitivity shortlist before jointly evaluating the selected union. These comparisons change the ranking objective as well as the search procedure. They therefore demonstrate a measured removal--recovery trade-off, not the independent benefit of greedy search over ranking single heads with the same objective or an end-to-end cost advantage (Appendix~\ref{app:baseline_decomposition_new}). No quantitative result in the paper is attributed to a separate reproduction of EAP or ACDC.

\subsection{Component reuse across tasks and prompts}
\label{app:related_reuse}

\citet{merullo2024reuse} examine shared components in indirect-object identification and the Colored Objects task. Their intervention analysis provides evidence of common computational steps across these tasks. \citet{nainani2024adaptive} test the indirect-object-identification circuit under prompt variants and find reuse with additional input edges. These studies establish that circuit reuse can occur; neither reuse nor failure should be presumed for another experimental setting.

More recently, \citet{devarda2026shared} study whether circuit overlap and loading predict arithmetic accuracy across numeric and verbal formats. Their predictive question differs from CEA's direct measurement of a fixed source set's intervention effects on specified target conditions. CEA makes no claim to predict correctness on arbitrary new tasks.

Our comparisons distinguish changes in count, ordering, wording, output coordinates, and replacement operators. In the labeled bridge, the fixed symbolic set and an NL-selected set are evaluated on the same NL targets at equal head budget. Source selection does not use NL prompts; the expanded FP32 transfer uses target-specific donor calibration. Positive effects support partial reuse under that intervention protocol. Because both tasks use the \texttt{yes/no} readout, the comparison does not isolate shared evidence aggregation from shared answer processing. Likewise, the seven targeted readout pairs and the larger registered scan are different samples, not a paired estimate of a reselection advantage.

\subsection{Faithfulness, localization, and statistical tests}
\label{app:related_faithfulness}

\citet{miller2024faithfulness} show that circuit faithfulness scores can change substantially with ablation methodology. Our operator audit measures that dependence for fixed CEA head sets under removal and two patch directions. It tests sensitivity of the reported effect to the replacement distribution rather than assuming a unique operator-independent circuit.

\citet{shi2024hypothesis} develop statistical tests for behavior preservation, localization, and minimality. Our selection-matched null asks a different question: whether the observed utility is unusual under the specified within-pair label-exchangeability null when the adaptive selection pipeline is repeated. The trial design and eligibility rules are in Appendix~\ref{app:supp_null}. Rejection of that null neither establishes minimality or uniqueness nor assigns semantic correctness to the symbolic response. Causal effects of small sets, hypothesis rejection, and cross-setting reuse remain separate claims.

\subsection{Representations and in-context learning interpretations}
\label{app:related_icl}

Analyses of information routing and residual representations provide context for interpreting writes at the final prompt position \citep{geva2023dissecting,dar2023analyzing,ferrando2024information}. Task-vector and function-vector work studies internal representations of demonstrated functions and their effects in other contexts \citep{hendel2023incontext,todd2024function}. CEA retains the identity of selected heads while allowing their activations to depend on the target inputs and donor distribution. A fixed head set is therefore not a fixed activation vector or a complete description of the earlier computation producing its writes.

Bayesian and gradient-based accounts of in-context learning relate model behavior to learning rules under specified tasks and assumptions \citep{xie2022explanation,garg2022what,akyurek2023what,vonoswald2023transformers,panwar2024bayesianprism,bai2023transformers,ahn2023transformers,li2023transformers}. The exact two-token log-odds identity in CEA supplies a scalar measurement; it does not identify a Bayesian update or optimization algorithm. Our theory instead states conditions for attribution fidelity, shortlist recovery, and greedy utility bounds. The finite empirical audits do not establish all those conditions for the implemented normalized objective.

\paragraph{Remaining mechanism questions.}
The evidence leaves open whether NL transfer reflects shared count/order processing, shared answer readout, or both; which computations at earlier tokens and along head--MLP paths supply the writes at the final prompt position; and whether the same components remain effective across pretraining checkpoints or longer heterogeneous evidence. These questions require interventions beyond the reported tests with fixed head sets.

\clearpage
\section{FP32 confidence intervals and statistical scope}
\label{app:reported_ci}

\paragraph{Intervals and their sampling units.}
Table~\ref{tab:expanded_nl_slots} gives descriptive model means and transfer-coverage counts, not model-level CIs. The condition-level procedures in Appendix~\ref{app:supp_nl} do not supply intervals for those means. The two bootstrap intervals in Table~\ref{tab:reported_ci_other} use matched clusters as specified in Appendix~\ref{app:supp_statistics}, with head sets and donor banks fixed. Behavioral and causal-selection seeds remain separate; no simultaneous coverage or inference across independently pretrained models is claimed.

\begin{table}[htbp]
\centering\small
\caption{FP32 outcomes outside Tables~\ref{tab:expanded_nl_slots} and~\ref{tab:audit_slots}. A dash indicates that no CI is reported; it does not denote zero uncertainty. Only the RANDOM baseline difference and raw null-test utility receive bootstrap intervals here.}
\label{tab:reported_ci_other}
\setlength{\tabcolsep}{4.0pt}
\begin{tabular}{@{}lcc@{}}
\toprule
Quantity & Point estimate or span & 95\% bootstrap CI \\
\midrule
BUC $-$ sensitivity, \textsc{Front} & \ResBaselineFrontDifference{} & -- \\
BUC $-$ sensitivity, \textsc{Back} & \ResBaselineBackDifference{} & -- \\
BUC $-$ sensitivity, \textsc{Interleave} & \ResBaselineInterleaveDifference{} & -- \\
BUC $-$ sensitivity, \textsc{Random} & \ResBaselineRandomDifference{} & \ResBaselineRandomDifferenceCI{} \\
Off-diagonal match: observed span & \ResEqualBudgetTransferRange{} & -- \\
Cue-change shrink: observed span & \ResCueShrinkRange{} & -- \\
Cue-change match: observed span & \ResCueMatchRange{} & -- \\
Evaluable conditions, of 1024 & \ResReadoutEvaluableCount{} & -- \\
Pass fraction: all registered & \ResReadoutAllPassRate{} & -- \\
Pass fraction: evaluable & \ResReadoutEvaluablePassRate{} & -- \\
Completed permutation trials & \ResNullCompletedTrials{} & -- \\
Observed raw null-test utility & \ResNullObservedStatistic{} & \ResNullObservedStatisticCI{} \\
Holm-adjusted permutation probability & \ResNullAdjustedP{} & -- \\
\bottomrule
\end{tabular}
\end{table}

\paragraph{Deterministic conditions.}
At each fixed count, \textsc{Front}, \textsc{Back}, and \textsc{Interleave} contain one unique symbolic input. Repeated copies of one matched pair provide no independent prompt variation. Consequently, these conditions receive no prompt-bootstrap CI. Their point estimates describe the specified inputs, not uncertainty over a richer prompt population. The \textsc{Random} difference uses distinct matched permutations, and its interval in Table~\ref{tab:reported_ci_other} excludes zero.

\paragraph{Finite-design summaries.}
Off-diagonal and cue spans describe 12 source--target and eight regime--wording conditions, respectively. The observed counts of 960 evaluable conditions and 99 completed null trials are fixed; their values determine the denominators. The passing fractions summarize the registered readout scan without a population CI over alternative registries. Table~\ref{tab:audit_slots} likewise gives finite operator minima, negative-marginal fractions, ratio minima, and local selection regrets without population intervals. The summaries describe the tested conditions, not a population guarantee or a bound for the full pool.

\paragraph{Permutation uncertainty.}
With \(B=99\) completed trials and one estimable hypothesis, \(\widehat p=0.01\) implies zero null exceedances. The separate 95\% Clopper--Pearson interval for the null exceedance probability is \([0,1-0.025^{1/99}]\approx[0,0.0366]\). It is not a bootstrap interval, an effect-size interval, an adjusted-\(p\) interval, or a trial-count interval (Appendix~\ref{app:supp_null}).

\paragraph{NL interpretation.}
The three model-level mean TR values are between zero and one, and the
observed accuracy decreases are positive. These are point estimates, not
model-level significance tests. The TR means use 11, 12, and 9 estimable
conditions in Pythia-410M, Pythia-1B, and GPT-2, respectively; their 1, 0,
and 3 suppressed condition-level intervals do not remove the corresponding
finite ratios from the means. Retained condition-level intervals are not
model-level CIs. Neither the positive means nor the estimability counts
establish every paired-behavior criterion, positive transfer in all target
conditions, complete circuit recovery, or differences between models.

\paragraph{BF16 bridge.}
The BF16 tables supply execution means and standard deviations, not the prompt-pair bootstrap intervals required by the declared criterion. Neither those summaries nor the separate FP32 procedures establish the complete BF16 criterion (Appendix~\ref{app:nl_bridge}).

\end{document}